\documentclass{article}

\newif\ifanonymized
\anonymizedfalse 

\usepackage{amssymb,amsmath,amsthm,bbm}
\usepackage{verbatim,float,url,dsfont}
\usepackage{algorithm,algorithmic}
\usepackage{graphicx,subcaption,psfrag}
\usepackage{mathtools}
\usepackage{multirow,multicol}
\usepackage{ragged2e}
\usepackage{xr-hyper}
\usepackage{array}
\usepackage{stmaryrd,scalerel}
\usepackage[shortlabels]{enumitem} 
\usepackage{booktabs,makecell,multirow,threeparttable,adjustbox}
\usepackage{tcolorbox}

\usepackage[colorlinks=true,citecolor=blue,urlcolor=blue,linkcolor=blue]{hyperref}
\usepackage[round]{natbib}

\usepackage[utf8]{inputenc} 
\usepackage[T1]{fontenc}    
\usepackage{booktabs}       
\usepackage{nicefrac}         
\usepackage{microtype}      

\usepackage{tikz}
\usetikzlibrary{decorations.pathreplacing,arrows.meta, patterns, calc}

\ifdefined\TimesFont 
\usepackage{times} 
\fi

\ifdefined\ParSkip 
\usepackage{parskip} 
\fi

\newcommand{\anonymized}[1]{%
  \ifanonymized
  \else
    #1
  \fi
}

\newtheorem{theorem}{Theorem}
\newtheorem{lemma}{Lemma}
\newtheorem{corollary}{Corollary}

\newtheorem{definition}{Definition}
\newtheorem{proposition}{Proposition}

\usepackage{mdframed}

\newmdenv[
  backgroundcolor=gray!10,
  linecolor=black,
  linewidth=0.8pt,
  roundcorner=4pt,
  skipabove=10pt,
  skipbelow=10pt,
  nobreak=true 
]{procbox}

\newtheorem{innerprocedure}{Procedure}

\newenvironment{procedure}[1][]
  {\begin{procbox}\begin{innerprocedure}[#1]}
  {\end{innerprocedure}\end{procbox}}

\newtheorem*{assumption*}{\assumptionnumber}
\providecommand{\assumptionnumber}{}


\makeatletter
\newcommand*\rel@kern[1]{\kern#1\dimexpr\macc@kerna}
\newcommand*\widebar[1]{%
  \begingroup
  \def\mathaccent##1##2{%
    \rel@kern{0.8}%
    \overline{\rel@kern{-0.8}\macc@nucleus\rel@kern{0.2}}%
    \rel@kern{-0.2}%
  }%
  \macc@depth\@ne
  \let\math@bgroup\@empty \let\math@egroup\macc@set@skewchar
  \mathsurround\z@ \frozen@everymath{\mathgroup\macc@group\relax}%
  \macc@set@skewchar\relax
  \let\mathaccentV\macc@nested@a
  \macc@nested@a\relax111{#1}%
  \endgroup
}
\makeatother

\DeclareMathOperator*{\argmin}{argmin}

\DeclareMathOperator{\sign}{sign}

\def\R{\mathbb{R}}

\def\T{\mathsf{T}}

\def\htheta{\hat{\theta}}

\def\th{^{\textnormal{th}}}

\def\one{\mathds{1}}

\def\cA{\mathcal{A}}
\def\cB{\mathcal{B}}

\def\cK{\mathcal{K}}

\usepackage[dvipsnames]{xcolor} 

\def\q{q}

\def\g{g}

\def\b{b}

\def\v{v}
\def\u{u}
\def\x{x}
\def\y{y}
\def\z{z}

\newcommand{\ttheta}{\tilde{\theta}}
\newcommand{\btheta}{\theta}
\newcommand{\bttheta}{\ttheta}
\newcommand{\bomega}{\omega}

\newcommand{\cov}{\mathrm{cov}}

\newcommand\numberthis{\addtocounter{equation}{1}\tag{\theequation}} 

\usepackage{xr} 
\usepackage[capitalise]{cleveref}

\newcommand{\tick}{\textcolor{green!60!black} {\checkmark}} 
\newcommand{\cross}{\textcolor{red!70!black}{$\times$}} 

\usepackage[]{color-edits}
\addauthor{td}{ForestGreen}
\addauthor{ig}{orange}
\addauthor{rt}{cyan}

\def\GD{gradient descent} 
\def\GEQ{gradient equilibrium} 

\usepackage[margin=1in]{geometry}

\title{Calibrated Multi-Level Quantile Forecasting}  

\author{Tiffany Ding \and Isaac Gibbs \and Ryan J. Tibshirani} 
\date{University of California, Berkeley}

\begin{document}

\maketitle

\begin{abstract}
We develop an online method that guarantees calibration of quantile forecasts at multiple quantile levels simultaneously. In this work, a sequence of quantile forecasts is said to be \emph{calibrated} provided that its $\alpha$-level predictions are greater than or equal to the target value at an $\alpha$ fraction of time steps, for each level $\alpha$. Our procedure, called the \emph{multi-level quantile tracker} (MultiQT), is lightweight and wraps around any point or quantile forecaster to produce adjusted quantile forecasts that are guaranteed to be calibrated, even against adversarial distribution shifts. Critically, it does so while ensuring that the quantiles remain ordered, e.g., the 0.5-level quantile forecast will never be larger than the 0.6-level forecast. Moreover, the method has a no-regret guarantee, implying it will not degrade the performance of the existing forecaster (asymptotically), with respect to the quantile loss. In our experiments, we find that MultiQT significantly improves the calibration of real forecasters in epidemic and energy forecasting problems, while leaving the quantile loss largely unchanged or slightly improved.



\end{abstract}

\section{Introduction}

Probabilistic forecasts are commonly conveyed via quantiles. An $\alpha$-level 
quantile forecast attempts to predict the value below which some unknown target
outcome $y_t$ falls with probability $\alpha$. Consider a forecaster that, at
each time $t$, outputs a vector of quantile forecasts  
\[
q_t = (q_t^{\alpha_1}, q_t^{\alpha_2}, \dots, q_t^{\alpha_m}),
\]
for prespecified quantile levels $\cA = \{\alpha_1, \alpha_2, \dots,
\alpha_m\}$, where $0 < \alpha_1 < \alpha_2 < \dots < \alpha_m$. Forecasts of
this type inform decision making in a wide range of applications, such as public
health \citep{doms2018assessing, lutz2019applying}, inventory management
\citep{cao2019quantile}, and energy grid operation \citep{hong2016probabilistic}. When decisions are made on the basis of forecasts that are
\emph{calibrated}, this can lead to a reliability guarantee. For example, if a
retailer has access to a sequence of calibrated 0.95-level quantile
forecasts of weekly demand, and they ensure their inventory level meets these demand forecasts, then this guarantees that they run out of stock at
most 5\% of weeks.

Although a single ($m=1$) quantile is sometimes sufficient for decision making, 
this is not true in general. When there are multiple downstream users, each with
different risk tolerances and uses for the forecasts, it is often more useful to 
provide forecasts at multiple ($m \ge 2$) quantile levels. In this work, we
seek to produce multi-level quantile forecasts that satisfy the following two
useful properties: 

\begin{enumerate}
\item \emph{Calibration.} For any sequence of target values $y_1, y_2, \dots$, including sequences chosen adversarially, the long-run coverage of the $\alpha$-level quantile forecasts should converge to $\alpha$ for each $\alpha \in \cA$. That is, if we define \smash{$\cov_t^{\alpha} = \one\{y_t \leq
    q_t^{\alpha}\}$}, then we want   
  \begin{equation}\label{eq:long_run_coverage} 
  \lim_{T\to\infty} \frac{1}{T} \sum_{t=1}^T \cov_t^{\alpha} = \alpha,
  \quad \text{for all $\alpha \in \cA$}.
  \end{equation}
  This ensures a coherence between forecasts and realized values, even
  if the distribution of the target changes over time.   

\item \emph{Distributional consistency.} Forecasts should also be ordered 
  across quantile levels. That is, we want: 
  \begin{equation}\label{eq:no_crossing}
  q_t^{\alpha_1} \le q_t^{\alpha_2} \le \dots \le q_t^{\alpha_m}, 
  \quad \text{for all $t = 1,2,\dots$}.
  \end{equation}
  Without this ordering, the vector of forecasts would not correspond to a valid
  probability distribution, making it difficult for decision makers to interpret
  or trust. 
\end{enumerate}

There are many methods for producing quantile forecasts, including classical
time series models such as ARIMA and exponential smoothing, as well as modern
machine learning approaches such as random forests, and deep neural
networks. However, these forecasts often fail to satisfy calibration. Our aim is
to take any existing forecaster and transform its predictions \emph{online} (in
real time) so that the resulting forecasts satisfy both
\eqref{eq:long_run_coverage} and \eqref{eq:no_crossing} for any sequence of
outcomes. Henceforth, we refer to this joint goal as \emph{calibration without 
  crossings}. Furthermore, subject to calibration without crossings, we want the
forecasts to remain \emph{sharp}: all else equal, the forecasts should correspond to a
probability distribution with low variance (the quantile predictions should not
be too dispersed), so that they provide the least possible uncertainty about the
outcome. 

Our first objective, online calibration, has been studied extensively
for the single $(m=1)$ quantile setting in the online conformal
prediction literature, beginning with \cite{gibbs2021adaptive}. Online conformal algorithms
achieve distribution-free calibration \eqref{eq:long_run_coverage} for a single
level $\alpha$. Of particular relevance to our paper is the quantile tracker
(QT) algorithm from \cite{angelopoulos2023conformal}. The idea behind this
method is simple: to track the $\alpha$-level quantile over time, we should
increase our current quantile estimate if it is smaller than $y_t$ (it
``miscovers'') and we should decrease our estimate if it is larger than or equal to
$y_t$ (it ``covers''). The amount by which we increase or decrease these
estimates is chosen to yield a long-run coverage of $\alpha$, which is guaranteed
whenever the target values are bounded in magnitude.

It is natural to try to use QT to solve the multi-level quantile calibration problem. However, simply applying this algorithm to multiple levels separately
often results in quantile crossings, violating
\eqref{eq:no_crossing}; in experiments on the COVID-19 Forecast Hub 
dataset from \cite{cramer2022evaluation}, QT produced crossings at 87\% of
time steps on average (see Appendix \ref{sec:QT_crossings_APPENDIX}).     
To solve the problem of simultaneously calibrating multiple quantiles without producing
crossings, we develop a procedure that we call the \emph{multi-level quantile  
  tracker (MultiQT)}, which combines a QT-style update for each level with an
ordering step to ensure forecasted quantiles are distributionally consistent. As
we later show, various naive ways of combining individual quantile calibration and ordering techniques do not
achieve calibration, but our method provably does.      

To derive the calibration guarantee for MultiQT, we first connect our goal of
calibration without crossings to a more general problem of \emph{constrained
  gradient equilibrium.}  Many statistical objectives in online settings
(including calibration) are special cases of a condition introduced by
\cite{angelopoulos2025gradient} called \emph{gradient equilibrium}, which says
that the average of the loss function gradients evaluated at the chosen iterates
converges to zero as the number of time steps goes to infinity. They show that to produce iterates which
achieve gradient equilibrium, one can simply run online gradient descent, provided that the losses satisfy certain weak
conditions. However, it was heretofore not
known whether gradient equilibrium can be achieved if the iterates must obey constraints, such as in our multi-level quantile
forecasting setting, where our forecasts must lie in the set of ordered
vectors. We provide an affirmative answer by showing that \emph{lazy gradient
  descent}, which combines online gradient updates with a projection step to 
satisfy the iterate constraints, provably achieves gradient equilibrium as long
as the loss function and constraint set jointly satisfy an additional condition
we call \emph{inward flow}. 

We show that the loss function and constraint set for the calibration without crossings problem satisfy inward flow. Thus, MultiQT, which can be written as lazy gradient descent on that loss function and constraint set, inherits a calibration guarantee from our more general analysis of constrained gradient equilibrium.
Finally, we prove a no-regret guarantee for MultiQT with respect to
the quantile loss. Due to the standard decomposition of the quantile loss into
calibration and sharpness terms, this result can be informally interpreted as saying that
MultiQT achieves calibration without paying a steep price in terms of sharpness.
 
\subsection{A peek at results: calibrating COVID-19 forecasts} 

To illustrate the behavior of our method in practice, we begin with a brief case
study. During the COVID-19 pandemic, forecasting teams submitted forecasts each
week to the United States COVID-19 Forecast Hub of COVID-19 deaths in each
state one, two, three, and four weeks into the future. In Figure
\ref{fig:example_raw}, we display one team's one-week-ahead forecasts for weekly
COVID-19 deaths in California. We can see that these quantile forecasts are too
narrow and biased downward; focusing in on the calibration plot shown in the right panel, we see that the forecasts fail to cover the true death count at the
desired rate and convey more certainty than is appropriate.
To remedy this, our proposed MultiQT method can be applied in real time to
recalibrate such forecasts. Figure \ref{fig:example_calibrated} shows the
results of running MultiQT online, where at each time $t$, the method uses the 
performance of the forecasts up through time $t-1$ to correct the current
forecast. We observe that MultiQT corrects the downward bias, particularly 
present in the upper quantiles, and the resulting forecasts achieve close to
perfect calibration. By improving the coherence of the forecasts with eventual
death counts, the use of a recalibration method such as MultiQT can improve the
quality of public communication about the expected trajectory of pandemics 
and help inform timely public health decisions regarding allocation of 
scarce resources and hospital staffing \citep{cramer2022evaluation}. We will
return to this COVID-19 forecasting application in Section
\ref{sec:experiments}.  

\begin{figure}[t!]
\begin{subfigure}{\textwidth} 
\centering
\begin{subfigure}{0.48\textwidth}
\includegraphics[width=\linewidth]{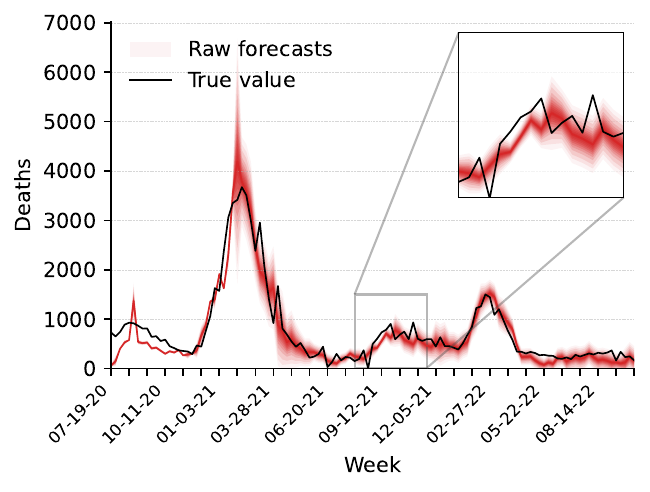}
\end{subfigure}%
\begin{subfigure}{0.23\textwidth}
\raisebox{1.65em}{
\includegraphics[width=\linewidth]{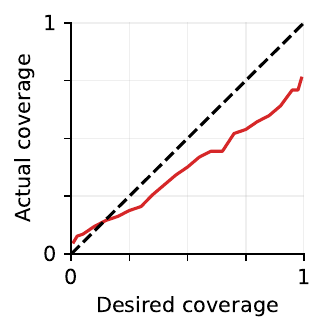}}
\end{subfigure}
\caption{Raw forecasts and their calibration.}
\label{fig:example_raw}
\end{subfigure}

\begin{subfigure}{\textwidth}
\centering 
\begin{subfigure}{0.48\textwidth}
\includegraphics[width=\linewidth]{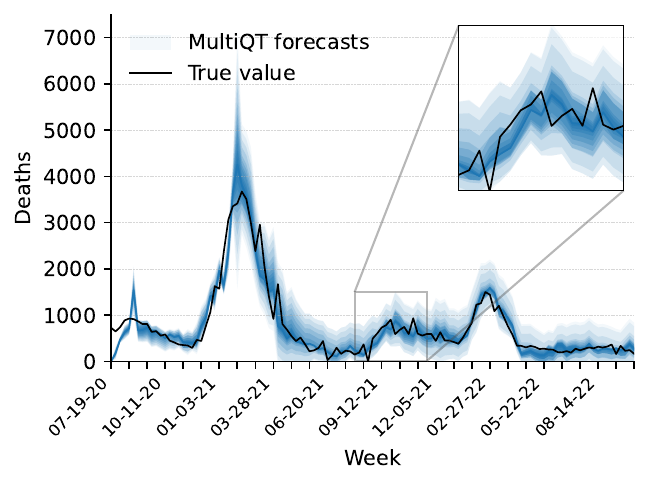}
\end{subfigure}%
\begin{subfigure}{0.23\textwidth}
\raisebox{1.65em}{
\includegraphics[width=\linewidth]{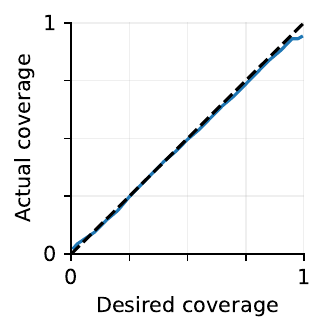}}
\end{subfigure}
\caption{Forecasts and calibration after applying MultiQT.}
\label{fig:example_calibrated}
\end{subfigure}

\caption{One-week-ahead forecasts of weekly COVID-19 deaths in California from
  July 11, 2020 to October 22, 2022 by forecaster \texttt{RobertWalraven-ESG},
  before (top) and after (bottom) applying MultiQT. Forecasts are made at 23
  levels, roughly equally-spaced in between 0.01 and 0.99. To visualize these
  forecasts, we plot colored bands where the lightest opacity connects the 0.01
  and 0.99 level forecasts, the next lightest connects the 0.025 and 0.975 level
  forecasts, and so on.}   
\label{fig:example}
\end{figure}

\subsection{Related work}

Online calibration of a single quantile in the presence of distribution shift
has been studied extensively in the context of online conformal prediction,
beginning with \cite{gibbs2021adaptive}. The central idea underlying many of
these methods is to run online gradient descent on the quantile loss, applied to an iterate either in $\alpha$-space \citep{gibbs2021adaptive} or in
$y$-space \citep{angelopoulos2023conformal}. The latter has the advantage of
leading to fewer infinite sets in general (as well as not requiring an expensive quantile
computation at each time step). Later developments in this line of work 
include ways to adaptively set the learning rate \citep{zaffran2022adaptive, 
  gibbs2024conformal}, extensions to 
losses besides coverage \citep{feldman2022achieving, lekeufack2024conformal},
approaches tailored to multi-horizon forecasting \citep{yang2024bellman,
  wang2024online} or that exploit error predictability
\citep{hu2025distribution}, and methods that consider strongly adaptive
regret \citep{bhatnagar2023improved, hajihashemi2024multi}. Gradient
equilibrium, proposed in \cite{angelopoulos2025gradient}, generalizes the
concept of online calibration to a setting with arbitrary losses, and in doing
so, generalizes the analysis of online gradient descent that underlies work on
online conformal prediction.

A complementary line of work on calibration uses Blackwell approachability and
related ideas, resulting in procedures that are generally more complex than
those based on gradient descent but also offer stronger (conditional)
guarantees. 
The basic idea underpinning many papers on calibration \citep{foster1999proof,
  foster2021forecast} and defensive forecasting \citep{vovk2005defensive,
  perdomo2025defense} is that certain properties defined in terms of
time-averages can be cast as special cases of Blackwell approachability. For
example, the convex set in Blackwell approachability can be defined to encode 
zero calibration error. \cite{gupta2022online} builds on these ideas to develop
algorithms for group-conditional quantile calibration, which are then applied to
online conformal prediction in \cite{bastani2022practical}. This was later
extended to the high-dimensional setting by \cite{noarov2023high}. 
While this framework could in principle accommodate multiple quantile
calibration without crossings, it lacks a practical implementation for this
problem setting, since it would require solving nontrivial inner optimization 
problems at each time step. In contrast, our approach is easily implementable
and wraps around any existing forecaster, albeit targeting a weaker goal 
of unconditional calibration. Also related are \cite{deshpande2023calibrated}
and \cite{marx2024calibrated}, which use Blackwell approachability to calibrate 
probabilistic forecasts that specify a distribution over $y_t$.  

Quantile prediction has a long history of study in statistics, dating back
to the seminal work of \cite{koenker1978regression}. Though it is traditionally
studied in the offline setting (with i.i.d.\ data), the problem of mitigating 
quantile crossing when jointly learning multiple quantiles has been present from
the start \citep{bassett1982empirical}. In the offline setting, solutions have been proposed in the
form of post-processing \citep{chernozhukov2010quantile, fakoor2023flexible},
constrained optimization \citep{liu2009stepwise}, or neural network learning
architectures that enforce monotonicity of the output vector
\citep{gasthaus2019probabilistic, park2022learning}.
In the online setting, \cite{zhang2024benefit} proposes a method that enforces
monotonicity but achieves only a no-regret guarantee and not a calibration
guarantee. \cite{li2025neural} make use of ideas from
\cite{angelopoulos2023conformal} to design a loss function for training a
forecaster that targets coverage with no quantile crossings, but in practice
their method still produces crossings (at roughly 10\% of time steps in their
experiments).  

Finally, our work relates to a broader literature on forecast recalibration,
which considers ways to improve the calibration of an existing forecaster 
\citep{brocklehurst1990recalibrating}
or an ensemble of forecasters \citep{hamill1997verification, raftery2005using,
  gneiting2013combining, vandendool2017probability}. 

\section{Methods}

In this section, we present our online method for generating calibrated,
distributionally consistent quantile forecasts given an arbitrary base forecaster. We do so by learning offsets that result in
calibrated forecasts when added to the base forecasts. All omitted proofs in
this and subsequent sections are deferred to Appendix \ref{sec:proofs_APPENDIX}
or \ref{sec:negative_results_APPENDIX} unless otherwise stated. 

\paragraph{Notation.} 

We use \smash{$\b_t = (b_t^{\alpha_1}, b_t^{\alpha_2}, \dots, b_t^{\alpha_m})
  \in \R^m$} to denote the vector of base forecasts at time $t$, where
\smash{$b_t^{\alpha}$} is the base forecast for level $\alpha$. We use
$\btheta_t \in \R^m$ to denote the offset vector at time $t$ (which we
adjust online) and $\q_t = \b_t + \btheta_t \in \R^m$ to denote the
corresponding vector of recalibrated forecasts at time $t$. As with the base
forecasts, we will often index elements of the offset and recalibrated forecast
vectors by the quantile level, as in \smash{$\theta_t^{\alpha}$} and
\smash{$q_t^{\alpha}$} for a level $\alpha$. We define \smash{$\cov_t^{\alpha} =
  \one\{y_t \leq q_t^{\alpha}\}$} to be the coverage indicator for the
$\alpha$-level forecast.  For a closed convex set $C \subseteq \R^d$, we use
\smash{$\Pi_C(x) = \argmin_{z \in C} \|x - z\|_2^2$} to denote the projection of
$x$ onto $C$. We define $\cK =\{x \in \R^d: x_1 \leq x_2 \leq \dots \leq x_d 
\}$ to be the $d$-dimensional isotonic cone, where the dimension $d$ can be understood from
context (in general, $d = m$). We refer to the projection \smash{$\Pi_{\cK}$} onto $\cK$ as
\emph{isotonic regression}.

\paragraph{Base forecasts.} 

We assume the base forecasts are distributionally consistent:
\[
b_t^{\alpha_1} \leq b_t^{\alpha_2} \leq \dots \leq b_t^{\alpha_m}, 
\]
for all $t$. These base forecasts can be constructed in any way, e.g.,
\smash{$b_t^{\alpha} = f_t^{\alpha}(x_t)$} where \smash{$f_t^{\alpha}$} is some
(possibly time-varying) predictor that optionally incorporates information from
features $x_t$ and is trained on past data $(x_s, y_s)$, $s < t$. In problems with no
base forecaster, we can set the base forecasts equal to zero
(i.e., \smash{$b_t^{\alpha} = 0$} for all $\alpha$ and $t$). If instead there
is a point forecaster that forecasts the mean or median $\mu_t$ at each time
$t$, we can set the base forecasts to this point forecast (i.e.,
\smash{$b_t^{\alpha} = \mu_t$} for all $\alpha$ and $t$).   

\medskip
We begin by presenting some relevant background on the quantile tracker, which
we then build on to present our proposed method, MultiQT.

\subsection{Background: quantile tracker} 

Given a desired coverage level $\alpha$, the \emph{quantile tracker (QT)}
method from \cite{angelopoulos2023conformal} works as follows. Given some
initial offset \smash{$\theta_1^{\alpha} \in \R$} and learning rate $\eta > 0$,
at each $t = 1,2,\dots$, we issue the adjusted forecast \smash{$q_t^{\alpha} =
  b_t^{\alpha} + \theta_t^{\alpha}$}, observe $y_t$, and then update the offset
according to:
\begin{equation}\label{eq:QT_update_rule}
\theta_{t+1}^{\alpha} = \theta_t^{\alpha} - \eta (\cov_t^{\alpha}-\alpha). 
\end{equation}
The update rule \eqref{eq:QT_update_rule} is intuitive: we increase the offset
by $\eta \alpha$ if we miscover, which makes it more likely we will cover at the
following time step, and decrease the offset by $\eta (1-\alpha)$ if we cover. 

The next result is from Proposition 1 of \cite{angelopoulos2023conformal}. 
It shows that the QT is guaranteed to achieve long-run coverage, as long as the
errors from the base forecaster are bounded.      

\begin{proposition}[\citealp{angelopoulos2023conformal}]
\label{prop:QT_guarantee}
Assume that \smash{$|y_t - b_t^{\alpha}| \leq R$} for all $t$ and some $R \geq  
0$. Then,  for all $T \geq 1$, the QT iterates \eqref{eq:QT_update_rule} satisfy the coverage error
bound
\[
\bigg|\frac{1}{T} \sum_{t=1}^T \cov^{\alpha}_t - \alpha \bigg| \leq
  \frac{2|\theta_1^{\alpha}| + R+\eta}{\eta T}.
\]
\end{proposition}

Consistent with this guarantee, the QT algorithm usually works well in
practice. Unfortunately, applying the QT updates separately to \emph{multiple} quantile
levels often results in crossed quantiles, which is undesirable. A natural
solution idea is to run QT separately for each level and then simply order the
forecasts at each time step before revealing them to the user. Two ways of enforcing ordering are by sorting the given vector of quantile forecasts, or by applying
isotonic regression. Perhaps surprisingly, neither one is able to achieve
calibration in general, as the next result shows.

\begin{proposition}
\label{prop:ordering_QT_fails}  
For a set $\cA$ of $m$ quantile levels, and for each $\alpha \in \cA$, let
\smash{$q_t^{\alpha}$} be obtained by the QT update rule 
\eqref{eq:QT_update_rule}. Given a map $G : \R^m \to \R^m$, let    
\[
\hat{\q}_t = G(q_t) \in \R^m
\]
be the vector obtained by using G to post process the vector \smash{$q_t =
  (q^{\alpha_1}, q^{\alpha_2}, \dots, q_t^{\alpha_m}) \in \R^m$} of QT forecasts
at time $t$. Then, for both \smash{$G(v) =  
  (v_{(1)}, \dots, v_{(m)})$}, which sorts the entries of its input, and 
\smash{$G(v) = \Pi_\cK(v)$}, which performs isotonic regression, there exists   
a set of quantile levels $\cA$ and sequence of target values and base
forecasts $(y_t,\b_t)$ with bounded errors (i.e., \smash{$|y_t - b_t^{\alpha}|$}
is bounded for all $\alpha$ and $t$) such that for any learning rate $\eta > 0$, there is an $\alpha \in \cA$ where  $\lim_{T  \to \infty} \frac{1}{T}
  \sum_{t=1}^T \one\{y_t \leq \hat{q}_t^{\alpha}\} \neq \alpha$ --- that is, the $\alpha$-level forecasts fail to achieve calibration.
\end{proposition}

The intuition for this result is simple: by \Cref{prop:QT_guarantee}, for each
$\alpha$, we know that the sequence \smash{$q_t^{\alpha}$} of QT iterates is
guaranteed to achieve long-run coverage $\alpha$. If we replace a nonvanishing
fraction of the values in this sequence with some arbitrary value, then we should
not expect the resulting sequence to still have coverage $\alpha$. This is
precisely what happens when crossings happen sufficiently often: whenever a crossing
occurs, applying $G$ maps \smash{$\hat{q}_t^{\alpha}$} to a value not equal to
\smash{$q_t^{\alpha}$} (under sorting, it gets mapped to \smash{$q_t^{\beta}$}
for some $\beta \neq \alpha$, and under isotonic regression, it gets mapped to
some local average). As a result, the long-run coverage of
\smash{$\hat{q}_t^{\alpha}$} will differ from that of \smash{$q_t^{\alpha}$}.
Based on this intuition, we construct a formal negative example in the appendix.  

\subsection{Multi-level quantile tracker}

We now describe our method, called the \emph{multi-level quantile tracker} 
(MultiQT), which adapts QT to the multiple quantile setting. This method is
simple and, as we will later show, has compelling theoretical guarantees and
strong empirical performance. At a high level, MultiQT maintains two vectors of
offsets: one hidden and one played. The hidden offsets, denoted
\smash{$\bttheta_t \in \R^m$}, do not generally result in ordered forecasts when
added to the base forecasts, but the played offsets, denoted $\btheta_t \in
\R^m$, do. MultiQT is described in Procedure \ref{proc:multiQT}.

\begin{procedure}
\label{proc:multiQT}
Choose some initial value \smash{$\bttheta_1 \in \R^m$} and learning rate $\eta
> 0$. For $t = 1,2,\dots$, repeat the following.   
\begin{enumerate}
\item Compute the played offset $\btheta_t = \Pi_{\cK-b_t}(\bttheta_t)$. 
\item Play the forecast $\q_t = \b_t + \btheta_t$. 
\item Observe $y_t$ and update the hidden offset: for each $\alpha \in \cA$,
\begin{equation}\label{eq:multiQT_update}
\ttheta_{t+1}^{\alpha} = \ttheta_t^{\alpha} - \eta (\cov_t^{\alpha}
- \alpha).
\end{equation}
\end{enumerate}
\end{procedure}

Note that steps 1 and 2 can be combined into a single step:
\begin{equation}\label{eq:multiQT_projection}
\q_t = \Pi_{\cK}(\b_t + \bttheta_t).
\end{equation}
This is equivalent because \smash{$\Pi_C(x+b)=b+\Pi_{C-b}(x)$} for any closed 
convex set $C \subseteq \R^d$ and vectors $x, b \in \R^d$ (where $C-b = \{x - b 
: x \in C\}$). 
Writing the MultiQT forecast $\q_t$ in this way makes it clear that it belongs
to $\cK$ and is thus distributionally consistent. However, when running MultiQT in 
practice, it is convenient to implement each iteration as
\eqref{eq:multiQT_projection} followed by \eqref{eq:multiQT_update}. When 
reasoning about its properties (calibration or regret), it is more convenient to 
use the form in Procedure \ref{proc:multiQT}. It is worth noting that each 
isotonic projection step \smash{$\Pi_{\cK}$} can be computed efficiently in
$O(m)$ time (where recall $m = |\cA|$ is the number of quantile levels), using
the pool adjacent violators algorithm (PAVA) \citep{ayer1955empirical,
  barlow1972statistical}.   

We highlight that in \eqref{eq:multiQT_update} the hidden offset vector is
updated based on the coverage induced by the played one: what appears in 
this update is \smash{$\cov_t^{\alpha} = \one\{y_t \leq b_t^{\alpha} + 
  \theta_t^{\alpha}\}$}, rather than \smash{$\one\{y_t \leq b_t^{\alpha} +  
  \ttheta_t^{\alpha}\}$}. More abstractly, the update takes a gradient step 
starting from the hidden offset but uses the gradient evaluated at the played
offset. As we will see, this combination (known generally as lazy gradient
descent) turns out to be crucial for achieving the desired calibration
guarantee.      

\paragraph{MultiQT with delayed feedback or lead time.} 
Procedure \ref{proc:multiQT} assumes that at each time $t$ we are able to
observe the outcome $y_t$ before making our next forecast $q_{t+1}$. However, there are settings where this is not the case. We model such settings using a general framework of \emph{delayed feedback}, in which the outcome associated with a forecast made at time $t$ is revealed only after the forecast is made at time $t+D$ for some constant delay $D \geq 0$. 

This framework naturally captures forecasting problems with a positive \emph{lead time}, defined as the number of time steps between forecast issuance and outcome realization. 
For example, in weekly COVID-19 death forecasting, a four-week-ahead forecast has a lead time of four and corresponds to a feedback delay of $D=3$. A lead time of one ($D = 0$) corresponds to the standard
MultiQT setting, whereas $D \geq 1$ can be understood as a delayed feedback
problem.
The lead time is also referred to as the forecast horizon.


In the delayed feedback setting with delay $D \geq 0$, we can run a modification
of MultiQT that is exactly like Procedure \ref{proc:multiQT} except the hidden
offset update in \eqref{eq:multiQT_update} is replaced with  
\begin{equation}\label{eq:multiQT_update_with_delay}   
\ttheta_{t+1}^{\alpha} = \ttheta_t^{\alpha} - \eta(\cov_{t-D}^{\alpha} -
\alpha)
\end{equation}
for $t > D$ and \smash{$\ttheta_{t+1}^{\alpha} = \ttheta_t^{\alpha}$}
for $t \leq D$. In other words, at time $t$ we update the hidden offset with
the (delayed) feedback observed at time $t$, except for the first $D$ time steps where no
feedback is observed. Compared to the original MultiQT update
\eqref{eq:multiQT_update}, the difference is that the coverage indicator used in
the above update corresponds to the forecast from time $t-D$, rather than time
$t$.     

\section{Gradient equilibrium} 
\label{sec:constrained_GEQ}

To show that MultiQT solves the problem of calibration without crossings, we
will first solve a more general problem we call \emph{constrained gradient 
  equilibrium} and then show that MultiQT is an instance of this general
solution (Figure \ref{fig:proof_structure}). Thinking about our problem at the  
more abstract gradient equilibrium level gives us a framework for cleanly
proving the desired calibration guarantee.    

\begin{figure}[H]
\centering
\begin{tikzpicture}[
  >=latex,
  every node/.style={font=\sffamily},
  concept/.style={draw, thick, rounded corners=6pt, fill=blue!5, minimum width=4.2cm, minimum height=1.1cm, align=center},
  conceptB/.style={concept, fill=green!5},
  arrow/.style={->, very thick},
  relabel/.style={font=\footnotesize}
]

\definecolor{royalblue}{HTML}{1E40AF}

\matrix[column sep=48mm, row sep=18mm]{
  \node[concept] (multiqt) {MultiQT}; & \node[concept] (lazygd) {Lazy gradient descent}; \\
  \node[conceptB] (calib) {Calibration\\without crossings}; & \node[conceptB] (geq) {Constrained GEQ}; \\
};

\draw[arrow] (multiqt.east) -- node[relabel, above, midway] {is an instance of} (lazygd.west);
\draw[arrow] (calib.east)   -- node[relabel, above, midway] {is an instance of} (geq.west);

\draw[arrow, royalblue]
  (lazygd.south) -- node[relabel, right, pos=0.52, xshift=2mm] {\shortstack{achieves$^*$\\(\Cref{prop:GEQ_lazyGD})}} (geq.north);

\draw[arrow, royalblue, dotted]
  (multiqt.south) -- node[relabel, right, pos=0.52, xshift=2mm] {\shortstack{achieves\\(inherits guarantee)}} (calib.north);

\node[overlay, anchor=north west, text=royalblue, font=\scriptsize]
  at ($(calib.south west)+(0,-.25cm)$)
  {$^*$under conditions on losses and constraints};

\end{tikzpicture}
\vspace{-10pt}
\caption{Illustration of the relationships between the MultiQT procedure, lazy 
  gradient descent, calibration without crossings, and constrained gradient
  equilibrium.}       
\label{fig:proof_structure}
\end{figure}
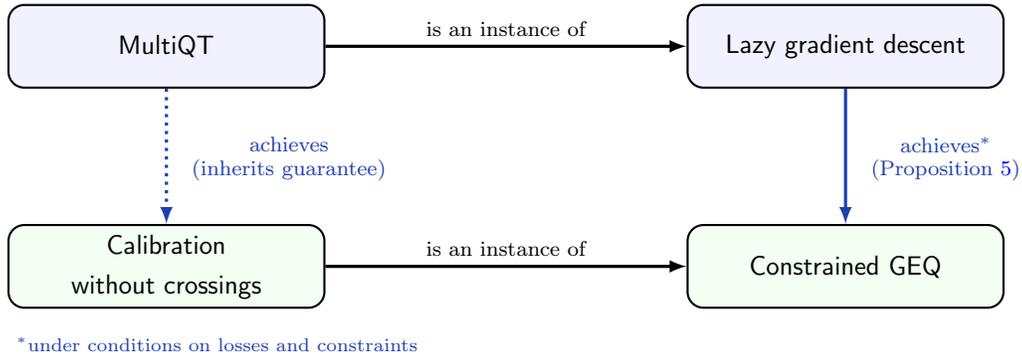

We begin by recalling the definition of gradient equilibrium, from
\cite{angelopoulos2025gradient}. 

\begin{definition}
A sequence of iterates $\btheta_t \in \R^d$, $t=1,2,\dots$ is said to satisfy
\emph{gradient equilibrium} (GEQ) with respect to a sequence of real-valued loss 
functions $\ell_t$, $t=1,2,\dots$ if   
\begin{equation}\label{eq:GEQ_definition}
\lim_{T \to \infty} \frac{1}{T} \sum_{t=1}^T \g_t(\btheta_t) = \mathbf{0},
\end{equation}
where each $\g_t(\theta)$ is a gradient (or subgradient) of $\ell_t$ at
$\theta_t$, assumed to be differentiable (or subdifferentiable) on its  
domain, and $\mathbf{0}$ is the $d$-dimensional zero vector. 
\end{definition}

There are many problems in online learning in which the iterates should be 
restricted to constraint sets (which may vary over time). This motivates
the following definition.

\begin{definition} 
A sequence of iterates $\btheta_t \in \R^d$, $t=1,2,\dots$ is said to satisfy
\emph{constrained gradient equilibrium} (constrained GEQ) with respect to a
sequence of loss functions $\ell_t$, $t=1,2,\dots$ and sets $C_t \subseteq
\R^d$, $t=1,2,\dots$ if \GEQ{} \eqref{eq:GEQ_definition} holds and,
additionally, $\theta_t \in C_t$ for all $t = 1,2,\dots$.   
\end{definition}

It is worth noting that, in optimization, it is common to reformulate a constraint set $C$ via a
characteristic function, denoted $I_C$ (zero on $C$ and $\infty$
otherwise) that is added to the loss $\ell_t$ and then treated as an unconstrained problem. However,  constrained \GEQ{} as we define it here is
\emph{not} the same as \GEQ{} with respect to the modified loss sequence
\smash{$\ell_t + I_{C_t}$}. This is an important distinction that we will revisit shortly.

\paragraph{Calibration without crossings as constrained GEQ.}

For $\alpha \in [0,1]$, let $\rho_{\alpha} : \R \times \R \to \R$ be the
$\alpha$-level quantile loss, where    
\begin{equation}\label{eq:quantile_loss}
\rho_{\alpha}(\hat y, y) = \begin{cases}
\alpha |y-\hat y| & \text{if $y \geq \hat y$} \\
(1-\alpha) |y - \hat y| & \text{otherwise}. 
\end{cases}
\end{equation}
Given a set of levels $\cA$ with $m = |\cA|$, and a vector of forecasts 
$\q \in \R^m$ at these levels, let \smash{$\rho_{\cA} : \R^m \times \R 
  \to \R$} be the aggregated quantile loss, where   
\begin{equation}\label{eq:aggregated_quantile_loss}
\rho_{\cA}(\q, y) = \sum_{\alpha \in \cA} \rho_{\alpha}(q^{\alpha}, y).
\end{equation}
Now, for each $t$, define a loss function $\ell_t$ on $\btheta_t \in \R^m$ that 
applies the aggregated quantile loss to $\q_t = \b_t + \btheta_t$ and $y_t$,
where $\b_t \in \R^m$ is a vector of base forecasts:
\begin{equation}\label{eq:multiqt_loss}
\ell_t(\btheta_t) = \rho_{\cA}(\b_t + \btheta_t, y_t) 
= \sum_{\alpha \in \cA} \rho_{\alpha}(b_t^{\alpha} + \theta_t^{\alpha}, y_t).  
\end{equation}
We will call \eqref{eq:multiqt_loss} the \emph{MultiQT loss}. A subgradient of
the MultiQT loss at $\btheta_t$ is 
\begin{equation}\label{eq:multiqt_gradient} 
\g_t(\btheta_t) =
\begin{bmatrix}   
\cov_t^{\alpha_1} - \alpha_1\\
\cov_t^{\alpha_2} - \alpha_2\\
\vdots \\    
\cov_t^{\alpha_m} - \alpha_m
\end{bmatrix}
\end{equation}
where we recall that \smash{$\cov_t^{\alpha} = \one\{y_t \leq q_t^{\alpha}\} = 
  \one\{y_t \leq b_t^{\alpha} + \theta_t^{\alpha}\}$}. To streamline
presentation, we will often refer to \eqref{eq:multiqt_gradient} as the
``gradient'' of the MultiQT loss. We now observe the following equivalence:    
\[
\lim_{T \to \infty} \frac{1}{T} \sum_{t=1}^T \g_t(\btheta_t) = \mathbf{0} \quad \iff
\quad \lim_{T \to \infty} \frac{1}{T} \sum_{t=1}^T \cov^{\alpha}_t = \alpha, \; 
\text{for all $\alpha \in \cA$}.   
\]
In other words, for a sequence of quantile forecasts $\q_t$, $ t= 1,2,\dots$ to 
be calibrated, it suffices to show that the offsets $\btheta_t$, $t = 1,2,\dots$
satisfy \GEQ{} with respect to the MultiQT loss defined in
\eqref{eq:multiqt_loss}.  

Furthermore, recall that our goal is to derive offsets that, once added to the base forecasts, yield forecasts that are not only calibrated but are also distributionally consistent. Setting the constraint set at time $t$ as
\begin{equation} \label{eq:multiqt_constraint}
C_t = \cK - \b_t,
\end{equation}
which is the isotonic cone shifted by the base forecast $\b_t$, ensures that the 
resulting forecast $\q_t$ does not have crossed quantiles. Thus, calibration 
without crossings is an instance of constrained \GEQ{}, for the losses and
constraints defined above.

\subsection{Background: gradient descent} 
\label{sec:gradient_descent}

Before solving the constrained \GEQ{} problem, we first consider the (unconstrained) \GEQ{} problem.
It turns out that we do not need to devise new algorithms in order to produce
iterates that satisfy \GEQ{}. Online gradient descent, a standard
algorithm in online learning, can also be used to solve the \GEQ{}
problem.

\paragraph{Gradient descent achieves GEQ.}

Given some initial point $\btheta_1 \in \R^d$ and learning rate $\eta > 0$, 
recall that \emph{online gradient descent}, which we will often simply call
gradient descent (GD), obtains iterates via 
\begin{equation}\label{eq:OGD_update}
\btheta_{t+1} = \btheta_t - \eta \g_t(\btheta_t),
\end{equation}
for $t=1,2,\dots$, where $\g_t(\btheta_t)$ is a (sub)gradient of the loss at 
$\btheta_t$. As explained in \cite{angelopoulos2025gradient}, we can rearrange   
\eqref{eq:OGD_update} to get $\g_t(\theta_t) = (\btheta_t - \btheta_{t+1})/\eta$
and then average over $t$ to yield  
\begin{equation}\label{eq:OGD_avg_gradient}
\frac{1}{T} \sum_{t=1}^T \g_t(\btheta_t) = \frac{\btheta_1 - \btheta_{T+1}}
{\eta T}.  
\end{equation}
Since $\btheta_1$ is chosen by us, it is bounded. Thus if we can bound
$\btheta_{T+1}$, this would imply a bound on the average gradient.   
\cite{angelopoulos2025gradient} show that a sufficient condition for
$\theta_{T+1}$ to be bounded or sublinear in $T$ is for the loss functions to
satisfy two conditions. The first is Lipschitzness; this is a standard 
condition, and we recall that a loss $\ell$ is said to be \emph{L-Lipschitz} if
for all $\btheta$, all of its subgradients $\g(\btheta)$ satisfy
$\|g(\theta)\|_2 \leq L$. The second is a new condition that they call restorativity,  which we describe below.     

\begin{definition}
A loss $\ell$ is said to be \emph{$(h, \phi)$-restorative}, for a constant $h 
\geq 0$ and nonnegative function $\phi$, if it has a subgradient $\g(\btheta)$
at each $\btheta$ which satisfies    
\begin{equation}\label{eq:restorativity}
\langle \btheta, \g(\btheta) \rangle \geq \phi(\btheta), \quad \text{whenever  
  $\|\btheta\|_2 > h$}, 
\end{equation}
where $\langle \u, \v \rangle = \u^\T \v$.
\end{definition}

Intuitively, restorativity \eqref{eq:restorativity} tells us that whenever the
iterates get too far from the origin, the negative gradient will push the
iterate back towards the origin. This can be seen most easily in the
one-dimensional setting where $\theta \in \R$ and $\phi(\theta) = 0$: in this
case, restorativity says that $\sign(\theta) = \sign(g(\theta))$ whenever
$|\theta| \geq h$. In other words, if $\theta$ is large in magnitude, then the
negative gradient will be anti-aligned with it, so following the negative
gradient will decrease the magnitude of $\theta$.

This intuition is formalized in the following result, which appears as
Proposition 5 in \cite{angelopoulos2025gradient}. It says that \GD{} produces 
iterates that grow slowly when the losses are restorative.

\begin{proposition}[\citealp{angelopoulos2025gradient}]
\label{prop:prop5_GEQ}
Assume that for each $t$, the loss function $\ell_t$ is $L$-Lipschitz and
$(h_t,0)$-restorative. Then, for all $T \geq 1$, the \GD{} iterates produced according to  
\eqref{eq:OGD_update} satisfy 
\[
\|\btheta_{T+1}\|_2 \leq \sqrt{\|\btheta_1\|_2^2 + \eta^2 L^2 T + 2 \eta L
  \sum_{t=1}^T h_t}.
\]
If $h_t$ is nondecreasing, then this implies 
\[
\bigg\| \frac{1}{T} \sum_{t=1}^T \g_t(\theta_t) \bigg\|_2 \leq
\frac{2\|\btheta_1\|_2}{\eta T} + \sqrt{\frac{L^2}{T} + \frac{2Lh_T}{\eta T}}, 
\]
which goes to zero as $T \to \infty$ as long as $h_T$ is sublinear in $T$. 
\end{proposition}

When all loss functions $\ell_t$ are $(h,0)$-restorative, note that the rate at which \GEQ{}
is obtained in \Cref{prop:prop5_GEQ} is \smash{$1/\sqrt{T}$}.
\cite{angelopoulos2025gradient} show that \GEQ{} rates of order $1/T$ are
possible under stronger conditions, including when $\phi$ is lower bounded by a
positive constant (rather than zero, as assumed in \Cref{prop:prop5_GEQ}). They
also show that the one-dimensional case is special: when $d=1$, $L$-Lipschitzness and $(h,0)$-restorativity are sufficient for achieving the fast $1/T$ rate.

By the calculations at the start of this section 
\eqref{eq:quantile_loss}--\eqref{eq:multiqt_gradient} specialized to the
singleton $\cA = \{\alpha\}$, we can see that the QT update
\eqref{eq:QT_update_rule} is simply online gradient descent with respect  
to the loss \smash{$\ell_t(\theta) = \rho_{\alpha}(b_t^{\alpha}+\theta, y_t)$}.
This loss is Lipschitz with $L=1$ and, under bounded errors $|y_t -
b_t^{\alpha}| \leq R$, it can be shown to be $(R,0)$-restorative. Hence, 
the calibration guarantee for QT can be derived as a special case of
\Cref{prop:prop5_GEQ}. Indeed, the exact result in \Cref{prop:QT_guarantee}
(which shows calibration is achieved at the rate $1/T$) can be derived as a
special case of the one-dimensional refinement of \Cref{prop:prop5_GEQ}. We
refer to Corollary 1 of \cite{angelopoulos2025gradient}.  

\paragraph{Projected gradient descent does not achieve constrained GEQ.}

Perhaps the most common way to enforce iterate constraints is via projection.  
Now that we have seen \GD{} achieves gradient equilibrium (under some mild  
conditions), a natural first guess for achieving constrained gradient
equilibrium would be to run projected \GD{}.  Given closed convex constraint sets
$C_t$, $t=1,2,\dots$, an initial $\theta_1 \in C_1$, and learning rate
$\eta > 0$, \emph{projected \GD{}} obtains iterates via the update rule:
\begin{equation}\label{eq:PGD_update}
\btheta_{t+1} = \Pi_{C_{t+1}}(\btheta_t - \eta \g_t(\btheta_t)),
\end{equation}
for $t = 1,2,\dots$. Unfortunately, projected \GD{} does not guarantee
constrained \GEQ{} in general, and, in fact, provably fails to achieve our goal
of calibration without crossings.

To see why, first observe that we can view projected \GD{} as ordinary \GD{} on
the modified loss sequence \smash{$\tilde\ell_t = \ell_t + I_{C_t}$}, where 
\[
I_{C_t}(\btheta) = \begin{cases}
0 & \text{if $\btheta \in C_t$} \\ 
\infty & \text{otherwise}
\end{cases}
\]
is the characteristic function of $C_t$. Subgradients of \smash{$\tilde\ell_t$}
at $\btheta$ are of the form \smash{$\tilde{\g}_t = \g_t(\btheta) +
  \v_t(\btheta)$}, where $\v_t(\btheta)$ is a subgradient of \smash{$I_{C_t}$}
at $\btheta$, i.e., an element of the normal cone of $C_t$ at $\btheta$. 

Next, as noted in Appendix B of \cite{angelopoulos2025gradient}, due to the \GEQ{} guarantee of gradient descent (\Cref{prop:prop5_GEQ}), the projected gradient descent iterates will achieve \GEQ{} with respect to this modified sequence, meaning 
\smash{$\lim_{T \to \infty} \frac{1}{T} \sum_{t=1}^T \tilde{\g}_t(\btheta_t) =
  0$}.
This implies  
\[
\lim_{T \to \infty} \frac{1}{T} \sum_{t=1}^T \g_t(\btheta_t) = - \lim_{T \to
  \infty} \frac{1}{T} \sum_{t=1}^T \v_t(\btheta_t).  
\]
Since the right-hand side is not zero in general,
projected \GD{} is not guaranteed to achieve constrained \GEQ{}. The next 
proposition goes further and shows that projected \GD{} provably fails to solve
our calibration without crossings problem.

\begin{proposition}
\label{prop:PGD_fails}
There exists a set of levels $\cA$ and sequence of target values and base
forecasts $(y_t,\b_t)$ with bounded errors (i.e., \smash{$|y_t - b_t^{\alpha}|$} is
bounded for all $\alpha$ and $t$) such that for any learning rate $\eta > 0$,
projected \GD{} \eqref{eq:PGD_update}, with the gradient $\g_t$ as defined in   
\eqref{eq:multiqt_gradient} and constraint set $C_t$ as defined in
\eqref{eq:multiqt_constraint}, fails to achieve calibration: \smash{$\lim_{T  \to
    \infty} \frac{1}{T} \sum_{t=1}^T \one\{y_t \leq b_t^{\alpha} +
  \theta_t^{\alpha}\} \neq \alpha$} for some $\alpha \in \cA$. 
\end{proposition}

\subsection{Lazy gradient descent}

As we saw above, incorporating constraints into \GD{} in the standard way (via
direct projection at each step), fails to achieve constrained \GEQ{}. Another way
to incorporate constraints into online gradient descent is to use what are known
as lazy updates \citep{shalev2012online, hazan2019introduction}. Presented with the same constraint sets as in projected \GD{}, we implement lazy updates by maintaining two sequences: a
hidden sequence \smash{$\bttheta_t$} and a played sequence $\btheta_t$. 
More precisely, \emph{lazy online gradient descent} (lazy GD) begins with an initial point $\bttheta_1 \in C_1$ and learning rate $\eta > 0$, and obtains iterates via a two-step procedure: 
\begin{align}  
\label{eq:lazyGD_projection} 
\btheta_{t} &= \Pi_{C_{t}}(\bttheta_{t}), \\
\bttheta_{t+1} &= \bttheta_t - \eta \g_t(\btheta_t). 
 \label{eq:lazyGD_update}
\end{align}
``Lazy'' refers to how \eqref{eq:lazyGD_update} takes the gradient step starting 
from the hidden iterate \smash{$\bttheta_t$} rather than the played iterate 
$\btheta_t$. If we instead took the gradient step starting
from $\btheta_t$ in \eqref{eq:lazyGD_update}, then this would be equivalent to
projected \GD{} \eqref{eq:PGD_update}. 

The utility of the lazy updates can be seen immediately; by rearranging
\eqref{eq:lazyGD_update} and averaging over $t$, just as in unconstrained
gradient descent, we get
\begin{equation}
\label{eq:lazyGD_avg_gradient}
\frac{1}{T} \sum_{t=1}^T \g_t(\btheta_t) = \frac{\bttheta_1 - \bttheta_{T+1}}  
{\eta T}. 
\end{equation}
This calculation leverages the fact that in lazy \GD{} we have effectively
decoupled the updates of the hidden iterates from projection, and thus to track
the average gradient, we can track total movement in the hidden sequence. The
only difference from the previous result \eqref{eq:OGD_avg_gradient} for ordinary
\GD{} is that the hidden iterates appear on the right-hand side in
\eqref{eq:lazyGD_avg_gradient}, rather than the played iterates. 

Thus, to bound the average gradients of lazy gradient descent, we want to bound the hidden iterates. Recall that to control the growth of the played iterates in standard gradient descent, we controlled the inner product between $\btheta_t$ and $\g_t(\btheta_t)$ via restorativity. To control the growth of the hidden iterates of lazy gradient descent, the relevant inner product is now between \smash{$\bttheta_t$} and $\g_t(\btheta_t)$.
Note carefully that we seek to measure the alignment of an
iterate \smash{$\bttheta_t$} with the gradient at a \emph{different} point: the
result $\btheta_t$ of projection onto the constraint set $C_t$. Restorativity of
the loss alone is not sufficient for this purpose. We need to introduce an
additional condition that controls the joint behavior of the loss and the
constraint set.

\begin{definition}
For a loss $\ell$ and set $C$, the pair $(\ell, C)$ is said to satisfy
\emph{inward flow} if there is a subgradient $\g(\btheta)$ at each $\btheta$
that satisfies    
\begin{equation}
\label{eq:inward_flow}
-\g(\btheta) \in T_C(\btheta), \quad \text{for $\btheta\in\mathrm{bd}(C)$}, 
\end{equation}
where $T_C(x)$ denotes the tangent cone of $C$ at $x$, defined as
\[
T_C(x) = \mathrm{cl}\big\{ y : \text{there exists $\delta > 0$ such that 
  $x+\varepsilon y \in C$ for all $\varepsilon \in (0,\delta]$} \big\}.
\]
In the above, $\mathrm{bd}(S)$ denotes the boundary of a set $S$, and
$\mathrm{cl}(S)$ denotes the closure of $S$.
\end{definition}

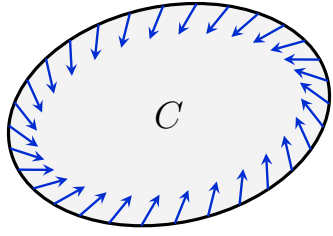
\begin{figure}[t]
\centering
\begin{tikzpicture}[scale=2.4, >=stealth]

\def\a{.9} 
\def\b{.6} 
\def\rot{12} 

\begin{scope}[rotate=\rot]
  \draw[fill=gray!20, fill opacity=0.45, draw=black, line width=1.2pt]
  (0,0) ellipse ({\a} and {\b});

  \foreach \th in {0,12,...,348} {
    \pgfmathsetmacro{\x}{\a*cos(\th)}
    \pgfmathsetmacro{\y}{\b*sin(\th)}

    \pgfmathsetmacro{\nx}{- \x/(\a*\a)}
    \pgfmathsetmacro{\ny}{- \y/(\b*\b)}
    \pgfmathsetmacro{\nlen}{sqrt(\nx*\nx+\ny*\ny)}
    \pgfmathsetmacro{\ux}{\nx/\nlen}
    \pgfmathsetmacro{\uy}{\ny/\nlen}

    \pgfmathsetmacro{\txraw}{- \a*sin(\th)}
    \pgfmathsetmacro{\tyraw}{  \b*cos(\th)}
    \pgfmathsetmacro{\tlen}{sqrt(\txraw*\txraw+\tyraw*\tyraw)}
    \pgfmathsetmacro{\tx}{\txraw/\tlen}
    \pgfmathsetmacro{\ty}{\tyraw/\tlen}

    \pgfmathsetmacro{\alpha}{0.95}  
    \pgfmathsetmacro{\beta}{0.55}   

    \pgfmathsetmacro{\vx}{\alpha*\ux + \beta*\tx}
    \pgfmathsetmacro{\vy}{\alpha*\uy + \beta*\ty}
    \pgfmathsetmacro{\vlen}{sqrt(\vx*\vx+\vy*\vy)}
    \pgfmathsetmacro{\vux}{\vx/\vlen}
    \pgfmathsetmacro{\vuy}{\vy/\vlen}

    \pgfmathsetmacro{\L}{0.2}

    \draw[->, line width=0.9pt, blue!65!teal]
      (\x,\y) -- ({\x + \L*\vux}, {\y + \L*\vuy});
  }

  \node at (0,0) {\Large $C$};
\end{scope}

\end{tikzpicture}
\vspace{-25pt}
\caption{Visualization of inward flow, with the arrows representing the negative 
  gradient $-g$ evaluated at different points on the boundary of the constraint
  set $C$.}      
\label{fig:inward_flow}
\end{figure}

Informally, inward flow \eqref{eq:inward_flow} says that if we start at any
$\btheta$ on the boundary of the constraint set and take an arbitrarily small
step in the direction of the negative gradient, then this will keep us within
the constraint set. In other words, following the direction of steepest descent
will lead us inward, ``flowing'' further into the constraint set. Figure
\ref{fig:inward_flow} provides a visualization.  


Combining both restorativity and inward flow, we are able to establish the
following result.   

\begin{proposition}
\label{prop:GEQ_lazyGD}
Assume that for each $t$, the loss function $\ell_t$ is $L$-Lipschitz and
$(h_t,0)$-restorative, the set $C_t$ is closed and convex, and the pair
$(\ell_t, C_t)$ satisfies inward flow. Then, for all $T \geq
1$,  the lazy \GD{} iterates produced by
\eqref{eq:lazyGD_projection} and \eqref{eq:lazyGD_update} satisfy 
\[
\|\bttheta_{T+1}\|_2 \leq \sqrt{\|\bttheta_1\|_2^2 + \eta^2 L^2 T + 2 \eta L
  \sum_{t=1}^{T} h_t}.
\]
If $h_t$ is nondecreasing, then this implies  
\[
\bigg\| \frac{1}{T} \sum_{t=1}^T \g_t(\btheta_t) \bigg\|_2 \leq
\frac{2\|\bttheta_1\|_2 }{\eta T} + \sqrt{\frac{L^2}{T} + \frac{2 L h_T}{\eta
    T}}, 
\]
which goes to zero as $T \to \infty$ as long as $h_T$ is sublinear in $T$.
\end{proposition}

We remark that \Cref{prop:prop5_GEQ} can be recovered as a special case of the   
above result: \GD{} is an instance of lazy \GD{} where the constraint set at 
every time $t$ is $C_t = \R^d$, thus the projection is the identity, and 
\smash{$\btheta_t = \bttheta_t$}. Since $\R^d$ has no boundary, inward flow is 
trivially satisfied, and in this way \Cref{prop:GEQ_lazyGD} exactly reduces to 
\Cref{prop:prop5_GEQ}.        

The above result can also be extended to the delayed feedback setting. With
delay $D \geq 0$, we can generalize lazy \GD{} and update the hidden iterates
according to:
\begin{equation}
\label{eq:lazyGD_update_with_delay}
\bttheta_{t+1} = \bttheta_t - \eta g_{t-D}(\btheta_{t-D}), 
\end{equation}
where we set $\g_t(\btheta_t)= 0$ for $ t \leq 0$. We maintain the projection
step \eqref{eq:lazyGD_projection} for obtaining the played iterates.     

\begin{proposition}
\label{prop:GEQ_lazyGD_with_delay}
Under the conditions of \Cref{prop:GEQ_lazyGD}, in the setting with feedback delay $D \geq 0$, the lazy \GD{} iterates produced by \eqref{eq:lazyGD_projection} and \eqref{eq:lazyGD_update_with_delay} satisfy    
\[
\|\bttheta_{T+D+1}\|_2 \leq \sqrt{\|\bttheta_1\|_2^2 + \eta^2 L^2 (2D+1)T + 
2 \eta L \sum_{t=1}^T h_t}.
\]
If $h_t$ is nondecreasing, then this implies 
\[
\bigg \|\frac{1}{T} \sum_{t=1}^T \g_t(\btheta_t) \bigg\|_2 \leq
\frac{2\|\bttheta_1\|_2 }{\eta T} + \sqrt{\frac{L^2(2D+1)}{T} + \frac{2 L
    h_T}{\eta T}},
 \]
which goes to zero as $T \to \infty$ as long as $h_T$ is sublinear in $T$. 
\end{proposition}

\paragraph{Summary of positive and negative results on constrained GEQ.} 

In Table \ref{tab:method_comparison} we review the results we have
established thus far on constrained \GEQ{} from Propositions
\ref{prop:ordering_QT_fails}, \ref{prop:PGD_fails}, and \ref{prop:GEQ_lazyGD}. 
Post hoc projection is the method from \Cref{prop:ordering_QT_fails}, with $G =
\Pi_\cK$; recall in that proposition we showed it fails to attain calibration
without crossings, and hence it fails to achieve constrained \GEQ{} in general.
Projected \GD{} similarly fails according to \Cref{prop:PGD_fails}, whereas
lazy \GD{} achieves constrained \GEQ{} (under restorativity and inward flow) by   
\Cref{prop:GEQ_lazyGD}.  

\begin{table}[H]
\caption{Summary of results on online methods that incorporate gradient
  updates and projection onto constraints. Note that the post hoc projection and
  projected \GD{} methods are rewritten here using a hidden sequence to make the differences between methods more salient.}  
\label{tab:method_comparison}
\centering 
\definecolor{lightgray}{RGB}{245,245,245}
\tcbset{box align=base}
\renewcommand{\arraystretch}{1.4}
\setlength{\tabcolsep}{8pt}
\begin{tabular}{>{\RaggedRight\arraybackslash}p{3.5cm} c c c}
\toprule
 & \textbf{\underline{Post hoc projection}} &
   \textbf{\underline{Projected GD}} &
   \textbf{\underline{Lazy GD}} \\ [0.5em]  

\textbf{Projection} &
\multicolumn{3}{c}{
  \tcbox[
    colframe=black!70,
    colback=black!10,
    boxrule=0.5pt,
    arc=2pt,
    left=136pt, right=136pt, top=1pt, bottom=1pt,
    boxsep=1pt,
    width=\dimexpr\linewidth-2\tabcolsep\relax
  ]{$\theta_t = \Pi_{C_t}(\tilde{\theta}_t)$}
} \\
[0.4em]  

\textbf{Hidden update} &
\tcbox[
  colframe=black!70,
  colback=lightgray,
  boxrule=0.5pt,
  arc=2pt,
  left=4pt, right=4pt, top=0pt, bottom=0pt
]{$\tilde{\theta}_{t+1} = \tilde{\theta}_t - \eta_t g_t(\tilde{\theta}_t)$}
&
\tcbox[
  colframe=black!70,
  colback=lightgray,
  boxrule=0.5pt,
  arc=2pt,
  left=4pt, right=4pt, top=0pt, bottom=0pt
]{$\tilde{\theta}_{t+1} = \theta_t - \eta_t g_t(\theta_t)$}
&
\tcbox[
  colframe=black!70,
  colback=lightgray,
  boxrule=0.5pt,
  arc=2pt,
  left=4pt, right=4pt, top=0pt, bottom=0pt
]{$\tilde{\theta}_{t+1} = \tilde{\theta}_t - \eta_t g_t(\theta_t)$}
\\



\textbf{Enforces constraint?} & 
\tick & \tick & \tick \\

\textbf{Achieves GEQ?} &   
\cross & \cross & $\;\;\tick^*$ \\
\bottomrule
\end{tabular}

\begin{tablenotes}
\footnotesize 
\item[] *under restorativity and inward flow 
\end{tablenotes}
\end{table}

Why does lazy \GD{} succeed in achieving constrained \GEQ{}, while the other two methods
fail? Here we give some intuition. First, \emph{post hoc projection discards
  current information}, as feedback from the played iterate is never
incorporated into subsequent updates. Specifically, the observed gradient  
$\g_t(\btheta_t)$ is not used to inform future updates, and this turns out to be 
problematic when the goal is to drive the average of such gradients to zero.  
Second, \emph{projected gradient descent discards past information}, since the
hidden update does not depend on the hidden iterate \smash{$\bttheta_t$}. The
hidden iterate encodes accumulated knowledge about the gradients of interest
over the sequence thus far (recall \eqref{eq:lazyGD_avg_gradient}), and
``forgetting'' this information again turns out to be problematic. Observe that 
\emph{lazy gradient descent combines both sources of information:} it preserves 
past knowledge by starting its update from \smash{$\bttheta_t$}, while at the
same time incorporating current feedback via the gradient evaluated at
\smash{$\btheta_t$}. This blend of retaining history and responding to 
present information is what distinguishes lazy gradient descent and enables it
to achieve constrained gradient equilibrium.    


\section{MultiQT theory}
\label{sec:multiQT_guarantees}

Having introduced the framework of constrained gradient equilibrium in the last
section, we are now ready to present the theoretical guarantees for
MultiQT. Recall the proof roadmap illustrated in Figure
\ref{fig:proof_structure}, and note that we have already shown (i) calibration
without crossings is an instance of constrained \GEQ{}, and (ii) lazy \GD{}
solves constrained \GEQ{} problems that satisfy Lipschitz, restorativity, and
inward flow conditions.  What remains to be shown are (iii) MultiQT is the
appropriate instantiation of lazy \GD{} for the problem of calibration without
crossings, and (iv) the calibration without crossings problem satisfies
the needed conditions (Lipschitz, restorativity, and inward flow). We address
(iii) in the next paragraph, and (iv) in the following subsection. After this, a
calibration guarantee for MultiQT will follow directly from the \GEQ{} theory developed in the previous section. 

\paragraph{MultiQT is lazy gradient descent.} 
 
This follows directly from the calculations already given at the start of
Section \ref{sec:constrained_GEQ}. Referring back to Procedure
\ref{proc:multiQT}, we can see that the hidden update \eqref{eq:multiQT_update}
is equivalent to \eqref{eq:lazyGD_update} with respect to the MultiQT gradient 
in \eqref{eq:multiqt_gradient}, and the played update is equivalent to
projection onto $C_t$ as in \eqref{eq:multiqt_constraint} (indexed slightly
differently because Procedure \ref{proc:multiQT} is clearer in the
context of forecasting).  

\subsection{Calibration guarantee}

It remains to show that the MultiQT losses and constraints satisfy the
Lipschitz, restorativity, and inward flow conditions. Lipschitzness is
straightforward to show: examining the gradient in \eqref{eq:multiqt_gradient},
we see that
\[ 
\|\g_t(\btheta)\|_2 = \sqrt{\sum_{\alpha \in \cA} (\cov_t^{\alpha} - \alpha)^2} 
\leq \sqrt{\sum_{\alpha \in \cA} 1} = \sqrt{m}, 
\]
so each loss is $\sqrt{m}$-Lipschitz. 

The second condition, restorativity, is satisfied by the MultiQT losses as long
as the errors between the base forecast and target values are bounded, as the 
following lemma establishes.

\begin{lemma}
\label{lemma:multiqt_cond1}
Assume that \smash{$|y_t - b_t^{\alpha}| \leq R$} for all $\alpha$.
Let \smash{$d_{\cA} = \min_{\alpha \in\cA} \min\{\alpha, 1-\alpha\}$} be the 
minimum distance between any level in $\cA$ and the boundary of $[0,1]$.    
Then the MultiQT loss defined in \eqref{eq:multiqt_loss} is $(h,0)$-restorative
for any \smash{$h \geq Rm^{3/2} / d_{\cA}$}, where recall $m = |\cA|$.
\end{lemma}

We next verify the third condition, inward flow, for the MultiQT loss and (shifted)
isotonic cone. 

\begin{lemma}
\label{lemma:multiqt_cond2} 
The MultiQT loss defined in \eqref{eq:multiqt_loss}, with gradient in
\eqref{eq:multiqt_gradient}, and the constraint $C_t$ defined in
\eqref{eq:multiqt_constraint}, together satisfy inward flow:
\smash{$-\g_t(\btheta) \in T_{C_t}(\btheta)$} for all $\btheta$ on the boundary 
of $C_t$. 
\end{lemma}

Before moving on, we reflect on the above two lemmas. While restorativity of the 
MultiQT loss is to be expected based on results from the single quantile case
\citep{angelopoulos2025gradient}, the fact that the MultiQT loss satisfies
inward flow over the isotonic cone is perhaps more surprising. Inward flow is
highly nontrivial, and requires the gradients of the loss to ``cooperate'' with
the geometry of the constraint set. This can fail to hold even in seemingly 
natural optimization problems, which is a point we return to in Section \ref{sec:discussion}. 

We now state the main result for MultiQT, which follows directly from the
previous results.  

\begin{theorem}
\label{thm:calibration}
Assume that \smash{$|y_t - b_t^{\alpha}| \leq R$} for all $\alpha,t$. Then, for all $T \geq 1$, the
MultiQT iterates from Procedure \ref{proc:multiQT} satisfy the $\ell_2$ coverage
error bound     
\[
\sqrt{\sum_{\alpha \in \cA} \bigg( \frac{1}{T} \sum_{t=1}^T \cov_t^{\alpha} -
  \alpha \bigg)^2} \leq \frac{2\|\bttheta_1\|_2}{\eta T} + \sqrt{\frac{m}{T} +
  \frac{2Rm^2}{\eta d_{\cA}T}},
\]
where recall $m = |\cA|$, and \smash{$d_{\cA}$} is as in
\Cref{lemma:multiqt_cond1}. Since $\|x\|_\infty \leq \|x\|_2$ for any vector $x \in
\R^m$, the right-hand side above is also an upper bound for \smash{$\big| \frac{1}{T}
  \sum_{t=1}^T \cov_t^{\alpha} - \alpha \big|$}, for each $\alpha \in \cA$.
\end{theorem}

\begin{proof}
We apply \Cref{prop:GEQ_lazyGD} to the current problem setting. Its conditions 
are verified by Lemmas \ref{lemma:multiqt_cond1} and \ref{lemma:multiqt_cond2},
with \smash{$L = \sqrt{m}$} and \smash{$h = Rm^{3/2} / d_{\cA}$}. 
\end{proof}

\Cref{thm:calibration} tells us that MultiQT is guaranteed to achieve
calibration, as described in \eqref{eq:long_run_coverage}. Moreover, the
projection step ensures that the forecasts satisfy the distributional
consistency property from \eqref{eq:no_crossing}. Thus, we have shown that the 
MultiQT method is guaranteed to satisfy our initial desiderata. This is true
even in the delayed feedback setting, as the next generalization shows. 

\begin{theorem}
\label{thm:calibration_with_delay}
Under the conditions of \Cref{thm:calibration}, in the setting with feedback delay $D \geq 0$, the MultiQT iterates from Procedure
\ref{proc:multiQT} with the modification in \eqref{eq:multiQT_update_with_delay}
satisfy   
\[
\sqrt{\sum_{\alpha \in \cA} \bigg( \frac{1}{T} \sum_{t=1}^T \cov_t^{\alpha} -
  \alpha \bigg)^2} \leq \frac{2\|\bttheta_1\|_2}{\eta T} +
\sqrt{\frac{m(2D+1)}{T} + \frac{2Rm^2}{\eta d_{\cA}T}}.
\]
\end{theorem}

\begin{proof}
We apply \Cref{prop:GEQ_lazyGD_with_delay} to the current problem
setting. Its conditions are again verified by Lemmas \ref{lemma:multiqt_cond1}
and \ref{lemma:multiqt_cond2}.   
\end{proof}

We can see from the above theorem that the bound worsens with 
increasing delay, which is unsurprising. 

\subsection{Regret guarantee}
\label{sec:regret_guarantee}

In this subsection, we provide a regret guarantee for MultiQT with respect to the MultiQT loss (the aggregated quantile loss). Beyond being the loss we take gradient updates with respect to, its relevance can be motivated by
the fact that it admits a decomposition into terms that can be interpreted as
emphasizing calibration and sharpness. To be precise, suppose that the set $\cA$
of quantile levels is symmetric around 1/2, and can therefore be written as
\[
\cA = \bigcup_{\beta \in \cB} \{\beta/2,\, 1 - \beta/2\},
\]
for some set $\cB \subset [0, 1/2]$. Then it can be shown \citep{bracher2021evaluating} that the
aggregated quantile loss \smash{$\rho_{\cA}$} in
\eqref{eq:aggregated_quantile_loss} has the following alternative
representation: 
\begin{equation}
\label{eq:quantile_loss_decomp}
\rho_{\cA}(\q, y) = \sum_{\beta \in \cB} \bigg[ \underbrace{\mathrm{dist}(y,  
  [q_{\beta/2}, q_{1-\beta/2}]) \vphantom{\frac{1}{2}}}_{\text{``calibration''}}  
  \,+\, \underbrace{\frac{\beta}{2} (q_{1-\beta/2} -
  q_{\beta/2})}_{\text{``sharpness''}} \bigg],    
\end{equation}
where $\mathrm{dist}(y,
[a,b])$ is zero if $y$ lies inside $[a,b]$, and otherwise equals the distance 
to the closest endpoint.
The expression on the right-hand side corresponds to the \emph{weighted interval score} of a collection of equi-tailed prediction intervals \smash{$[q_{\beta/2},
q_{1-\beta/2}]$}, $\beta \in \cB$. In each summand, the first term---which
measures the distance of the target $y$ to the interval \smash{$[q_{\beta/2},
  q_{1-\beta/2}]$}---can be interpreted as a calibration penalty,
whereas the second term---which measures the length of the interval---can be interpreted as a sharpness penalty. 

Given that we have already shown via \GEQ{} theory that the MultiQT method achieves
calibration (which as we have defined it, means long-run coverage per quantile
level), the regret theory below can be interpreted in light of the decomposition
\eqref{eq:quantile_loss_decomp} as saying that MultiQT-adjusted forecasts also
encourage sharpness, i.e., they give rise to prediction intervals that are as
concentrated and informative as possible. Moreover, the quantile loss is a
proper scoring rule for quantile forecasts \citep{gneiting2007strictly,
gneiting2023model}, commonly used in the applied forecasting community, and
regret results with respect to quantile loss may therefore be meaningful in
their own right.

We now turn to our regret guarantee for MultiQT. We will study  
\[
\frac{1}{T} \sum_{t=1}^T \ell_t(\btheta_t) - \inf_{\btheta \in \R^m} 
\frac{1}{T} \sum_{t=1}^T \ell_t(\btheta),  
\]
called the (average) \emph{regret} of MultiQT iterates $\btheta_t$ with respect
to the MultiQT losses $\ell_t$ defined in \eqref{eq:multiqt_loss}. The analysis of regret in
this setting is somewhat nonstandard, because of the time-varying constraints
$C_t$, $t = 1,2,\dots$. In the appendix we derive a more general regret bound for
problems with constraints satisfying inward flow, of which the
following is a consequence.   

\begin{theorem}
\label{thm:multiQT_regret}
Assume that \smash{$|y_t - b_t^{\alpha}| \leq R$} for all $\alpha$ and $t$ and define
\smash{$\ell_t(\btheta) = \rho_{\cA}(b_t+\btheta, y_t)$}. Then, for all $T \geq 1$, the MultiQT iterates
from Procedure \ref{proc:multiQT} satisfy the regret bound     
\begin{equation}
\label{eq:multiQT_regret1}
\frac{1}{T} \sum_{t=1}^T \ell_t(\btheta_t) - \inf_{\btheta \in \R^m} 
\frac{1}{T} \sum_{t=1}^T \ell_t(\btheta) \leq \frac{\|\bttheta_1\|_2^2}{\eta T}
+ \frac{R^2 m}{\eta T} + \frac{\eta L^2}{2}.
\end{equation}
\end{theorem}

When \smash{$\bttheta_1 = 0$}, the learning rate that minimizes the right-hand
side in \eqref{eq:multiQT_regret1} is \smash{$\eta = R \sqrt{2m}/(L \sqrt{T})$};
plugging this in gives the bound  
\begin{equation}
\label{eq:multiQT_regret2}
\frac{1}{T} \sum_{t=1}^T \ell_t(\btheta_t) - \inf_{\btheta \in \R^m}
\frac{1}{T} \sum_{t=1}^T \ell_t(\btheta) \leq \frac{RL \sqrt{2m}}{\sqrt{T}}.  
\end{equation}
This is a no-regret result: the right-hand side converges to zero as $T \to
\infty$, at the rate \smash{$1/\sqrt{T}$}.  
Note that we can take as the comparator the vector of all zeros ($\btheta = \mathbf{0}$),
hence \eqref{eq:multiQT_regret2}, or more generally \eqref{eq:multiQT_regret1},
also bounds the excess average quantile loss suffered by MultiQT compared to
\emph{no adjustment}, i.e., compared to that incurred by the original base
forecasts. 

Furthermore, the above result bounds the regret of MultiQT
iterates compared to the vector $\btheta^*$ of empirical quantiles in hindsight. More precisely, each entry \smash{$\btheta^{*,\alpha}$} is an 
$\alpha$-level empirical quantile of \smash{$y_t - b_t^{\alpha}$}, 
$t = 1,\dots,T$. In fact, $\btheta^*$ is the optimal comparator --- that is, it minimizes the average loss through time $T$.

The next theorem generalizes the previous one to the setting of delayed feedback.    

\begin{theorem}
\label{thm:multiQT_regret_with_delay}
Under the conditions of \Cref{thm:multiQT_regret}, in the setting with feedback delay $D \geq 0$, the MultiQT iterates from Procedure
\ref{proc:multiQT} with the modification in \eqref{eq:multiQT_update_with_delay} 
satisfy
\[
\frac{1}{T} \sum_{t=1}^T \ell_t(\btheta_t) - \inf_{\btheta \in \R^m} 
\frac{1}{T} \sum_{t=1}^T \ell_t(\btheta) \leq \frac{\|\bttheta_1\|_2^2}{\eta T}
+ \frac{R^2 m}{\eta T} + \frac{\eta (2D+1) L^2}{2}.
\]
\end{theorem}

\subsection{Calibration-regret tradeoff?}

When the learning rate $\eta$ is chosen appropriately, the bounds in
\Cref{thm:calibration} and \Cref{thm:multiQT_regret} guarantee that the
$\ell_2$ calibration error and regret, respectively, vanish at the rate
\smash{$1/\sqrt{T}$}. However, the choice of learning rate
is pulled in opposite directions by these results: the theoretical bounds tell us that calibration improves with larger $\eta$, whereas regret improves with smaller $\eta$.

To balance the bounds provided in these theorems, we can identify the
dominant terms: \smash{$O(1/\sqrt{\eta T})$} in \Cref{thm:calibration}, versus
$O(\eta)$ in \Cref{thm:multiQT_regret}. Equating these two leads to the choice
of learning rate \smash{$\eta = O(T^{-1/3})$}, which then yields a
\smash{$O(T^{-1/3})$} bound on both calibration error and regret.

This leads to an interesting question: is this tradeoff fundamental?
That is, must any method necessarily trade off calibration and regret (as we have
defined them here)? And if so, is \smash{$O(T^{-1/3})$} the optimal common rate
at which they both can be controlled?

More refined results on MultiQT calibration error, which we will describe later
in the discussion, suggest that the answer to the latter question may be no in
general, and most certainly no in a more specialized case. When the base
forecasts are point forecasts (i.e., \smash{$b_t^\alpha = \mu_t$} for all
$\alpha$ and $t$), we can obtain faster rates for the calibration error of
MultiQT, with dominant term $O(1/(\eta T))$; balancing this with the regret
bound leads to a choice of learning rate \smash{$\eta = O(T^{-1/2})$}, which
provides a \smash{$O(T^{-1/2})$} bound on calibration error and regret.  Whether
this can be extended beyond point forecasts is an open question, as is the
general tradeoff between calibration and regret (or \GEQ{} error and regret,
even more generally).

\section{Experiments}
\label{sec:experiments}

We apply the MultiQT method to two real forecasting datasets relating to
COVID-19 deaths and renewable energy production.\anonymized{\footnote{Code for reproducing our experiments is available at \url{https://github.com/tiffanyding/multiQT}.}} To evaluate calibration, we
will investigate plots of the actual (empirical) coverage 
\[
\widehat{\mathrm{cov}}^{\alpha} = \frac{1}{T} \sum_{t=1}^T \cov_t^{\alpha}
\]
versus the desired (nominal) coverage $\alpha$, over the given set $\cA$ of
quantile levels. We will also examine the $\ell_1$ calibration error, normalized
by the number of levels $m = |\cA|$, defined as   
\[
\text{Calibration error} = \frac{1}{m} \sum_{\alpha \in \cA}
|\widehat{\mathrm{cov}}^{\alpha} - \alpha|.
\]
In the applied forecasting literature, it is common to measure miscalibration by 
first computing the probability integral transform (PIT) values associated with
the forecast distributions and target values over time and then reporting the 
entropy of these PIT values \citep{gneiting2007probabilistic,
  rumack2022recalibrating}. As this metric is not specific to quantile forecasts and requires a conversion to the cumulative density function, we
primarily study calibration error as defined in the above display, but we
provide results using PIT entropy in Appendix
\ref{sec:pit_entropy_results_APPENDIX} for completeness.  

We additionally study the quantile loss, normalized by the number of levels,
defined as 
\[
\text{Quantile loss} = \frac{1}{Tm} \sum_{t=1}^T \rho_{\cA}(\q_t, y_t).
\]
As explained in Section \ref{sec:regret_guarantee}, this loss function is
proper, emphasizes sharpness, and is commonly used in the forecasting community
(where it is often written in an equivalent form, called the weighted interval
score).  

Lastly, to set the learning rate $\eta$ in MultiQT, we modify a heuristic
proposed in \cite{angelopoulos2023conformal} for the learning rate in QT.  They
set the learning rate adaptively: the learning rate at time $t$ is set to be 0.1
times the largest absolute error \smash{$|y_s - b_s^{\alpha}|$} seen in the last
50 time steps $s = t-50,\dots,t-1$. We replace this max of recent errors with a
90\% quantile to avoid setting excessively large learning rates after
encountering a single large error. Specifically, we set the learning rate at
time $t$ as 
\[
\eta_t = \max\bigg\{ 0.1 \cdot \mathrm{Quantile}_{0.9}\bigg( \bigcup_{\alpha \in
  \cA}\{|y_s - b_s^{\alpha}|\}_{s=t-50}^{t-1} \bigg), \, \epsilon \bigg\},    
\]
for $\epsilon=0.1$. The lower limit of $\epsilon$ ensures that the learning rate 
is positive even if the residuals are zero.  

\subsection{COVID-19 death forecasting}

During the COVID-19 pandemic, forecasts of the pandemic's trajectory were used
to help guide short-term decisions relating to policy and resource allocation,
as well as for general public communication. The United States COVID-19 Forecast
Hub is a repository of forecasts made in real time of topline COVID-19 outcomes
collected over the pandemic, in a large collaborative effort between researchers
and the United States Center for Disease Control and Prevention
\citep{cramer2022united}. The COVID-19 Forecast Hub ran from April 2020 through
April 2024, and it collected forecasts of COVID-19 cases, hospitalizations, and 
deaths for subsets of this full four-year period, at varying spatial and
temporal resolutions. For our analysis, we focus on forecasts of weekly COVID-19
deaths at the state level, which were collected for 23 quantile levels,   
\[
\cA = \{0.01, 0.025, 0.05, 0.1, 0.15, \dots, 0.85, 0.9, 0.95, 0.975, 0.99\},
\]
and at $h$ weeks ahead, for horizons (lead times) $h \in \{1,2,3,4\}$.  

We apply our MultiQT procedure to weekly state-level COVID-19 death forecasts generated by 15 forecasting teams, corresponding to $15\times 50 = 750$ time series for each forecast horizon $h$. 
This set of forecasters is obtained by starting from the set of forecasters considered in \cite{cramer2022evaluation}, then filtering out forecasters with missing forecasts or forecasts for fewer than 50 time steps for any state. The selected forecasters have forecasts for periods ranging from 68 to 152 weeks. 

We apply MultiQT separately to each forecaster-state combination. When applying 
MultiQT to forecasts that are $h=1$ week ahead, we use the implementation described in 
Procedure \ref{proc:multiQT}; for $h \in \{2,3,4\}$, we run the delayed feedback
version of MultiQT with the modification in
\eqref{eq:multiQT_update_with_delay}, where $D = h-1$. For some forecasting teams, their
forecasts are well calibrated to begin with, while for others, their forecasts 
are systematically biased in some way (too low or too high, or their confidence 
bands are too narrow or too wide). In general, we find that wrapping the MultiQT 
method around these base forecasts successfully corrects such biases and
improves calibration, as we now describe.   

Figure \ref{fig:raw_coverage_covid_h=4} displays the calibration of
one-week-ahead death forecasts from the COVID-19 Forecast Hub. Each colored line corresponds to a single forecaster for a single location.
When one of these lines is below the black dashed line, it means that forecasts are
biased downward for the corresponding levels, whereas being above the dashed line means
that forecasts are biased upward. Both forms
of miscalibration generally dilute the utility of forecasts to decision
makers. Figure \ref{fig:cal_coverage_covid_h=4} plots the calibration of the
same forecasts after applying MultiQT. We see that MultiQT reduces both types of
bias (it brings lines closer to the diagonal, from above and below). Furthermore, Figure \ref{fig:covid_calibration_all_h} in the appendix shows that MultiQT similarly improves two-, three-, and four-week-ahead forecasts.

Figure \ref{fig:covid_arrows} illustrates the change in calibration error and
quantile loss induced by MultiQT, with one panel per forecast horizon $h$. Each
arrow represents one forecaster averaged over all states. The tail of each arrow
represents the performance of the raw forecasts, while its head represents the
performance after we apply MultiQT. All arrows point downward, which tells us
that MultiQT achieves its goal of improving calibration. In fact, after
recalibrating with MultiQT, most forecasters achieve calibration
error close to zero, corresponding to perfect calibration. We also see that this
improvement in calibration does not significantly degrade the quantile loss
and, in fact, often leads to a slight improvement. This is consistent with the
regret guarantee stated in \Cref{thm:multiQT_regret} (and we note that our
choice of learning rate in this section is more aligned with improving
calibration, whereas if we were to target regret, we would choose a learning
rate decreasing with time).      

\begin{figure}[t]
\centering
\begin{subfigure}{0.3\textwidth}
\includegraphics[width=\linewidth]{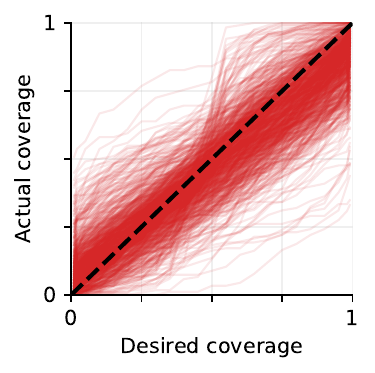}
\caption{Raw forecasts}
\label{fig:raw_coverage_covid_h=4}
\end{subfigure}%
\hspace{5pt}
\begin{subfigure}{0.3\textwidth}
\includegraphics[width=\linewidth]{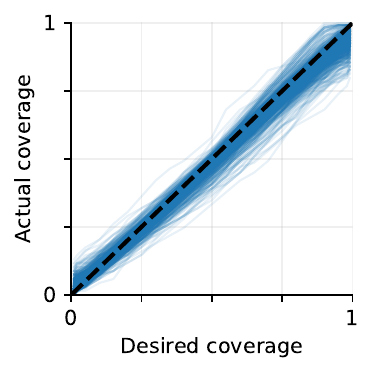}
\caption{MultiQT forecasts}
\label{fig:cal_coverage_covid_h=4}
\end{subfigure} 
\caption{Actual versus desired coverage for one-week-ahead COVID-19 death
  forecasts before (left) and after (right) applying MultiQT. Each line
  corresponds to a single forecaster for a single location.}
\label{fig:covid_calibration_h=1}
\end{figure}

\begin{figure}[t]
\includegraphics[width=\textwidth]{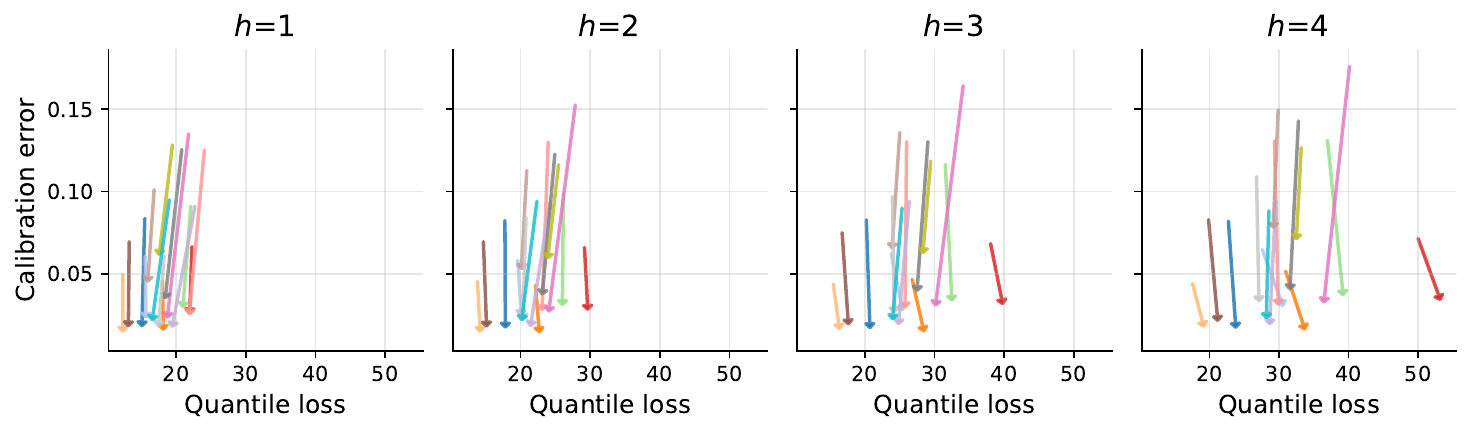}
\caption{Calibration error versus quantile loss for raw forecasts (tail of
  arrow) and MultiQT forecasts (head) for $h$-week-ahead COVID-19
  death forecasts, where $h\in\{1,2,3,4\}$. Each color represents a forecaster,
  and the coordinates of the head and tail are determined by averaging the given
  metric across all 50 states for the specified horizon. For each metric, lower
  is better.}    
\label{fig:covid_arrows}
\end{figure}

In Figures
\ref{fig:all_forecasters_ca_h=1_pt1}-\ref{fig:all_forecasters_vt_h=1_pt2} in the
appendix, we provide visualizations of the MultiQT-adjusted forecasts for each individual
COVID-19 forecaster.  

\subsection{Energy forecasting}

While renewable energy sources such as wind and solar hold great promise for
reducing carbon emissions, a significant downside is that they suffer from
uncertain production due to the inherent stochasticity of weather. This
uncertainty must be properly accounted for in order to successfully integrate
renewable energy sources into the energy grid. Grid operators rely on accurate
forecasts of renewable energy production to determine whether (and for what
times) it is necessary to procure additional energy reserves via what are known
as balancing capacity markets \citep{hirth2015balancing, balancing_market}.  

The ARPA-E PERFORM dataset was assembled to help develop more efficient and 
reliable energy grids \citep{bryce2023solar}. It consists of probabilistic
forecasts made by the National Renewable Energy Laboratory, a national
laboratory of the U.S. Department of Energy, for wind and solar energy generated
at various sites in the United States along with the actual values, all measured
in megawatts. Quantile forecasts are made at 99 levels $\cA$, which are evenly spaced from 0.1 to 0.99. For our analysis, we focus on day-ahead forecasting
for sites belonging to the Electric Reliability Council of Texas (ERCOT), the
main operator of the electrical grid in Texas. For wind power, there are 264
sites, and for solar power, there are 226 proposed sites, making a total of 490
sites. Day-ahead forecasts are made at 12:00 p.m.\ CST each day for energy  
production during each hour of the subsequent day. The dataset provides
forecasts for each day of 2018. 

We run MultiQT separately for each hour of the day. For example, one sequence of
targets we consider is the wind production of a particular site at 10:00 a.m.\ on
January 1, 10:00 a.m.\ on January 2, 10:00 a.m.\ on January 3, and so on. 
Motivated by how balancing capacity products are available in four-hour blocks \citep{balancing_market}, we specifically focus on the hours 2:00 a.m., 6:00 a.m., 10:00 a.m., 2:00 p.m.,
6:00 p.m., and 10:00 p.m. Each of these hours belongs to a different four-hour
block and its forecasts can be used to inform whether a balancing capacity
product will be needed for that time block. For the first three hours we
consider (2:00 a.m., 6:00 a.m., 10:00 a.m.), feedback from the previous day's
forecast is available before the next day's forecasts are issued at 12:00 p.m.,
so there is no delay in feedback. However, for the afternoon and evening hours
(2:00 p.m., 6:00 p.m., and 10:00 p.m.), there is a one-day delay because we do
not observe feedback for these hours before issuing the next day's forecasts.  
Therefore, for these hours, we run MultiQT with a feedback delay of $D=1$.     

Figure \ref{fig:energy_calibration_curves} displays the calibration of quantile
forecasts before and after applying MultiQT to the forecasts for energy
production at 10:00 a.m. We can see that the raw wind forecasts are generally
biased upward, with calibration curves falling above the diagonal line, and the 
solar forecasts are generally too narrow, with calibration curves that are too flat 
(nearly horizontal). MultiQT corrects each of these issues and delivers near
perfect calibration. A similar improvement in the calibration of energy
forecasts can be seen for other hours of the day, displayed in Figure 
\ref{fig:wind_calibration_all_hours} in the appendix.

Figure \ref{fig:energy_arrows} again illustrates the change in calibration error
and quantile loss induced by MultiQT for all six hours we consider, in the same format
as Figure \ref{fig:covid_arrows} for the COVID-19 dataset. Here, each arrow
corresponds to a different site. The results are qualitatively similar to those
for the COVID-19 dataset: MultiQT consistently improves forecast calibration and
never substantially increases the average quantile loss. In particular, for the
solar forecasts, we generally see a strong improvement in quantile loss due to
the extremely poor calibration of the raw forecasts.  


\begin{figure}[t]
\centering 
\begin{subfigure}{0.49\textwidth} 
\centering
\begin{subfigure}{0.5\textwidth}
\includegraphics[width=\linewidth]{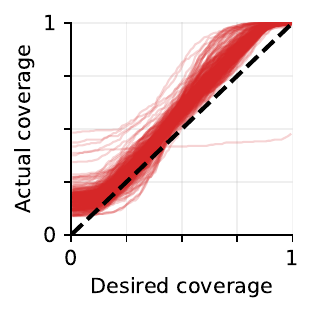}
\end{subfigure}%
\begin{subfigure}{0.5\textwidth}
\includegraphics[width=\linewidth]{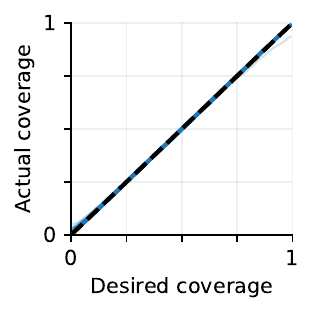}
\end{subfigure}
\caption{Wind}
\end{subfigure}%
\begin{subfigure}{0.49\textwidth} 
\centering
\begin{subfigure}{0.5\textwidth}
\includegraphics[width=\linewidth]{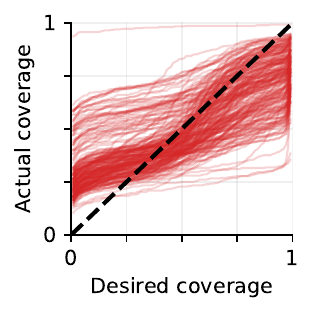}
\end{subfigure}%
\begin{subfigure}{0.5\textwidth}
\includegraphics[width=\linewidth]{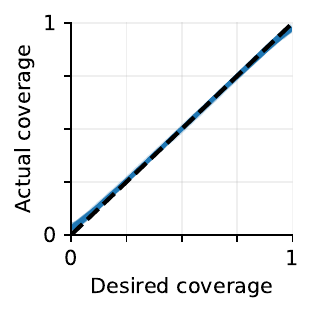}
\end{subfigure}
\caption{Solar}
\end{subfigure}
\caption{Actual versus desired coverage for day-ahead wind (a) and solar (b)
  energy forecasts for the 10:00 a.m.\ time period. In each of (a) and (b), the
  left panel corresponds to raw forecasts, and the right panel to
  MultiQT-adjusted forecasts; each line corresponds to a different site.}          
\label{fig:energy_calibration_curves}
\end{figure}

\begin{figure}[t]
\centering
\begin{subfigure}{0.5\textwidth}
\includegraphics[width=\linewidth]{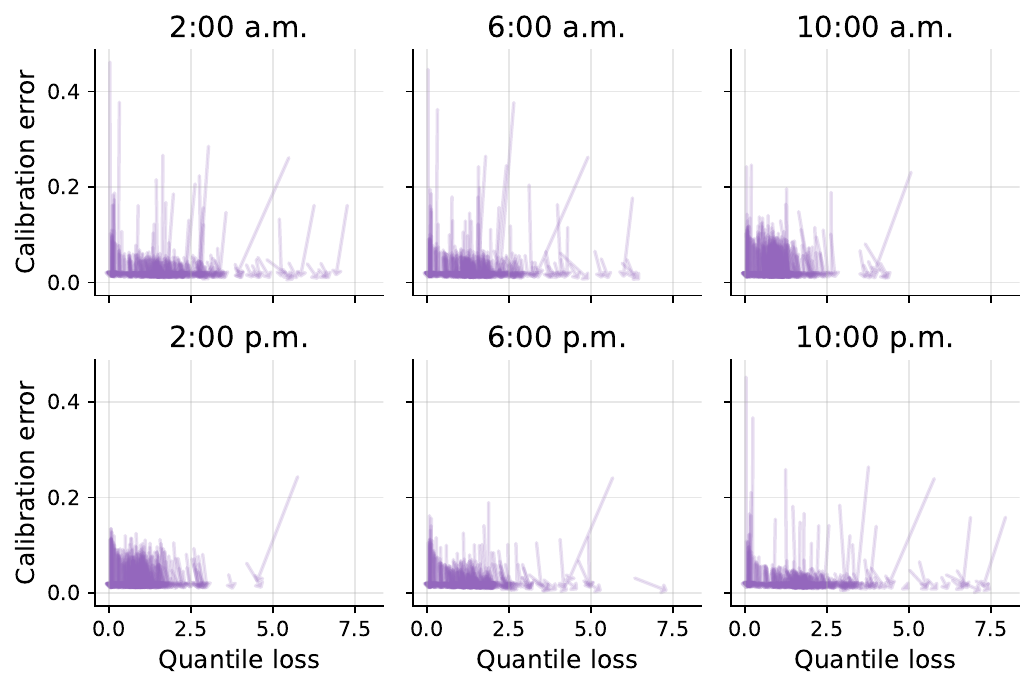}
\caption{Wind}
\label{fig:energy_wind_arrows}
\end{subfigure}%
\begin{subfigure}{0.5\textwidth}
\includegraphics[width=\linewidth]{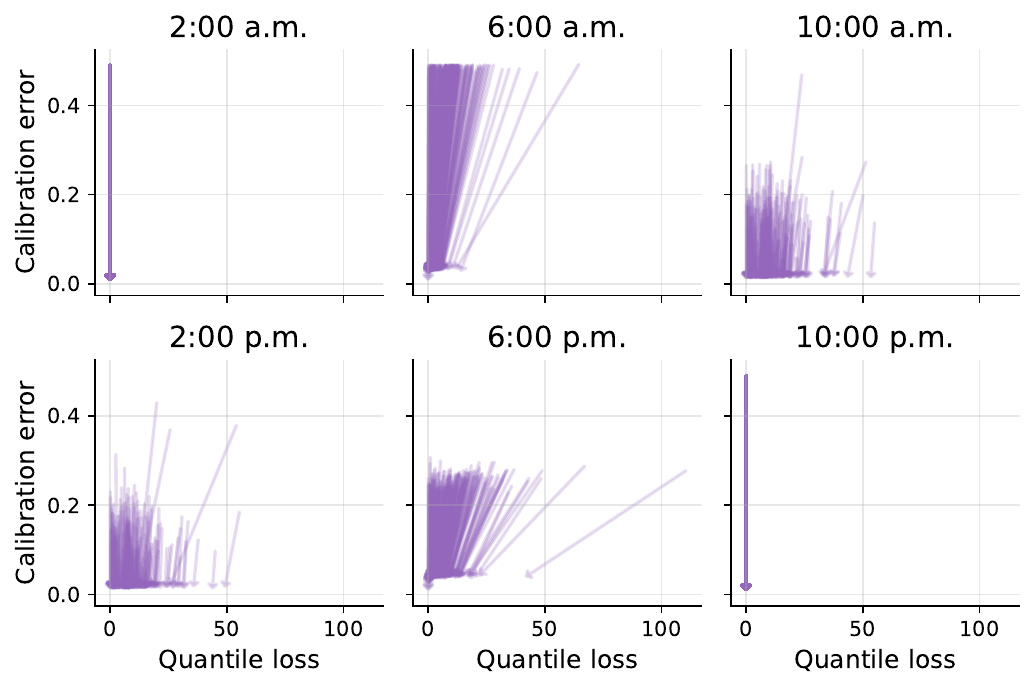}
\caption{Solar}
\label{fig:energy_solar_arrows}
\end{subfigure}
\caption{Calibration error versus quantile loss for raw forecasts (tail of
  arrow) and MultiQT forecasts (head) for day-ahead wind and solar energy
  production at 2:00 a.m., 6:00 a.m., 10:00 a.m., 2:00 p.m., 6:00 p.m., and
  10:00 p.m. Each arrow represents a different site. For each metric, lower 
  is better.}    
\label{fig:energy_arrows}
\end{figure}

As a case study to better understand how MultiQT changes the raw forecasts, we
visualize the forecasts before and after applying MultiQT for a wind energy site
whose raw forecasts were particularly miscalibrated in Figure
\ref{fig:energy_case_study}. For visual clarity, we show forecasts
only for a 50-day period starting from March 1, 2018 but the underlying
experiment covers the entire year of 2018. We can see that the raw forecasts
are upwardly biased and too narrow in many places compared to the true energy
output.  MultiQT largely corrects for this and provides a better representation
of the uncertainty. In Figures
\ref{fig:wind_8forecasters}-\ref{fig:solar_8forecasters} in the appendix, we
present analogous plots for additional wind and solar sites.

\begin{figure}[t]
\begin{subfigure}{\textwidth}
\centering
\begin{subfigure}{0.52\textwidth}
\includegraphics[width=\linewidth]{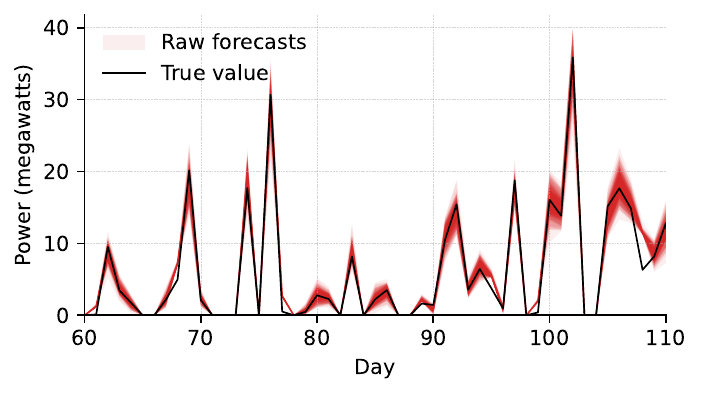}
\end{subfigure}%
\begin{subfigure}{0.23\textwidth}
\centering\raisebox{0.1em}{
\includegraphics[width=\linewidth]{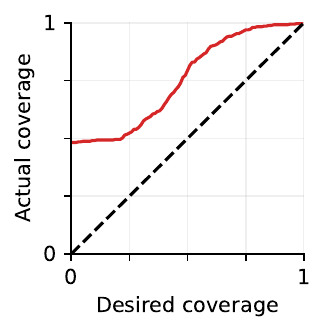}}
\end{subfigure}
\caption{Raw forecasts and their calibration.}
\end{subfigure}

\begin{subfigure}{\textwidth} 
\centering
\begin{subfigure}{0.52\textwidth}
\includegraphics[width=\linewidth]{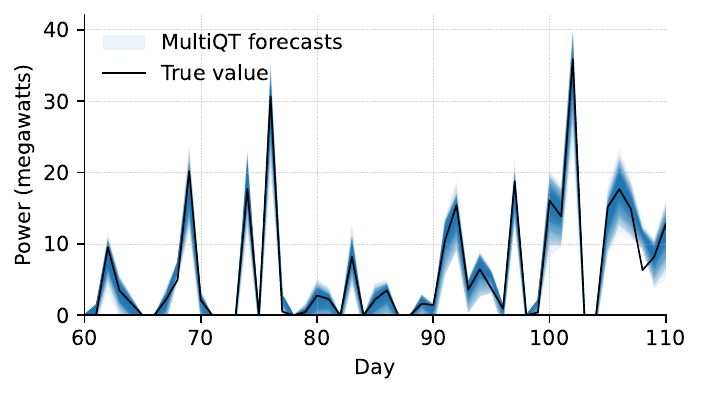}
\end{subfigure}%
\begin{subfigure}{0.23\textwidth}
\centering\raisebox{0.1em}{
\includegraphics[width=\linewidth]{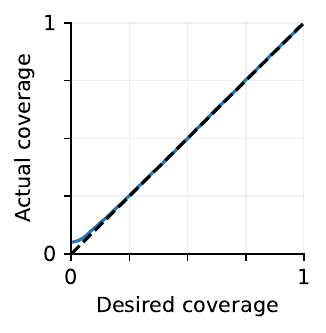}}
\end{subfigure}
\caption{Forecasts and calibration after applying MultiQT.}
\end{subfigure}

\caption{Day-ahead wind energy forecasts for the site
  \texttt{Wind\_Power\_Partners\_94\_Wind\_Farm} for 10:00 a.m.\ each day from
  March 1, 2018 to April 20, 2018, before (top) and after (bottom) applying
  MultiQT.}   
\label{fig:energy_case_study}
\end{figure}

\section{Discussion}
\label{sec:discussion}

In this paper, we proposed a simple procedure which wraps around any existing 
online quantile forecaster to produce corrected forecasts that are guaranteed to
be calibrated without crossing, meaning quantile forecasts at successive levels
are always properly ordered. Our method, MultiQT, is an instance of lazy gradient
descent applied to a particular online learning problem, involving quantile
losses and the isotonic cone. To establish a calibration guarantee (for
arbitrary and potentially even adversarial data sequences), we abstract to a
more general problem of achieving constrained gradient equilibrium (GEQ) via lazy
gradient descent, and we derive new \GEQ{} theory for this algorithm. We also
derive a regret guarantee with respect to the quantile loss. In experiments with datasets from
COVID-19 and energy forecasting, we find that MultiQT
significantly improves the calibration of real forecasters,
for the most part without sacrificing quantile loss, and often slightly
improving it.

We finish by discussing some topics related to the main thrust of our paper,  
and ideas for future work.

\paragraph{From quantile forecasts to prediction intervals.}

Throughout, we have touched on some reasons why achieving calibration while maintaining distributional consistency is
an important problem. Here is yet another useful consequence: these two properties 
together allow us to construct nested prediction intervals with the correct
long-run coverage, where ``nested'' means that the $(1-\alpha)$-level interval  
will always be contained in the $(1-\beta)$-level interval when $\alpha > 
\beta$. We emphasize that properly ordered quantiles is critical for obtaining
nested intervals; otherwise nestedness may not be satisfied, e.g., at a given time step
we might have one or both of the endpoints of the 0.5-level prediction interval
lying outside the 0.9-level prediction interval. Coverage of such intervals can be
verified directly: to construct equi-tailed $(1-\alpha)$-level intervals, we can
define \smash{$I^{1-\alpha}_t = (q_t^{\alpha/2}, q_t^{1-\alpha/2}]$}, and then    
\begin{align*}
\frac{1}{T} \sum_{t=1}^T \one\{y_t \in I^{1-\alpha}_t\} 
&= 1 - \frac{1}{T} \sum_{t=1}^T \one\{y_t \leq q_t^{\alpha/2}\} - 
\frac{1}{T} \sum_{t=1}^T \one\{y_t > q_t^{1-\alpha/2}\} \\
&= 1 - \alpha/2 - \alpha/2 = 1 -\alpha.
\end{align*}
Note that have implicitly relied on properly ordered quantiles in using 
\smash{$q_t^{\alpha/2} \leq q_t^{1-\alpha/2}$} for all $t$. Further, if
each $y_t$ is continuously distributed, then using closed intervals
\smash{$[q_t^{\alpha/2}, q_t^{1-\alpha/2}]$}, $t = 1,2,\dots$ will achieve the 
correct coverage with probability one.    

\paragraph{Faster rates for GEQ and calibration error.} 

In Section \ref{sec:multiQT_guarantees}, we showed that the calibration error of
the MultiQT forecasts approaches zero at a \smash{$1/\sqrt{T}$} rate. This was
established based on new theory for constrained \GEQ{} in Section
\ref{sec:constrained_GEQ}. Here we refine these analyses to give faster rates
under stronger assumptions. Proofs are given in Appendix
\ref{sec:proofs_for_faster_rate_APPENDIX}. 

We first refine the \GEQ{} result in \Cref{prop:GEQ_lazyGD} by assuming that the function 
$\phi$ that appears in the definition of restorativity
\eqref{eq:restorativity} can be lower bounded by a positive constant, and that 
the distance between pairs of hidden and played iterates remains bounded.  

\begin{proposition}
\label{prop:GEQ_with_bounded_distance}
Under the conditions of \Cref{prop:GEQ_lazyGD}, additionally assume that each
loss $\ell_t$ is now $(h_t,\phi_t)$-restorative, with $\phi_t(\btheta) \geq \eta
L^2/2$ for $\|\btheta\|_2 > h_t$, and that each pair of hidden and played
iterates \smash{$(\bttheta_t, \btheta_t)$} remains within a bounded distance of
each other: \smash{$\|\bttheta_t - \btheta_t\|_2 \leq B$}. If $h_t$ is a
nondecreasing sequence, then, for all $T \geq 1$, the lazy \GD{} iterates in
\eqref{eq:lazyGD_projection} and \eqref{eq:lazyGD_update} satisfy 
\[
\|\bttheta_{T+1}\|_2 \leq \max\{\|\bttheta_1\|_2, h_T\} + B + \eta L. 
\]
This implies  
\[
\bigg\| \frac{1}{T} \sum_{t=1}^T \g_t(\btheta_t) \bigg\|_2 \leq
\frac{2\|\bttheta_1\|_2 }{\eta T} + \frac{L}{T} + \frac{h_T + B}{\eta T},
\]
which goes to zero as $T \to \infty$ as long as $h_T$ is sublinear in $T$.
\end{proposition}

When $h_t = h$ for all $t$, \Cref{prop:GEQ_lazyGD} gives a \smash{$1/\sqrt{T}$}
rate for \GEQ{}. By adding a stronger assumption on $\phi$, and importantly, the 
assumption that hidden and played sequences do not diverge away from each other,
we see that \Cref{prop:GEQ_with_bounded_distance} improves this to a $1/T$ rate.

The assumption that $\phi$ is lower bounded by a positive constant, which 
\cite{angelopoulos2025gradient} refer to as a ``positive curvature''
condition, is not strong. We can show that quantile loss satisfies this
condition as an extension of \Cref{lemma:multiqt_cond1}. The assumption that
\smash{$\|\bttheta_t - \btheta_t\|_2$} remains bounded is trickier to 
analyze. Though it seems intuitive that in most instances we would expect this 
to be the case, it is nonetheless challenging to verify formally. Fortunately,
the next lemma shows that this is true for MultiQT in a specialized setting,
where the base forecasts are point forecasts.   

\begin{lemma}
\label{lemma:point_forecasts_bounded_distance}
If the base forecasts are point forecasts, i.e., \smash{$b_t^{\alpha} = \mu_t$}
for all $\alpha$ and $t$, and $|y_t - \mu_t| \leq R$ for all $t$, then the MultiQT iterates starting from initialization \smash{$\bttheta_1 \in \cK$} satisfy 
\smash{$\|\bttheta_t - \btheta_t\|_2 \leq \eta m^{3/2} / \sqrt{3}$} for all $t$.  
\end{lemma}

Combining the previous results we get the following $1/T$ rate on the
calibration error of MultiQT when the base forecasts are point forecasts.   

\begin{corollary}
\label{cor:multiQT_point_forecast_calibration_rate}
Under the conditions of \Cref{lemma:point_forecasts_bounded_distance}, for all $T \geq 1$, 
the MultiQT iterates from Procedure \ref{proc:multiQT} obey the $\ell_2$
coverage error bound     
\[
\sqrt{\sum_{\alpha \in \cA} \bigg( \frac{1}{T} \sum_{t=1}^T \cov_t^{\alpha} -
  \alpha \bigg)^2} \leq \frac{2\|\bttheta_1\|_2 }{\eta T} + \frac{\sqrt{m}}{T} +
\frac{m^{3/2}}{2 d_{\cA} T} + \frac{R m^{3/2}}{d_{\cA} \eta T} +
\frac{m^{3/2}}{\sqrt{3} T}, 
\]
where recall $m = |\cA|$, and \smash{$d_{\cA}$} is as in
\Cref{lemma:multiqt_cond1}. Additionally recall that, since $\|x\|_\infty \leq \|x\|_2$ for any vector $x \in
\R^m$, the right-hand side above is also an upper bound for \smash{$\big| \frac{1}{T}
  \sum_{t=1}^T \cov_t^{\alpha} - \alpha \big|$}, for each $\alpha \in \cA$.
\end{corollary}

It is an open question whether MultiQT maintains a bounded distance between
hidden and played iterates in general, when each $b_t$ is an arbitrary ordered
vector of quantile forecasts, and hence whether the $1/T$ rate of
\Cref{cor:multiQT_point_forecast_calibration_rate} extends to this general 
setting.

\paragraph{Inward flow.} 

A key condition we used in our analysis is inward flow, which says that the
negative gradient field points inwards at the boundary of the constraint set. We
showed that lazy gradient descent achieves constrained \GEQ{} when inward
flow is satisfied, and our calibration guarantee for MultiQT relied on the fact
that the MultiQT loss and constraint set jointly satisfy inward flow. This
property is far from trivial, 
and can fail even in seemingly natural modifications of the MultiQT problem; for
example, we might seek $\varepsilon$-separated quantiles, and define for a
constant $\varepsilon \geq 0$ the constraint set
\begin{equation}
\label{eq:eps_separation}
C_t^{\varepsilon} = \Big\{ x \in \R^m : x_i + \varepsilon \leq x_{i+1}, \; i =
1,2, \dots, m-1 \Big\} - \b_t. 
\end{equation}
Notice that setting $\varepsilon = 0$ recovers the original constraint set
\eqref{eq:multiqt_constraint}. Unfortunately, the MultiQT loss and the
$\varepsilon$-separated constraint set \smash{$C_t^{\varepsilon}$} do not
satisfy inward flow. This is visualized in Figure
\ref{fig:gradient_field_epsilon}: we can see that the negative gradient, denoted by the small arrows, does not point inwards at all points on the boundary
of \smash{$C_t^{\varepsilon}$}. We can contrast this with Figure \ref{fig:gradient_field},
which visualizes the constraint used in MultiQT.

Violating inward flow means that the result in \Cref{thm:calibration} does not
apply but leaves open the possibility that other analyses may be used to establish
a calibration result; however, we show that this is  not the case for the
$\varepsilon$-separated constraint set, and construct a
formal negative example in Appendix \ref{sec:negative_results_APPENDIX}. 

\begin{figure}[h!]
\centering
\begin{subfigure}{0.36\textwidth}
\centering

\tikzset{>={Stealth[length=2.6pt,width=3.5pt]}}

\begin{tikzpicture}[line cap=round]

\def\aone{0.30}   
\def\atwo{0.70}   
\def\Y{0.55}      
\def\s{0.32}      
\def\L{2.4}       
\def\LA{2.6}      
\def\grid{-2.4,-1.6,-0.8,0,0.8,1.6,2.4}
\def\tick{0.09}   
\def\eps{0}    

\fill[gray!35]
  (-\LA,-{\LA+\eps}) --
  (-\LA,\LA) --
  ({\LA-\eps},\LA)
  -- cycle;



\node[gray, scale=1.1] at (-1.87,1.87) {$C_t$};

\draw[<->] (-\LA,0) -- (\LA,0) node[pos=1,below right] {$\theta_t^{\alpha_1}$};
\draw[<->] (0,-\LA) -- (0,\LA) node[pos=1,above left] {$\theta_t^{\alpha_2}$};

\draw (-\tick, \Y) -- (\tick, \Y) node[left=4pt] {$y_t$};   
\draw (\Y, -\tick) -- (\Y, \tick) node[below=4pt] {$y_t$};  


\foreach \x in \grid {
  \foreach \y in \grid {

    \pgfmathsetmacro{\vx}{(\x<\Y ? \aone : \aone-1)}
    \pgfmathsetmacro{\vy}{(\y<\Y ? \atwo : \atwo-1)}
    \pgfmathsetmacro{\len}{sqrt((\vx)^2 + (\vy)^2)}
    \pgfmathsetmacro{\ux}{\vx/(\len+1e-9)}
    \pgfmathsetmacro{\uy}{\vy/(\len+1e-9)}

    \pgfmathsetmacro{\d}{abs(\y - (\x + \eps))}

    \ifdim \d pt<0.01pt
      \def\arrowcolor{blue}%
    \else
      \def\arrowcolor{blue!40}%
    \fi

    \draw[->, line width=1pt, \arrowcolor]
      (\x,\y) -- ++({\s*\ux},{\s*\uy});
  }
}

\end{tikzpicture}
\vspace{-20pt}
\caption{Original constraint set}
\label{fig:gradient_field}
\end{subfigure}%
\hspace{10pt}
\begin{subfigure}{0.36\textwidth}
\centering

\tikzset{>={Stealth[length=2.6pt,width=3.5pt]}}

\begin{tikzpicture}[line cap=round]

\def\aone{0.30}   
\def\atwo{0.70}   
\def\Y{0.55}      
\def\s{0.32}      
\def\L{2.4}       
\def\LA{2.6}      
\def\grid{-2.4,-1.6,-0.8,0,0.8,1.6,2.4}
\def\tick{0.09}   
\def\eps{1.60}    

\fill[gray!35]
  (-\LA,-{\LA+\eps}) --
  (-\LA,\LA) --
  ({\LA-\eps},\LA)
  -- cycle;


\draw[color=violet] (-\tick, \eps) -- (\tick, \eps) node[left=4pt, text=violet]{$\varepsilon$};   

\node[gray, scale=1.1] at (-1.87,1.87) {$C_t^{\varepsilon}$};

\draw[<->] (-\LA,0) -- (\LA,0) node[pos=1,below right] {$\theta_t^{\alpha_1}$};
\draw[<->] (0,-\LA) -- (0,\LA) node[pos=1,above left] {$\theta_t^{\alpha_2}$};

\draw (-\tick, \Y) -- (\tick, \Y) node[left=4pt] {$y_t$};   
\draw (\Y, -\tick) -- (\Y, \tick) node[below=4pt] {$y_t$};  


\foreach \x in \grid {
  \foreach \y in \grid {

    \pgfmathsetmacro{\vx}{(\x<\Y ? \aone : \aone-1)}
    \pgfmathsetmacro{\vy}{(\y<\Y ? \atwo : \atwo-1)}
    \pgfmathsetmacro{\len}{sqrt((\vx)^2 + (\vy)^2)}
    \pgfmathsetmacro{\ux}{\vx/(\len+1e-9)}
    \pgfmathsetmacro{\uy}{\vy/(\len+1e-9)}

    \pgfmathsetmacro{\d}{abs(\y - (\x + \eps))}

    \ifdim \d pt<0.01pt
      \def\arrowcolor{blue}%
    \else
      \def\arrowcolor{blue!40}%
    \fi

    \draw[->, line width=1pt, \arrowcolor]
      (\x,\y) -- ++({\s*\ux},{\s*\uy});
  }
}

\end{tikzpicture}
\vspace{-20pt}
\caption{$\varepsilon$-separated constraint set}
\label{fig:gradient_field_epsilon}
\end{subfigure} 
\caption{Visualization of the negative gradient field of the MultiQT loss
  (arrows) for two quantiles, with no base forecaster (\smash{$b_t^{\alpha_1} =  
    b_t^{\alpha_2} = 0$}), and with the target value $y_t$ as drawn. The inward
  flow property is satisfied for the original MultiQT constraint set (left), but
  not the $\varepsilon$-separated constraint set (right).}      
\label{fig:gradient_field_quantile_loss}
\end{figure}
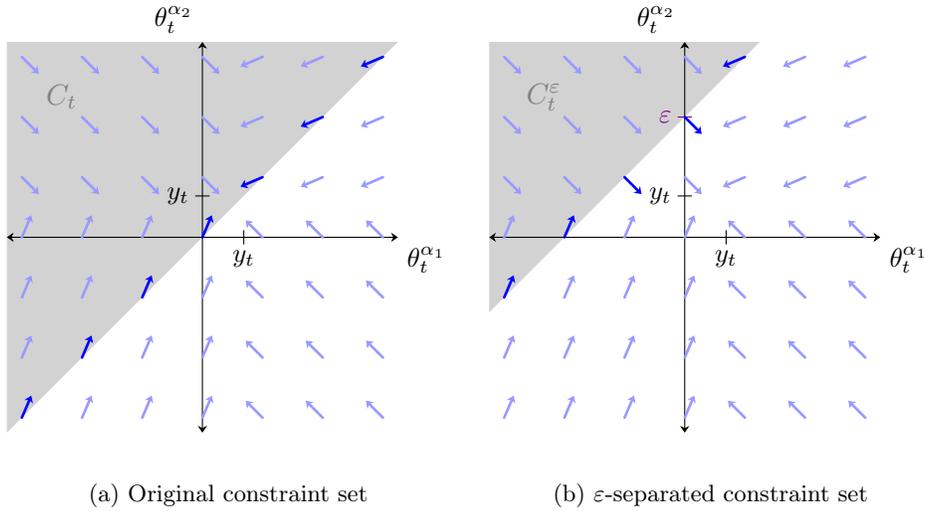

In general, it is unclear to us to what degree inward flow is necessary for lazy
\GD{} to achieve constrained \GEQ{}, and if not, whether there are other
sufficient conditions that may be more naturally satisfied in some settings. 
The study of \GEQ{} under iterate constraints is still nascent and requires
further development.     


\paragraph{Future work.}

We close with some ideas for future work, in addition to ones already
mentioned. In multi-horizon forecasting settings, such as the COVID-19
forecasting problem (where forecasts were simultaneously issued for outcomes
one, two, three, and four weeks ahead), the forecast residuals share correlation 
between successive horizons, and one may leverage scorecasting techniques as in
\cite{angelopoulos2023conformal, wang2024online} to improve sharpness at an
individual quantile level. Combining this with MultiQT would be an important
practical development. 

Another idea is to approach conditional notions of calibration, which require 
coverage to be obtained at each quantile level conditional on the quantile
prediction between issued. This is of course stronger than the notion considered
in our paper (which does not perform any conditioning). It is worth noting that 
our notion of coverage in this paper corresponds to a discretization of what is
called probabilistic calibration (also called PIT calibration) in the
forecasting literature, see, e.g., \cite{gneiting2007probabilistic}. Conditional
versions will have different names, depending on what precisely is being 
conditioned on, with the strongest version being called auto-calibration, see,
e.g., \cite{gneiting2023regression}. As a future direction, it would be
desirable to be able to encode conditional notions of calibration as a form of 
gradient equilibrium, and show this can be obtained with standard lightweight
online methods such as (lazy) gradient descent, as existing methods for
achieving conditional calibration such as \cite{noarov2023high} are more
computationally complex.    

A final idea would be to consider an infinite-dimensional version of our
problem, where the base forecasts take the form of a quantile function $b_t :
[0,1] \to \R$. This could be seen as taking $m \to \infty$ in our current setup.
This poses numerous challenges (algorithmically and theoretically), but would  
nonetheless be an interesting direction.  

\anonymized{
\section*{Acknowledgments}

We thank Rina Barber and Aaron Roth for helpful discussions and Erez Buchweitz
for guidance on working with the COVID-19 Forecast Hub dataset. TD was supported
by the National Science Foundation Graduate Research Fellowship Program grant
no.\ 2146752. IG and RJT were supported by the Office of Naval Research grant
no.\ N00014-20-1-2787.   
}

{\RaggedRight
\bibliographystyle{plainnat}
\bibliography{references}}

\newpage
\appendix


\section{Distributional inconsistency of QT}
\label{sec:QT_crossings_APPENDIX}

To demonstrate that running the quantile tracker (QT) separately for each
quantile level cannot be used to solve calibration without crossings, we run
this method on the COVID-19 dataset from \cite{cramer2022united}, as described in Section \ref{sec:experiments}. We also use the learning rate heuristic described in that section. We run QT on 750 time series of one-week-ahead forecasts of weekly COVID-19 deaths at the state level (15 forecasters $\times$ 50 states).   Figure \ref{fig:QT_crossing_frac} plots the fraction of time steps in each time series that have at least one pair of crossed quantiles (where we say a crossing has occurred at time $t$ if there exists quantile 
levels $\alpha < \beta$ where the corresponding QT-adjusted forecasts 
satisfy \smash{$q_t^{\alpha} > q_t^{\beta}$}). We see QT produces
distributionally inconsistent quantiles at 87\% of time steps on average, which
is practically undesirable.       

\begin{figure}[h!]
\centering
\includegraphics[width=0.55\linewidth]{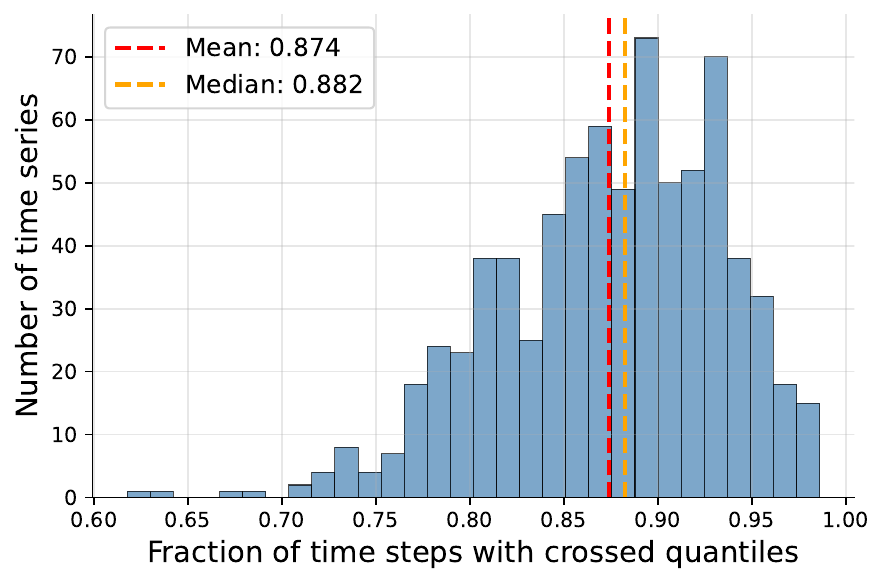}
\caption{Histogram of the fraction of time steps with crossings after applying
  QT to one-week-ahead COVID-19 death forecasts separately for each quantile, over the
  15 forecasting teams and all 50 states.} 
\label{fig:QT_crossing_frac}
\end{figure}


\section{Proofs for Sections \ref{sec:constrained_GEQ} and
  \ref{sec:multiQT_guarantees}}  
\label{sec:proofs_APPENDIX} 

We prove results for the delayed feedback setting with a constant delay of $D \geq 0$. 
Results for the no-delay setting follow immediately by setting $D=0$. The object
of our analysis is lazy gradient descent (lazy GD) with constant delay, which we
now discuss in some further detail. 

With delay $D \geq 0$, we do not observe $y_t$ at time $t$ (except when $D=0$); rather, it is
observed after we play our action at time $t+D$, at which point we use it to take
a gradient step. Recall that lazy \GD{} in the current
setting is given by \eqref{eq:lazyGD_update_with_delay} followed by 
\eqref{eq:lazyGD_projection}, where for convenience we set $\g_t(\btheta_t)= 0$
for $t \leq 0$. By unrolling \eqref{eq:lazyGD_update_with_delay} we obtain 
\begin{equation}
\label{eq:lazyGD_update_with_delay_unrolled}
\bttheta_{t+1} = \bttheta_1 - \eta \sum_{s=1}^{t-D} \g_s(\btheta_s),
\end{equation}
where we adopt the convention that the summation from $a$ to $b$ is zero if $b
\leq a$. Using \eqref{eq:lazyGD_update_with_delay_unrolled}, we get two simple facts that we will use
below. First, substituting \eqref{eq:lazyGD_update_with_delay_unrolled} into \eqref{eq:lazyGD_projection} allows us to rewrite the whole lazy \GD{} algorithm as  
\begin{equation} 
\label{eq:lazyGD_with_delay_unrolled}
\btheta_{t+1} = \Pi_{C_{t+1}}\bigg( \bttheta_1 - \eta \sum_{s=1}^{t-D}
\g_s(\btheta_s) \bigg).
\end{equation} 
Second, by rearranging \eqref{eq:lazyGD_update_with_delay_unrolled} and
using the triangle inequality, we get that the average gradient satisfies    
\begin{equation}
\label{eq:lazyGD_avg_gradient_with_delay}
\bigg\| \frac{1}{T} \sum_{t=1}^T \g_t(\btheta_t) \bigg\|_2 \leq
\frac{\|\bttheta_1\|_2 + \|\bttheta_{T+D+1}\|_2}{\eta T}.
\end{equation}

\subsection{Constrained \GEQ{} for lazy \GD{} with delay}  

In this section, we prove the results needed to show that lazy gradient descent (with delay) achieves constrained gradient equilibrium when inward flow holds. 

\subsubsection{Proof of \Cref{prop:GEQ_lazyGD_with_delay}}

We follow the general proof structure from Proposition 5 of
\cite{angelopoulos2025gradient}. We begin by expanding the square in
\eqref{eq:lazyGD_update}:  
\begin{align*}
\|\bttheta_{T+D+1}\|_2^2 
&= \|\bttheta_{T+D}\|_2^2 + \eta^2 \|\g_T(\btheta_T)\|_2^2 - 2\eta \langle
\g_T(\btheta_T), \bttheta_{T+D} \rangle \\
\numberthis\label{eq:A1}
&\leq \|\bttheta_{T+D}\|_2^2 + \eta^2 L^2 - 2\eta \langle \g_T(\btheta_T),
  \bttheta_{T+D} \rangle,
\end{align*}
where the second line uses Lipschitzness. We focus on bounding the third term in
\eqref{eq:A1}, which we rewrite as   
\begin{equation}\label{eq:A2} 
-2\eta \langle \g_T(\btheta_T), \bttheta_{T+D} \rangle = - 2\eta \langle
\g_T(\btheta_T), \bttheta_T \rangle - 2\eta \langle \g_T(\btheta_T),
\bttheta_{T+D} - \bttheta_T \rangle.
\end{equation}
We start by bounding the first term in \eqref{eq:A2}.
Due to inward flow, instead of bounding the inner product of the gradient and the hidden iterate, we can instead bound the inner product with the played iterate: by part (ii) of \Cref{lemma:inward_flow} below, we have
\smash{$-\langle \g_T(\btheta_T), \bttheta_T \rangle \leq -\langle
  \g_T(\btheta_T), \btheta_T \rangle$}. Proceeding in cases, if $\|\btheta_T\|_2
> h_T$, then the restorativity condition kicks in and we have   
\[
-\langle \g_T(\btheta_T), \btheta_T \rangle \leq 0.
\]
Otherwise, if $\|\btheta_T\|_2 \leq h_T$, then we have
\[
-\langle \g_T(\btheta_T), \btheta_T \rangle 
 \leq \|\g_T(\btheta_T)\|_2 \|\btheta_T\|_2
\leq L h_T
\]
by Cauchy-Schwarz, Lipschitzness, and the assumption on
$\|\btheta_T\|_2$. Combining the above arguments, we have \smash{$-\langle  
  \g_T(\btheta_T), \bttheta_T \rangle \leq \max\{0, L h_T\} =  L h_T$}. 

The second term in \eqref{eq:A2} is the penalty we incur for delayed feedback. 
To bound it, note that  
\begin{align*}
-\langle \g_T(\btheta_T), \bttheta_{T+D} - \bttheta_T \rangle 
&\leq \|\g_T(\btheta_T)\|_2 \|\bttheta_{T+D} - \bttheta_T\|_2 \\
&\leq \|\g_T(\btheta_T)\|_2 \bigg\|\eta \sum_{t=T}^{T+D-1} \g_t(\btheta_t)
\bigg\|_2 \\ 
&\leq \eta \|\g_T(\btheta_T)\|_2 \bigg( \sum_{t=T}^{T+D-1}
\|\g_t(\btheta_t)\|_2 \bigg) \\ 
&\leq \eta DL^2,
\end{align*}
where the first line uses Cauchy-Schwarz, the second line is due to 
\eqref{eq:lazyGD_update_with_delay}, the third uses the triangle inequality, and
the fourth uses the Lipschitzness assumption. 
 
Inserting both of these bounds into \eqref{eq:A2}, we get
\[
- 2\eta \langle \g_T(\btheta_T), \bttheta_{T+D} \rangle \leq 2 \eta L h_T +
2\eta^2 DL^2.  
\]
Plugging this back into \eqref{eq:A1}, we obtain 
\begin{align*}
\|\bttheta_{T+D+1}\|_2^2 
&\leq \|\bttheta_{T+D}\|_2^2 + \eta^2L^2 + 2 \eta L h_T + 2\eta^2DL^2 \\
&= \|\bttheta_{T+D}\|_2^2 + \eta^2 L^2 (2D+1) + 2 \eta L h_T \\
&\leq \|\bttheta_{D+1}\|_2^2 + \eta^2 L^2 (2D+1)T + 2 \eta L \sum_{t=1}^T h_t \\
&\leq \|\bttheta_1\|_2^2 + \eta^2 L^2 (2D+1)T + 2 \eta L \sum_{t=1}^T h_t,
\end{align*}
where the last line uses \smash{$\|\bttheta_t \|_2 = \|\bttheta_1\|_2$}, for all
$t \leq D+1$. Taking a square root gives the bound on
\smash{$\|\bttheta_{T+D+1}\|_2$} stated in the theorem.
    
To get the bound on the $\ell_2$ norm of the average gradient, we first simplify
the \smash{$\|\bttheta_{T+D+1}\|_2$} bound by observing that the nondecreasing
property of $h_t$ implies \smash{$\sum_{t=1}^T h_t \leq T h_T$}. We then apply 
the fact \smash{$\sqrt{a + b} \leq \sqrt{a} + \sqrt{b}$}, and lastly invoke
\eqref{eq:lazyGD_avg_gradient_with_delay}.  

\subsubsection{Gradient alignment lemma}

We now state and prove a fact used in the previous proof about the effect of
projection on the inner product of an iterate and its gradient when inward
flow is satisfied. We use {$N_C(\x) = \{\v : \v^{\T}(\x-\y) \geq
  0 \; \text{for all $\y \in C$}\}$} to denote the normal cone of a set $C$ at
$x$.  

\begin{lemma}
\label{lemma:inward_flow}
If a loss $\ell$ (with gradient $g$) and a closed convex set $C$ satisfy inward flow, then:

\begin{enumerate}
\item[(i)] 
$\langle \v, \g(\z) \rangle \geq 0$ for all $\z \in \mathrm{bd}(C)$ and $\v \in N_C(\z)$.

\item[(ii)] $\langle \z, \g(\Pi_C(\z)) \rangle \geq
  \langle \Pi_C(\z), \g(\Pi_C(\z)) \rangle$ for all $\z$.
\end{enumerate}
\end{lemma}

\begin{proof}
For part (i), by definition of inward flow, we know that there exists
$\varepsilon>0$ such that $\z - \varepsilon \g(\z) = \bomega$ for some
$\bomega \in C$. We can thus write    
\begin{align*}
\langle \v, \g(\z) \rangle 
&= \frac{1}{\varepsilon} \langle \v, \varepsilon \g(\z) \rangle \\ 
&= \frac{1}{\varepsilon} \langle \v, \z - \bomega \rangle  \\
&\geq 0,
\end{align*}
where the inequality holds as \smash{$\v \in N_C(\z)$}. For part (ii), if $\z \in C$, the result is trivial. Now consider $\z \notin C$. Abbreviating
\smash{$\z_0 = \Pi_C(\z)$}, by definition of Euclidean projection, there exists
\smash{$\v \in N_C(\z_0)$} such that $\z = \z_0 + \v$. As $\langle \v, \g(\z_0)
\rangle \geq 0$ by part (i), it follows that $\langle \z, \g(\z_0) \rangle =
\langle \z_0, \g(\z_0) \rangle + \langle \v, \g(\z_0) \rangle \geq \langle \z_0,
\g(\z_0) \rangle$, and this completes the proof.   
\end{proof}

\subsection{Calibration theory for MultiQT}

What remains now is to prove Lemmas \ref{lemma:multiqt_cond1} and
\ref{lemma:multiqt_cond2}, which we do below.

\subsubsection{Proof of \Cref{lemma:multiqt_cond1}}

Suppose $\|\btheta_t\|_2 > h$. This implies there exists $\alpha^* \in \cA$
such that \smash{$|\theta^{\alpha^*}_t| > h/\sqrt{m}$}, because otherwise we
would have \smash{$\|\btheta_t\|_2^2 = {\sum_{i=1}^m (\theta_t^{\alpha_i})^2}
  \leq {\sum_{i=1}^m h^2/m} = h^2$}. Now expand the inner product:     
\begin{align*} 
\langle \btheta_t, \g_t(\btheta_t) \rangle 
&= \sum_{\alpha \in \cA} \theta_t^{\alpha} (\cov_t^{\alpha} - \alpha) \\
\numberthis\label{eq:restorativity_proof_decomposition}
&= \sum_{\alpha : \theta_t^{\alpha} < -R} \theta_t^{\alpha} 
(\cov_t^{\alpha} - \alpha) + 
\sum_{\alpha : \theta_t^{\alpha} > R} \theta_t^{\alpha}
(\cov_t^{\alpha} - \alpha) + 
\sum_{\alpha : -R \leq \theta_t^{\alpha} \leq R} \theta_t^{\alpha} 
(\cov_t^{\alpha} - \alpha). 
\end{align*}
We will show that the first two sums must be positive and then argue that at
least one of the sums must be large. In the first summation, since
\smash{$\theta_t^{\alpha} < -R$}, we must have \smash{$\cov_t^{\alpha} = 0$}, so  
\smash{$\cov_t^{\alpha} - \alpha = -\alpha$} is negative; thus each summand is 
positive. In the second summation, since \smash{$\theta_t^{\alpha} > R$}, we
must have \smash{$\cov_t^{\alpha} = 1$}, so \smash{$\cov_t^{\alpha} - \alpha =
  1-\alpha$} is positive; thus each of these summands is also positive.   

To see that at least one of the sums must be large, observe that since \smash{$h
  \geq Rm^{3/2} / d_{\cA}$} by assumption, we have \smash{$|\theta^{\alpha^*}_t|
  > h/\sqrt{m} \geq Rm / d_{\cA} \geq R$}. Thus we know that $\alpha^*$ must
appear in the indices of one of the first two summations. If $\alpha^*$ appears
in the first summation, this means \smash{$\theta^{\alpha^*}_t < -h/\sqrt{m}$},
so the first summation can be lower bounded by \smash{$h \alpha^*/\sqrt{m}$}.
If $\alpha^*$ appears in the second summation, the second summation can
be similarly lower bounded by \smash{$h (1-\alpha^*)/\sqrt{m}$}. Combining, we
conclude that the first two sums in \eqref{eq:restorativity_proof_decomposition}
can be lower bounded as    
\[
\sum_{\alpha : \theta_t^{\alpha} < -R} \theta_t^{\alpha} 
(\cov_t^{\alpha} - \alpha) + 
\sum_{\alpha : \theta_t^{\alpha} > R} \theta_t^{\alpha} 
(\cov_t^{\alpha} - \alpha) 
\geq \frac{h \min(\alpha^*, 1-\alpha^*)}{\sqrt{m}} 
\geq \frac{h d_{\cA}}{\sqrt{m}}.
\] 
The third sum in \eqref{eq:restorativity_proof_decomposition} is lower bounded
by $-Rm$, since \smash{$\cov_t^{\alpha} - \alpha \in [-1,1]$}. Plugging this all
back into \eqref{eq:restorativity_proof_decomposition}, we get
\begin{equation}
\label{eq:restorativity_proof_conclusion}
\langle \btheta_t, \g_t(\btheta_t) \rangle \geq \frac{h d_{\cA}}{\sqrt{m}} - Rm.
\end{equation}
Note that the right-hand side is nonnegative for any \smash{$h \geq Rm^{3/2} /
  d_{\cA}$}.

\subsubsection{Proof of \Cref{lemma:multiqt_cond2}}

We will show that for any $\theta_t \in C_t = \cK - \b_t$, there exists $\delta
> 0$ such that $\theta_t - \varepsilon g_t(\theta_t) \in C_t$, for all
$\varepsilon \leq \delta$. That is, we will show that if we took a small enough
step in the direction of the negative gradient in \eqref{eq:multiqt_gradient}
starting from $\btheta_t$, then the quantiles would remain uncrossed. We do so
by first arguing that we do not have to worry about crossings between quantiles
on the same side of $y_t$, then arguing that the quantiles which sandwich $y_t$
must be separated by a positive distance, allowing us to maintain proper ordering
for small enough $\delta$.

For $\varepsilon>0$, let $\bomega = \btheta_t - \varepsilon
\g_t(\btheta_t)$, with elements \smash{$\omega^{\alpha} = \theta_t^{\alpha}  -
  \varepsilon (\cov_t^{\alpha} - \alpha)$}, $\alpha \in \cA$. We want to show
that for small enough $\varepsilon$, we  
have $\bomega \in C_t$. In other words, we must verify \smash{$\omega^{\alpha_i}
  + b_t^{\alpha_i} \leq \omega^{\alpha_{i+1}} + b_t^{\alpha_{i+1}}$} for $i =
1,\dots,m-1$. First we show that for any pair $\alpha < \beta$, if
\smash{$\cov_t^{\alpha} = \cov_t^{\beta}$}, then \smash{$\omega^{\alpha} + 
  b_t^{\alpha} \leq \omega^{\beta} + b_t^{\beta}$} for any $\varepsilon >0$. 
To see this, observe
\begin{align*}
\omega^{\beta} + b_t^{\beta} - (\omega^{\alpha} + b_t^{\alpha}) 
&= \theta^{\beta}_t - \varepsilon (\cov_t^{\beta} - \beta) + 
b_t^{\beta} - [\theta^{\alpha}_t - \varepsilon (\cov_t^{\alpha} -\alpha) + 
b_t^{\alpha}] \\
&= \theta^{\beta}_t + b_t^{\beta} -  (\theta^{\alpha}_t + b_t^{\alpha}) - 
\varepsilon (\cov_t^{\beta} - \cov_t^{\alpha} - \beta + \alpha) \\
&\geq  \theta^{\beta}_t + b_t^{\beta} -  (\theta^{\alpha}_t + b_t^{\alpha}) \\
&\geq 0,
\end{align*}
where the third line uses \smash{$\cov_t^{\alpha} = \cov_t^{\beta}$} and $\beta
> \alpha$, and the fourth uses $\theta_t \in C_t$. Because 
\[
\theta_t^{\alpha_1} + b_t^{\alpha_1} \leq \theta_t^{\alpha_2} + b_t^{\alpha_2}
\leq \dots \leq \theta_t^{\alpha_m} + b_t^{\alpha_m},
\]
we know that
\[
0 \leq \cov_t^{\alpha_1} \leq \cov_t^{\alpha_2} \leq \dots \leq
\cov_t^{\alpha_m} \leq 1.
\]
Thus, there exists $k \in \{-1, 0, \dots, m\}$ such that
\smash{$\cov_t^{\alpha_i} = 0$} for all $i \leq k$ and \smash{$\cov_t^{\alpha_i}
  = 1$} for all $i > k$. For $i<k$ and $i>k$, we have \smash{$\omega^{\alpha_i}
  + b_t^{\alpha_i} \leq \omega^{\alpha_{i+1}} + b_t^{\alpha_{i+1}}$} for any
$\varepsilon > 0$ by the fact above. The only case that remains to check is
$i=k$. 
If $k=-1$ or $k=m$, then this means \smash{$\cov_t^{\alpha}$} is the same for
all quantile levels, so we are done, by the fact proven above. Now
consider $0 \leq k \leq m-1$. Since \smash{$\cov_t^{\alpha_k} = 0$} and 
\smash{$\cov_t^{\alpha_{k+1}} = 1$}, this implies 
\[
y_t \in (\theta_t^{\alpha_k} + b_t^{\alpha_k}, 
\theta_t^{\alpha_{k+1}} + b_t^{\alpha_{k+1}}],
\]
which implies \smash{$(\theta_t^{\alpha_{k+1}} + b_t^{\alpha_{k+1}}) -
  (\theta_t^{\alpha_k} + b_t^{\alpha_k}) > 0$}. Informally, since
\smash{$\theta_t^{\alpha_k}$} and \smash{$\theta_t^{\alpha_{k+1}}$} are
separated by a positive amount, we can increase \smash{$\theta_t^{\alpha_k}$} by  
a little and decrease \smash{$\theta_t^{\alpha_{k+1}}$} by a little and still
maintain the ordering. Formally, setting \smash{$\delta =
  [(\theta_t^{\alpha_{k+1}} + b_t^{\alpha_{k+1}}) - (\theta_t^{\alpha_k} +
  b_t^{\alpha_k})] / 2$}, we see that for any $\varepsilon \leq \delta$, 
\begin{align*}
\omega^{\alpha_{k+1}} + b_t^{\alpha_{k+1}} - (\omega^{\alpha_k} +
  b_t^{\alpha_k}) &= 
\theta^{\alpha_{k+1}}_t - \varepsilon (\cov_t^{\alpha_{k+1}} - \alpha_{k+1}) +
  b_t^{\alpha_{k+1}} - [\theta^{\alpha_k}_t - \varepsilon (\cov_t^{\alpha_k} -
  \alpha_k) + b_t^{\alpha_k}] \\
&\geq \theta^{\alpha_{k+1}}_t + b_t^{\alpha_{k+1}} - ( \theta^{\alpha_k}_t +
  b_t^{\alpha_k})  - 2\varepsilon \\ 
&\geq 0,
\end{align*}
where the second line is due to \smash{$|\cov_t^{\alpha} - \alpha| \leq 1$} for
any $\alpha \in [0,1]$, and the third line is due to the choice of
$\varepsilon$.  

\subsection{Regret of lazy \GD{} with delay}

To build up towards a proof of \Cref{thm:multiQT_regret_with_delay}, we first derive a
general bound on the regret of lazy gradient descent with delay, in a setting
with time-varying constraint sets $C_t \subseteq \R^m$, $t = 1,2,\dots$. The time-varying
constraint sets are what makes this problem unusual, and to our knowledge,
standard regret bounds do not apply in this setting. While it might be possible
to adapt more general results for lazy gradient updates (or follow the
regularized leader) under adaptive regularization, e.g., as surveyed in
\cite{mcmahan2017survey}, we provide a relatively simple and self-contained
analysis below, which leverages inward flow. 

\begin{theorem}
\label{thm:lazyGD_regret}
Assume that for each $t$, the loss function $\ell_t$ is $L$-Lipschitz and
convex, the set $C_t$ is closed and convex, and the pair $(\ell_t, C_t)$
satisfies inward flow. Then, for all $T \geq 
1$ and $\u \in \R^m$, the lazy \GD{} iterates produced by
\eqref{eq:lazyGD_update} and \eqref{eq:lazyGD_projection} satisfy 
\[
\frac{1}{T} \sum_{t=1}^T \ell_t(\btheta_t) - \frac{1}{T} \sum_{t=1}^T \ell_t(\u)
\leq \frac{\|\bttheta_1 - \u\|_2^2}{2\eta T}  + \frac{\eta (2D+1) L^2}{2}.  
\]
\end{theorem}

\begin{proof}
By convexity, $\ell_t(\u) \geq \ell_t(\btheta_t) + \langle \g_t(\btheta_t),
\u-\btheta_t \rangle$. Rearranging and summing over $t$ gives   
\begin{equation}
\label{eq:lazyGD_regret_bound1}
\sum_{t=1}^T (\ell_t(\btheta_t) - \ell_t(\u))
\leq \sum_{t=1}^T \langle \g_t(\btheta_t), \btheta_t - \u \rangle.
\end{equation}
By the representation \eqref{eq:lazyGD_with_delay_unrolled}, note that we can
write  
\[
\bttheta_1 - \eta \sum_{s=1}^{t-D-1} \g_s(\btheta_s) = \btheta_{t} + \v_{t} 
\]
for some $\v_{t} \in N_{C_t}(\btheta_{t})$. (Recall our convention that we
set $g_t(\btheta_t) = 0$ for $t \leq 0$.) In other words, 
\[
\btheta_{t} = \bttheta_1 - \eta \sum_{s=1}^{t-D-1} \g_s(\btheta_s) - \v_{t}, 
\]
and therefore we have
\begin{align*}
\sum_{t=1}^T \langle \g_t(\btheta_t), \btheta_t \rangle 
&= \sum_{t=1}^T \langle \g_t(\btheta_t), \bttheta_1 \rangle - \eta
  \sum_{t=1}^T \bigg\langle \g_t(\btheta_t), \sum_{s=1}^{t-D-1} \g_s(\btheta_s) 
  \bigg\rangle - \sum_{t=1}^T \langle \g_t(\btheta_t), \v_t \rangle \\
&\leq \sum_{t=1}^T \langle \g_t(\btheta_t), \bttheta_1 \rangle - \eta
  \sum_{t=1}^T \bigg\langle \g_t(\btheta_t), \sum_{s=1}^{t-D-1} \g_s(\btheta_s) 
  \bigg\rangle,
\end{align*}
where the second line holds because each summand in the third sum satisfies $\langle
\g_t(\btheta_t), \v_t \rangle \geq 0$: if $\btheta_t \in \mathrm{int}(C_t)$,
then $\v_t = 0$, otherwise if $\btheta_t \in \mathrm{bd}(C_t)$, then this inner
product is nonnegative by part (i) of Lemma \ref{lemma:inward_flow} as a
consequence of inward flow. Plugging this into \eqref{eq:lazyGD_regret_bound1}
gives     
\begin{equation}
\label{eq:lazyGD_regret_bound2}
\sum_{t=1}^T (\ell_t(\btheta_t) - \ell_t(\u))
\leq \sum_{t=1}^T \langle \g_t(\btheta_t), \bttheta_1 - \u \rangle - \eta  
\sum_{t=1}^T \bigg\langle \g_t(\btheta_t), \sum_{s=1}^{t-D-1}\g_s(\btheta_s) 
\bigg\rangle, 
\end{equation}
and from here on, we follow standard arguments for online gradient descent
(or follow the regularized leader, more generally). Beginning with the second
term in \eqref{eq:lazyGD_regret_bound2}, observe 
\begin{align*}
\bigg\langle \g_t(\btheta_t), \sum_{s=1}^{t-D-1}\g_s(\btheta_s) \bigg\rangle   
&= \bigg\langle \g_t(\btheta_t), \sum_{s=1}^{t-1} \g_s(\btheta_s) \bigg\rangle -
\bigg\langle \g_t(\btheta_t), \sum_{s=t-D}^{t-1} \g_s(\btheta_s) \bigg\rangle \\
\numberthis\label{eq:lazyGD_regret_bound3}
&\geq \bigg\langle \g_t(\btheta_t), \sum_{s=1}^{t-1} \g_s(\btheta_s) \bigg\rangle -
DL^2, 
\end{align*}
where the second line uses Lipschitzness. For the first term above, note that
\[
\bigg\langle \g_t(\btheta_t), \sum_{s=1}^{t-1} \g_s(\btheta_s) \bigg\rangle =
\frac{1}{2} \Bigg( \bigg\| \sum_{s=1}^{t}\g_s(\btheta_s) \bigg\|_2^2 -
\bigg\| \sum_{s=1}^{t-1} \g_s(\btheta_s) \bigg\|_2^2 - \|\g_t(\btheta_t)\|_2^2
\Bigg).  
\]
Summing over $t = 1,\dots,T$, the right-hand side telescopes, yielding
\[
\sum_{t=1}^T \bigg\langle \g_t(\btheta_t), \sum_{s<t} \g_s(\btheta_s)
\bigg\rangle = \frac{1}{2} \bigg\| \sum_{t=1}^T \g_t(\btheta_t) \bigg\|_2^2 -
\frac{1}{2} \sum_{t=1}^T \|\g_t(\btheta_t)\|_2^2.
\]
Substituting back into \eqref{eq:lazyGD_regret_bound3}, we get
\begin{align*}
\sum_{t=1}^T \bigg\langle \g_t(\btheta_t), \sum_{s=1}^{t-D-1} \g_s(\btheta_s)
  \bigg\rangle 
&\geq \frac{1}{2} \bigg\| \sum_{t=1}^T \g_t(\btheta_t) \bigg\|_2^2 - \frac{1}{2} 
\sum_{t=1}^T \|\g_t(\btheta_t)\|_2^2 - DL^2T \\
&\geq \frac{1}{2} \bigg\| \sum_{t=1}^T \g_t(\btheta_t) \bigg\|_2^2 - \frac{1}{2}
(2D+1) L^2T,
\end{align*}
where the second line again uses Lipschitzness, and from
\eqref{eq:lazyGD_regret_bound2} we then have   
\begin{equation}
\label{eq:lazyGD_regret_bound4} 
\sum_{t=1}^T (\ell_t(\btheta_t) - \ell_t(\u))
\leq \sum_{t=1}^T \langle \g_t(\btheta_t), \bttheta_1 - \u \rangle -
\frac{\eta}{2} \bigg\| \sum_{t=1}^T \g_t(\btheta_t) \bigg\|_2^2 + \frac{\eta}{2}
(2D+1) L^2T.
\end{equation}
Now we bound the first two terms in \eqref{eq:lazyGD_regret_bound4}. Observe   
\[
0 \leq \bigg\| \bttheta_1 - \u - \eta \sum_{t=1}^T \g_t(\btheta_t)  
  \bigg\|_2^2 = \|\bttheta_1 - \u\|_2^2 - 2\eta \bigg\langle \bttheta_1 - \u,    
  \sum_{t=1}^T \g_t(\btheta_t) \bigg\rangle + \eta^2 \bigg\| \sum_{t=1}^T 
  \g_t(\btheta_t) \bigg\|_2^2, 
\]
which implies
\[
\sum_{t=1}^T \langle \g_t(\btheta_t), \bttheta_1 - \u \rangle \leq 
\frac{1}{2\eta} \|\bttheta_1 - \u\|_2^2 + \frac{\eta}{2} \bigg\| \sum_{t=1}^T    
  \g_t(\btheta_t) \bigg\|_2^2.
\]
Plugging this into \eqref{eq:lazyGD_regret_bound4} yields the desired result. 
\end{proof}

\subsection{Regret theory for MultiQT}

In this section, we prove the regret bound for MultiQT, which follows easily by combining our above result on the regret of lazy mirror descent with a lemma characterizing the optimal comparator. 

\subsubsection{Proof of \Cref{thm:multiQT_regret_with_delay}}

This is a direct consequence of \Cref{thm:lazyGD_regret},
specialized to the MultiQT setting. By this result, for any $\btheta$
(which was written as $\u$ in the previous theorem),
\begin{align*}
\frac{1}{T} \sum_{t=1}^T \ell_t(\btheta_t) - \frac{1}{T} \sum_{t=1}^T
  \ell_t(\btheta) 
&\leq \frac{\|\bttheta_1 - \btheta\|_2^2}{2\eta T} + \frac{\eta (2D+1) L^2}{2}
  \\    
&\leq \frac{\|\bttheta_1\|_2^2}{\eta T} + \frac{\|\btheta\|_2^2}{\eta T} +
\frac{\eta (2D+1) L^2}{2},
\end{align*}
where the second line uses $\|a+b\|_2^2 \leq 2\|a\|_2^2 + 2\|b\|_2^2$. Now we    
plug in an optimal point $\btheta^*$, which is defined to minimize the aggregate
quantile loss, and we use the optimal comparator lemma below, which says that 
$\|\btheta^*\|_2^2 \leq R^2 m$. This completes the proof.      

\subsubsection{Optimal comparator lemma}

We state and prove a result used above regarding the $\ell_2$ norm of the optimal
comparator in the MultiQT setting.

\begin{lemma}
\label{lemma:optimal_comparator}
Let $\btheta^*$ be an
optimal fixed offset in hindsight, according to the MultiQT
loss --- that is, $\btheta^*$ minimizes 
\smash{$\sum_{t=1}^T \ell_t(\btheta)$} over all $\btheta \in \R^m$, where $\ell_t(\theta) = \rho_{\cA}(b_t + \theta, y_t)$. Then, under the conditions of \Cref{thm:multiQT_regret_with_delay}, we have $\|\btheta^*\|_2^2 \leq R^2 m$.  
\end{lemma}

\begin{proof}
In brief, $\btheta^*$ is the vector of empirical quantiles of the residuals,
which is guaranteed to lie in between the extremes of these residuals, implying the claimed result after a comparison inequality between $\ell_\infty$
and $\ell_2$ norms. In more detail, recall that in \eqref{eq:quantile_loss} we defined the quantile loss
$\rho_{\alpha}$ for a single level $\alpha$. From
the definition, it is immediate that  
\[
\rho_{\alpha}(b_t^{\alpha} + \theta^{\alpha}, y_t) =
\rho_{\alpha}(\theta_t^{\alpha}, y_t - b_t^{\alpha}),
\]
and therefore the MultiQT loss $\ell_t$
defined in \eqref{eq:multiqt_loss} can be
rewritten as 
\[
\ell_t(\theta) = \sum_{\alpha \in \cA} \rho_{\alpha}(\theta_t^{\alpha}, y_t -
b_t^{\alpha}).  
\]
Summing this over $t = 1,\dots,T$, we get
\[
\sum_{t=1}^T \sum_{\alpha \in \cA} \rho_{\alpha}(\theta_t^{\alpha}, y_t -
b_t^{\alpha}) = \sum_{\alpha \in \cA} \sum_{t=1}^T
\rho_{\alpha}(\theta_t^{\alpha}, y_t - b_t^{\alpha}),  
\]
of which $\btheta^*$ is the minimizer
over all $\btheta \in \R^m$. This minimization decouples into separate
minimizations per quantile level. For each level $\alpha$, by a standard result,
minimizers of the loss
\[
\sum_{t=1}^T \rho_{\alpha}(\theta_t^{\alpha}, y_t - b_t^{\alpha})   
\]
are empirical $\alpha$-level quantiles of the given data, here the residuals \smash{$y_t - 
  b_t^{\alpha}$}, $t = 1,\dots,T$. In general, this will not be unique, but any
such minimizer will lie in between the maximum and minimum values of the data, which are bounded by $-R$ 
and $R$ by assumption. Hence, we know that $\|\btheta^*\|_\infty \leq R$, 
and therefore \smash{$\|\btheta^*\|_2 \leq R \sqrt{m}$} by a standard
comparison inequality between $\ell_\infty$ and $\ell_2$ norms. 
\end{proof}






\section{Miscellaneous negative results}
\label{sec:negative_results_APPENDIX}

In this section, we prove the negative results stated in the paper.

\subsection{Proof of \Cref{prop:ordering_QT_fails}}

We prove this by constructing a counterexample, which is valid both when $G$ is the sorting operator and the isotonic projection  
operator. Take \smash{$b_t^{\alpha} = 0$} for all $\alpha$ and $t$. This implies $\q_t = \btheta_t$, so below we will reference $\btheta_t$ directly. For
simplicity, we consider only two quantile levels $\cA=\{\alpha, \beta\}$, where
$\alpha=0.5$ and $\beta=0.75$.
Recall that \smash{$\theta_t^{\alpha}$} is the $\alpha$-level QT forecast
at time $t$, and let \smash{$\htheta_t^{\alpha}$} denote the $\alpha$-level
forecast after applying the map $G$. We initialize \smash{$\theta_1^{\alpha} =
  \theta_1^{\beta} = 0$}.  
We now define a sequence $y_t$ with crossing events at a nonvanishing fraction
of time steps, resulting in the incorrect long-run coverage after applying
$G$. 
Consider the following sequence of
outcomes, visualized in Figure \ref{fig:QT_sorting_failure}:
\begin{itemize}
\item $y_1$ lands above both forecasts, so both forecasts increase and become 
  separated by a positive gap; 
\item $y_2$ lands in this gap, so the $\alpha$-level forecast decreases and the
  $\beta$-level forecast increases, and the quantiles are now crossed;  
\item $y_3$ lands in between the two crossed quantiles;
\item $y_4$ through $y_8$ are a sequence of values that cause the forecasts to
  reset to the starting point of zero at time $t=9$, at which point we repeat
  the entire subsequence ad infinitum. 
\end{itemize}
Of the eight time steps in each subsequence, \smash{$\theta_t^{\alpha}$} covers
$y_t$ four times, yielding the desired coverage of 0.5, and
\smash{$\theta_t^{\beta}$} covers $y_t$ six times out of eight, yielding the
desired coverage of 0.75. 

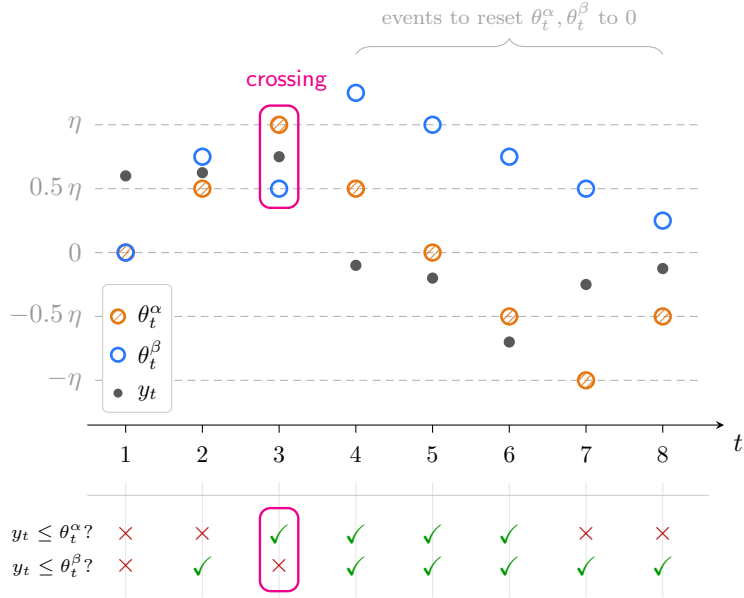
\begin{figure}[t]
\centering
\begin{tikzpicture}[x=1.2cm,y=2.0cm,>=stealth, scale=0.85]
  \definecolor{aColor}{RGB}{230,120,20}   
  \definecolor{bColor}{RGB}{40,120,255}   
  \definecolor{yColor}{gray}{0.35}        

  \draw[->] (0.5,-1.35) -- (8.8,-1.35) node[below right] {$t$};
  \foreach \t in {1,...,8} {
    \draw (\t,-1.35) -- (\t,-1.41) node[below=2pt,scale=0.9] {$\t$};
  }

  \foreach \yy/\lab in {1.0/\,$\eta$, 0.5/0.5\,$\eta$, 0/0,
                        -0.5/$-0.5\,\eta$, -1.0/$-\eta$} {
    \draw[densely dashed,gray!60] (0.6,\yy) -- (8.5,\yy);
    \node[left,gray!80] at (0.55,\yy) {\lab};
  }

  \foreach \t/\y in {1/0.00, 2/0.50, 3/1.00, 4/0.50, 5/0.00, 6/-0.50, 7/-1.00, 8/-0.50} {
    \draw[aColor, line width=1pt, pattern=north east lines, pattern color=aColor, fill opacity=0.9]
      (\t,\y) circle [radius=3.5pt];
  }

  \foreach \t/\y in {1/0.00, 2/0.75, 3/0.50, 4/1.25, 5/1.00, 6/0.75, 7/0.50, 8/0.25} {
    \draw[bColor, line width=1pt] (\t,\y) circle [radius=3.5pt];
  }

  \foreach \t/\y in {
    1/0.60,  2/0.625, 3/0.75,
    4/-0.10, 5/-0.20, 6/-0.70,
    7/-0.25, 8/-0.125
  } {
    \fill[yColor] (\t,\y) circle [radius=2.5pt];
  }

  \draw[magenta, line width=0.9pt, rounded corners=4pt]
        (2.75,0.35) rectangle (3.25,1.15);
  \node[magenta] at (3.1,1.35) {\small \textsf{crossing}};

  \draw[decorate,decoration={brace,amplitude=6pt},gray!60]
        (4.0,1.55) -- (8.0,1.55);
  \node[gray!70] at (6.0,1.85) {\footnotesize events to reset $\theta_t^\alpha,\theta_t^\beta$ to $0$};

  \foreach \t in {1,...,8} { \draw[gray!20] (\t,-1.8) -- (\t,-2.7); }
  \draw[gray!40] (0.5,-1.9) -- (8.6,-1.9);

  \node[anchor=east] at (0.7,-2.20) {\scriptsize $y_t\le \theta_t^\alpha$?};
  \node[anchor=east] at (0.7,-2.45) {\scriptsize $y_t\le \theta_t^\beta$?};


  \foreach \t/\sym in {1/\cross,2/\cross,3/\tick,4/\tick,5/\tick,6/\tick,7/\cross,8/\cross} {
    \node at (\t,-2.20) {\sym};
  }
  \foreach \t/\sym in {1/\cross,2/\tick,3/\cross,4/\tick,5/\tick,6/\tick,7/\tick,8/\tick} {
    \node at (\t,-2.45) {\sym};
  }


  \draw[magenta, line width=0.9pt, rounded corners=4pt]
        (2.75,-2.65) rectangle (3.25,-2);

  \begin{scope}[shift={(0.7,-1.25)}]
    \draw[fill=white, draw=gray!40,rounded corners=2pt] (0,0) rectangle (0.9,1.0);
    \draw[aColor, line width=1pt, pattern=north east lines, pattern color=aColor, fill opacity=0.7]
          (0.2,0.75) circle [radius=3pt];
    \node[anchor=west,scale=0.9] at (0.35,0.75) {$\theta_t^{\alpha}$};
    \draw[bColor, line width=1pt] (0.2,0.45) circle [radius=3pt];
    \node[anchor=west,scale=0.9] at (0.35,0.45) {$\theta_t^{\beta}$};
    \fill[yColor] (0.2,0.15) circle [radius=2pt];
    \node[anchor=west,scale=0.9] at (0.35,0.15) {$y_t$};
  \end{scope}
\end{tikzpicture}
\vspace{-1cm}
\caption{An example where post hoc ordering the QT iterates (via either sorting or isotonic regression) fails to achieve the correct
  coverage with two quantile levels, $\alpha=0.5$ and $\beta=0.75$. The sequence
  $y_t$ relative to \smash{$\theta_t^{\alpha}$} and \smash{$\theta_t^{\beta}$}
  is (1) above both, (2) in between, (3) in between the crossed quantiles, (4-6)
  below both, (7-8) in between both. Both forecasts are initialized to zero at
  time $t=1$ and return to zero at time $t=9$, at which point the sequence of
  $y_t$ is repeated. Success and failure of coverage is marked with
  \tick~and \cross, respectively. Averaging across each row, we see that 
  \smash{$\theta_t^{\alpha}$} achieves a coverage of 0.5 and
  \smash{$\theta_t^{\beta}$} a coverage of 0.75, as desired. However, after
  applying post hoc ordering, the coverage events at $t=3$ are modified: for sorting, the coverage events are swapped with each other, and for isotonic regression, either both coverage events become successes or both become failures. In all cases, the ordered iterates fail to achieve the correct coverage for at least one quantile level.}   
\label{fig:QT_sorting_failure}
\end{figure}

When $G$ is the sorting operator, the crossing means the coverage events at the third time step are swapped, causing
the sorted quantiles \smash{$\htheta_t^{\alpha}$} and
\smash{$\htheta_t^{\beta}$} to yield coverages of 3/8 and 7/8, respectively.    

When $G$ is the isotonic regression operator, the crossing at the third time step similarly causes a problem. By the pool adjacent violators
algorithm (PAVA) \citep{barlow1972statistical}, we know that isotonic regression
maps any pair of crossed quantiles to the same value. Thus, rather than swapping the coverage events at
$t=3$, the use of isotonic regression causes one of the
coverage events to flip. Let \smash{$\htheta_3^*$} denote the common value that 
\smash{$\theta_3^{\alpha}$} and \smash{$\theta_3^{\beta}$} are mapped to by
isotonic regression (their average). Now, if \smash{$y_3 \leq \htheta_3^*$},
then the coverage indicators for the ordered quantiles at $t=3$ will both be
one; conversely, if \smash{$y_3 > \htheta_3^*$}, then both indicators will both
be zero. In either case, coverage will fail to be obtained by the ordered 
quantiles at one of the levels (the one whose indicators flipped after applying
isotonic regression).  

\subsection{Proof of \Cref{prop:PGD_fails}}

We construct a counterexample where projected \GD{} fails to yield
coverage. As in the proof of \Cref{prop:ordering_QT_fails}, let
\smash{$b_t^{\alpha} = 0$} for all $\alpha$ and $t$, and $\cA=\{\alpha,
\beta\}$, where now $\alpha < \beta$ and $\alpha + \beta = 0.5$. 
Initialize \smash{$\theta^{\alpha}_1 = \theta^{\beta}_1 = q$} for some
$q\in\R$. Suppose we observe $y_1 > q$, so
\smash{$\ttheta^{\alpha}_2=q+\eta\alpha$} and
\smash{$\ttheta^{\beta}_2=q+\eta\beta$}. Since $\alpha < \beta$, the quantiles
are ordered and therefore we have \smash{$\theta_2^{\alpha} =
  \ttheta_2^{\alpha}$} and \smash{$\theta_2^{\beta} = \ttheta_2^{\beta}$}.    

Now suppose we observe \smash{$y_2 \in (\theta^{\alpha}_2, \theta^{\beta}_2]$},
so we updates the hidden iterates to
\begin{align*}
\ttheta^{\alpha}_3 &= q + 2\eta\alpha, \\
\ttheta^{\beta}_3 &= q + \eta\beta - \eta(1-\beta) = q + \eta(2\beta-1).  
\end{align*}
Since $\beta = 0.5 - \alpha$, we have $2\beta - 1 = -2\alpha$, so a crossing has
occurred: \smash{$\ttheta^{\alpha}_3 > \ttheta^{\beta}_3$}. By the pool adjacent
violators algorithm (PAVA) \citep{barlow1972statistical}, we know that isotonic
regression will map these two values to their average: 
\[
\theta^{\alpha}_3 = \theta^{\beta}_3 
= \frac{\ttheta_3^{\alpha} + \ttheta_3^{\beta}}{2} 
= q + \frac{\eta (2 (\alpha + \beta) -1)}{2} = q,
\]
where the last equality uses $\alpha + \beta = 0.5$. This puts us back to the 
starting point from $t=1$; we can therefore repeat this process, so that the
$\alpha$-level forecasts achieve a coverage of 0.5 and the $\beta$-level
forecasts achieve a coverage of 1.    

\subsection{MultiQT with sorting}

The next result highlights the importance of using isotonic projection to
enforce the ordering constraints in MultiQT, as it shows that replacing this
with sorting does not achieve calibration in general.  

\begin{proposition}\label{prop:multiqt_with_sorting_fails}
Let $\q_t, t=1,2,\dots$ be forecasts obtained by running Procedure
\ref{proc:multiQT} but where the projection step in
\eqref{eq:multiQT_projection} is replaced with sorting---that is, $\q_t \in
\R^m$ is set equal to the vector which results from sorting the entries of
\smash{$\b_t + \bttheta_t \in \R^m$}. Then, there exists a set of levels $\cA$ 
and sequence of target values and base forecasts $(y_t,\b_t)$ with bounded errors 
(i.e., \smash{$|y_t - b_t^{\alpha}|$} is bounded for all $\alpha$ and $t$) such
that for any learning rate $\eta > 0$ the forecasts fail to achieve calibration :
\smash{$\lim_{T \to \infty} \frac{1}{T} \sum_{t=1}^T \one\{y_t \leq
  q_t^{\alpha}\} \neq \alpha$} for some $\alpha \in \cA$.
\end{proposition}

\begin{proof}
As in previous counterexamples, we set \smash{$b_t^{\alpha} = 0$} for all
$\alpha$ and $t$, thus $\q_t = \btheta_t$ for all $t$, and $\cA=\{\alpha,  
\beta\}$, where $0 < \alpha < \beta < 1$ and $\alpha = 1-\beta$. Initialize
\smash{$\ttheta^{\alpha}_1 < \ttheta^{\beta}_1$}, such that
\smash{$\ttheta^{\beta}_1 - \ttheta^{\alpha}_1 < 2 \eta \alpha$}. Since these
are ordered, we have \smash{$\theta_1^{\alpha} = \ttheta_1^{\alpha}$} and 
\smash{$\theta_1^{\beta} = \ttheta_1^{\beta}$}.

Suppose we observe \smash{$y_1 \in (\theta^{\alpha}_1, \theta^{\beta}_1]$}, so
the hidden iterate updates are 
\begin{align*}
\ttheta^{\alpha}_2 &= \ttheta^{\alpha}_1 + \eta\alpha, \\
\ttheta^{\beta}_2 &= \ttheta^{\beta}_1 - \eta(1-\beta). 
\end{align*}
These are crossed, as \smash{$\ttheta^{\beta}_2 - \ttheta^{\alpha}_2 = 
\ttheta^{\beta}_1 - \ttheta^{\alpha}_1 - \eta(1-\beta+\alpha) =
\ttheta^{\beta}_1 - \ttheta^{\alpha}_1 - 2\eta\alpha < 0$}, by assumption.  
Sorting yields the played updates: \smash{$\theta_2^{\alpha} =
  \ttheta_2^{\beta}$} and \smash{$\theta_2^{\beta} = \ttheta_2^{\alpha}$}, and
the key realization for this counterexample is that the played iterates now have
a  \emph{positive} gap (instead of zero gap, as with isotonic projection), which
can be exploited to continue driving them farther away from each other.

In particular, suppose \smash{$y_2 \in (\theta^{\alpha}_2, \theta^{\beta}_2]$},
so the hidden iterate updates are
\begin{align*}
\ttheta^{\alpha}_3 &= \ttheta^{\alpha}_2 + \eta\alpha, \\
\ttheta^{\beta}_3 &= \ttheta^{\beta}_2 - \eta(1-\beta). 
\end{align*}
Since \smash{$\ttheta^{\alpha}_2$} and \smash{$\ttheta^{\beta}_2$} are already
crossed, this update keeps them crossed (and increases the gap in between
them). Sorting yields played iterates \smash{$\theta_3^{\alpha} =
  \ttheta_3^{\beta}$} and \smash{$\theta_3^{\beta} = \ttheta_3^{\alpha}$}.
This can be repeated ad infinitum, causing \smash{$\ttheta_t^{\alpha} \to \infty$}
and  \smash{$\ttheta_t^{\beta} \to -\infty$} as $t \to \infty$. Hence, the
$\beta$-level coverage goes to 1 and $\alpha$-level coverage goes to 0. 
\end{proof}

\subsection{MultiQT with positively separated quantiles}

The next result highlights the importance of not only projection, but
specifically projection to the (shifted) isotonic cone; modifying the constraint
set to induce quantiles separated by $\varepsilon > 0$ fails to achieve 
calibration in general.     

\begin{proposition}\label{prop:multiqt_with_epsilon_separation_fails} 
Consider the set defined in \eqref{eq:eps_separation}, which can be equivalently
written as \smash{$C_t^{\varepsilon} = \cK^{\varepsilon} - \b_t$}, where 
\[
\cK^{\varepsilon} = \Big\{ x \in \R^m : x_i + \varepsilon \leq x_{i+1},
\; i = 1,2, \dots, m-1 \Big\}
\]
for $\varepsilon > 0$.
Let $\q_t, t=1,2,\dots$ be forecasts obtained by running Procedure
\ref{proc:multiQT} but where the projection step in
\eqref{eq:multiQT_projection} is replaced with \smash{$\q_t =
  \Pi_{\cK^{\varepsilon}}(\b_t + \bttheta_t)$.} Then, for any $\varepsilon > 0$, 
there exists a set of levels $\cA$ and sequence of target values and base 
forecasts $(y_t,\b_t)$ with bounded errors (i.e., \smash{$|y_t - b_t^{\alpha}|$}
is bounded for all $\alpha$ and $t$) such that for any learning rate $\eta > 0$
the forecasts fail to achieve calibration : \smash{$\lim_{T \to \infty} \frac{1}{T}
  \sum_{t=1}^T \one\{y_t \leq q_t^{\alpha}\} \neq \alpha$} for some $\alpha \in \cA$. 
\end{proposition}

\begin{proof}
We adapt the construction from the proof of
\Cref{prop:multiqt_with_sorting_fails}. Consider the same initialization and initial target $y_1$. Under projection onto
\smash{$\cK^{\varepsilon}$}, the played updates satisfy
\smash{$\theta_2^{\beta} - \theta_2^{\alpha} \geq \varepsilon$}, creating a 
strictly positive gap. By choosing \smash{$y_2 \in (\theta^{\alpha}_2,
  \theta^{\beta}_2]$}, we can continue driving the played iterates away from one 
another. This results in the same limiting behavior as in the previous
construction, where coverage at the upper level converges to 1 and coverage at
the lower level converges to 0.  
\end{proof}


\section{Fast rates for constrained gradient equilibrium}
\label{sec:proofs_for_faster_rate_APPENDIX}

We derive fast rates for constrained gradient equilibrium and, consequently, calibration without crossings by refining the
analysis to leverage a positive curvature assumption. We note that these
results could be extended to the setting of delayed feedback, but for
simplicity, we state and prove all results in the no-delay setting ($D=0$). 

\subsection{Proof of \Cref{prop:GEQ_with_bounded_distance}}

We follow the general proof structure from Proposition 5 of
\cite{angelopoulos2025gradient}. For convenience, let us redefine \smash{$h_t =  
\max\{\|\bttheta_1\|_2, h_t\}$} and let \smash{$h_0 = \|\bttheta_1\|_2$}. 
We will use induction to show $\|\bttheta_{T+1}\|_2 \leq h_T + B + \eta L$ for
all $T \geq 0$. The base case for $T=0$ holds trivially. For the inductive step,
assume the inequality holds up through $T$. We split into two cases. First, if
$\|\bttheta_T\|_2 \leq h_T + B$, then by the triangle inequality we have 
\begin{align*}
\|\bttheta_{T+1}\|_2 
&\leq \|\bttheta_T\|_2 + \eta \|g_T(\btheta_T)\|_2 \\
&\leq h_T + B + \eta L,
\end{align*}
where the second inequality invokes Lipschitzness of the loss. Second, if
$\|\bttheta_T\|_2 > h_T + B$, then
\begin{align*}
\|\bttheta_{T+1}\|_2^2 
&= \|\bttheta_T\|_2^2 + \eta^2 \|g_T(\btheta_T)\|_2^2 - 
  2 \eta \langle \bttheta_T, g_T(\btheta_T) \rangle \\
&\leq \|\bttheta_T\|_2^2 + \eta^2 L^2 - 
  2 \eta \langle \bttheta_T, g_T(\btheta_T) \rangle \\
&\leq \|\bttheta_T\|_2^2 + \eta^2 L^2 - 2 \eta \phi_T(\btheta_T) \\
&\leq \|\bttheta_T\|_2^2 \\
&\leq (h_T + B + \eta L)^2 \\
&\leq (h_{T+1} + B + \eta L)^2
\end{align*}
where second line applies Lipschitzness, the third is discussed below, the fourth uses the assumed curvature
condition on $\phi_T(\btheta_T)$, the fifth applies the inductive hypothesis,
and the sixth uses the increasing property of $h_t$. Taking a square root
would conclude the inductive step.    

It remains to verify the third line above, which uses \smash{$\langle
  \bttheta_T, \g_T(\btheta_T) \rangle \geq \phi_T(\btheta_T)$}. This is due to
restorativity, inward flow, and the bounded distance assumption. In particular,
note that by the triangle inequality 
\[   
\|\btheta_T\|_2 > \|\bttheta_T\|_2 - \|\btheta_T - \bttheta_T\|_2 \geq h + B - B
= h.  
\]
Thus, by restorativity of $\ell$, we have $\langle \btheta_T, \g_T(\btheta_T) \rangle
\geq \phi_T(\btheta_T)$. By inward flow, we can then apply
part (ii) of \Cref{lemma:inward_flow} to get \smash{$\langle \bttheta_T, \g_T(\btheta_T) \rangle 
  \geq \phi_T(\btheta_T)$} as desired. This completes the proof of the iterate
bound.

For the average gradient bound, we simplify the iterate bound using $\max\{a,b\}
\leq a + b$, and then apply \eqref{eq:lazyGD_avg_gradient_with_delay} with
$D=0$. 

\subsection{Proof of \Cref{lemma:point_forecasts_bounded_distance}}

Our proof has two main steps. First, we will show that if the base forecaster is
a point forecaster, then the entries of \smash{$\bttheta_t$} cannot get too
crossed: specifically, \smash{$\bttheta^{\alpha_i}_t - \ttheta_t^{\alpha_{i+1}}
  \leq \eta$} for all $i = 1,2,\dots, m-1$. We then bound the projection 
distance \smash{$\|\bttheta_t - \btheta_t\|_2 = \|\bttheta_t  - 
  \Pi_{\cK-b_t}(\bttheta_t)\|_2$} for any \smash{$\bttheta_t$} satisfying this
crossing bound to obtain the desired result.     

Before beginning with the first step, we make the following important
observation about the MultiQT iterates when the base forecasts are point
forecasts: if \smash{$b_t^{\alpha} = \mu_t$} for all $\alpha$, then the played
iterate is simply the result of running isotonic regression on the hidden
iterate, i.e.,   
\[
\btheta_t = \Pi_{\cK}(\bttheta_t).
\]
To see why, recall that by definition \smash{$\btheta_t =
  \Pi_{\cK-b_t}(\bttheta_t)$}, but for $b_t = \mu_t \mathbf{1}$, where $\mathbf{1} \in \R^m$ denotes the vector of all ones, the isotonic cone is shift-invariant: $\cK - \mu_t \mathbf{1} = 
\cK$. 

We now show that \smash{$\bttheta^{\alpha_i}_t - \ttheta_t^{\alpha_{i+1}} \leq
  \eta$} for all $i = 1,2,\dots, m-1$ and all times $t$, by induction on
$t$. The base case holds by assumption: \smash{$\bttheta_1$} in the lemma is
assumed to lie in $\cK$. Now assume the statement holds through time $t$. Define
\smash{$\Delta_t^i = \ttheta_t^{\alpha_{i+1}} - \bttheta^{\alpha_i}_t$}, where 
\smash{$\Delta^i_t < 0$} means a crossing has occurred, and \smash{$\Delta^i_t 
  \geq 0$} means entries $i$ and $i+1$ of \smash{$\bttheta_t$} are ordered. Fix
any $i$. We break our analysis into two cases.
\begin{itemize}
\item \emph{Case 1}: \smash{$\Delta^i_t < 0$} (the entries are crossed at time 
  $t$). In this case, isotonic regression will pool entries $i$ and $i+1$ so
  that \smash{$q_t^{\alpha_i} = q_t^{\alpha_{i+1}}$}
  \citep{barlow1972statistical}, which implies \smash{$\cov_t^{\alpha_i} = 
    \cov_t^{\alpha_{i+1}}$}. Thus,     
  \[
  \Delta^i_{t+1} = \Delta^i_t - \eta [(\cov_t^{\alpha_{i+1}} - \alpha_{i+1}) -
  (\cov_t^{\alpha_i} - \alpha_i)] = \Delta^i_t + \eta(\alpha_{i+1} - \alpha_i)
  \geq \Delta_t^i \geq - \eta, 
  \]
  where the last inequality follows from the inductive hypothesis.

\item \emph{Case 2}: \smash{$\Delta^i_t \geq 0$} (the entries are ordered at
  time $t$). In this case, \smash{$q_t^{\alpha_i} \leq q_t^{\alpha_{i+1}}$}, 
  so \smash{$\cov_t^{\alpha_{i+1}} \geq \cov_t^{\alpha_i}$}, and 
  \[
  \Delta^i_{t+1} = \Delta^i_t - \eta [(\cov_t^{\alpha_{i+1}} -
  \cov_t^{\alpha_i}) - (\alpha_{i+1} - \alpha_i)] \geq \Delta^i_t - \eta (1 -
  (\alpha_{i+1} - \alpha_i)). 
  \]
  Since \smash{$\Delta^i_t \geq 0$} and $(\alpha_{i+1} - \alpha_i) > 0$, it
  follows that \smash{$\Delta^i_{t+1} > -\eta$}. 
\end{itemize}
In both cases, \smash{$\Delta^i_{t+1} \geq -\eta$}, which establishes
\smash{$\ttheta^{\alpha_i}_{t+1} - \ttheta_{t+1}^{\alpha_{i+1}} \leq \eta$},
completing the inductive proof.  

Now fix any \smash{$\bttheta_t$} satisfying \smash{$\bttheta^{\alpha_i}_t -  
\ttheta_t^{\alpha_{i+1}} \leq \eta$} for all $i = 1,2,\dots,m-1$. Consider
constructing the ordered vector \smash{$\bar\btheta_t$} as follows: iterate
through the indices of \smash{$\bttheta_t$} and, whenever we encounter an
unordered entry, we set its value equal to that of the previous index. To make
this more explicit:
\begin{itemize}
\item we set \smash{$\bar\theta_t^{\alpha_1} = \ttheta_t^{\alpha_1}$};
\item for $i=2,\dots,m-1$, we set 
  \[
  \bar\theta_t^{\alpha_i} = \begin{cases}
  \bar\theta_t^{\alpha_{i-1}} & \text{if $\ttheta_t^{\alpha_i} <
    \bar\theta_t^{\alpha_{i-1}}$}, \\ 
  \ttheta_t^{\alpha_i} & \text{otherwise}.
  \end{cases}
  \]
\end{itemize}
Since $\bar\btheta_t \in \cK$, we have
\begin{equation}
\label{eq:projection_bound}
\|\bttheta_t - \Pi_{\cK}(\bttheta_t)\|_2 \leq \|\bttheta_t - \bar\btheta_t\|_2. 
\end{equation}
To bound the right-hand side, observe that \smash{$\ttheta_t^{\alpha_1} -
  \bar\theta_t^{\alpha_1} = 0$} and, for any $i \geq 2$, 
\[
  \ttheta_t^{\alpha_i} - \bar\theta_t^{\alpha_i} = \begin{cases} 
  \ttheta_t^{\alpha_i} - \bar\theta_t^{\alpha_{i-1}} & 
  \text{if $\ttheta_t^{\alpha_i} < \bar\theta_t^{\alpha_{i-1}}$}, \\  
  0 & \text{otherwise}.
  \end{cases}
\]
As \smash{$\ttheta_t^{\alpha_i} \geq \ttheta_t^{\alpha_{i-1}} - \eta$}, the
above display implies 
\[
\ttheta_t^{\alpha_{i-1}} - \bar\theta_t^{\alpha_{i-1}} - \eta \leq 
\ttheta_t^{\alpha_i} - \bar\theta_t^{\alpha_i} \leq 0.
\]
Therefore \smash{$|\ttheta_t^{\alpha_i} - \bar\theta_t^{\alpha_i}| \leq \eta
  (i-1)$}, and 
\[
\|\bttheta_t - \bar\btheta_t\|_2 \leq \sqrt{\sum_{i=1}^m(\eta (i-1))^2} 
= \eta \sqrt{\sum_{i=1}^m (i-1)^2}
= \eta \sqrt{\frac{m(m-1)(2m-1)}{6}}
\leq \frac{\eta m^{3/2}}{\sqrt{3}}. 
\]
Plugging this bound back into \eqref{eq:projection_bound} completes the proof.

\subsection{Proof of \Cref{cor:multiQT_point_forecast_calibration_rate}}

If we inspect the conclusion \eqref{eq:restorativity_proof_conclusion} from the
proof of \Cref{lemma:multiqt_cond1} carefully, then we see what was shown here 
is actually stronger than the conclusion stated in the lemma. Rephrased, this
proof showed that for any $c \geq 1$, the MultiQT loss $\ell_t$ is $(h_c,
\phi_c)$-restorative at all times $t$, for \smash{$h_c = c Rm^{3/2} /
  d_{\cA}$}, and $\phi_c(\theta) = (c-1) Rm$. Thus to satisfy the positive
curvature condition in \Cref{prop:GEQ_with_bounded_distance}, we require 
\[
(c-1) Rm \geq \eta L^2 / 2 \iff c \geq \eta L^2 / (2Rm) + 1.
\]
The smallest allowable value of $c$ here is $c^* = \eta L^2 / (2Rm)
+ 1$. Recalling that $L^2 = m$ for the MultiQT loss, this leads to the value
\smash{$h^* = h_{c^*} = (\eta/2 + R)(m^{3/2} /  d_{\cA})$}. 

Note that we have shown that each MultiQT loss is $(h^*, \eta L^2 /
2)$-restorative. Since, additionally, the hidden and played iterates remain
within an $\ell_2$ distance of \smash{$B = \eta m^{3/2} / \sqrt{3}$} from each
other by \Cref{lemma:point_forecasts_bounded_distance}, we can apply 
\Cref{prop:GEQ_with_bounded_distance}, which yields the result. 

\section{Additional experimental results} \label{sec:additional_experiments}

In this section, we provide supplementary empirical results that further
illustrate the behavior of MultiQT across our forecasting datasets. First, we 
recompute the main experimental results using PIT entropy as an alternative
calibration metric, in place of the $\ell_1$ calibration error. We find that our conclusions remain qualitatively unchanged. Then, we 
present additional case studies (comprehensive calibration curves and forecast 
visualizations) for COVID-19 death forecasting and energy forecasting, which
confirm that MultiQT consistently improves calibration. 

\subsection{Results using PIT entropy} 
\label{sec:pit_entropy_results_APPENDIX}


\emph{PIT entropy} \citep{gneiting2007probabilistic, rumack2022recalibrating} is
a calibration metric based on the entropy of the distribution of the probability
integral transform (PIT) values, computed from forecasts and their corresponding
targets. Specifically, given forecasts represented via cumulative distribution  
functions (CDFs) $F_t$, $t = 1,2,3,\dots$, and associated target values $y_t$, $
t= 1,2,3,\dots$, the PIT values are defined as $F_t(y_t)$, $t =
1,2,3,\dots$. The PIT entropy is then defined as the Shannon entropy of the
empirical distribution of these PIT values.  

To compute this entropy in practice, we divide the unit interval into $K = 10$  
equal-width bins. Let \smash{$\hat{p}_k$} be the empirical frequency of PIT
values for bin $k$. As our metric, we use the normalized Shannon entropy: 
\[
\hat{H} = -\frac{1}{\log K} \sum_{k=1}^K \hat{p}_k \log \hat{p}_k, 
\]
where the division by $\log K$ ensures that \smash{$\hat{H}$} lies in $[0,1]$. 
Note that under perfect calibration, the PIT values should be distributed
uniformly on $[0,1]$, which has maximal entropy. Thus, a value of
\smash{$\hat{H}$} near one indicates good calibration, while a value near zero
indicates poor calibration.    

In our setting, we have quantile forecasts $\q_t$, $t = 1,2,3,\dots$ and not
CDFs. To construct CDFs from these forecasts (so that we can compute PIT entropy),
we follow the procedure from Appendix A.1 of \cite{buchweitz2025asymmetric}.  
This uses linear interpolation in between intermediate quantiles combined with
exponential tails for values outside the extreme quantiles. Below we describe
the procedure for constructing $F_t$ from $\q_t$.
\begin{itemize}
\item If there exists $i \leq m-1$ such that \smash{$y \in
  [q_t^{\alpha_i},q_t^{\alpha_{i+1}}]$}, then \smash{$F_t(y) = \alpha_i +
  \frac{y -  q_t^{\alpha_i}}{q_t^{\alpha_{i+1}}- q_t^{\alpha_i}}(\alpha_{i+1} - 
  \alpha_i)$}.  
  \begin{itemize}
  \item Note that when ties occur (i.e.,
    \smash{$q_t^{\alpha_{i+1}}=q_t^{\alpha_{i}}$}), the interior slope on this
    segment is undefined, so when $y$ equals a tied forecast value, we set 
    $F_t(y)$ to be the largest quantile level in the tied block. Formally,
    \smash{$F_t(y)=\alpha^{i^*}$} where \smash{$i^* =\max\{i : q_t^{\alpha_i} = 
      y\}$}.     
  \end{itemize}

\item If no such $i$ exists (so \smash{$y < q_t^{\alpha_1}$} or \smash{$y >
    q_t^{\alpha_m}$}), then we use exponential tails, chosen so that the density
  at the boundary matches that at the nearest interior segment.  
  \begin{itemize}
  \item For $y < q_t^{\alpha_1}$, we first find the smallest $i$ such that
    $q_t^{\alpha_{i+1}} \neq q_t^{\alpha_i}$, and compute \smash{$\rho =
    \frac{\alpha_{i+1}-\alpha_i}{q_t^{\alpha_{i+1}} - q_t^{\alpha_i}}$}. Define  
    $\lambda =  \frac{\rho}{\alpha_1}$. We then let $F_t(y) = \alpha_1
    e^{\lambda (y-q_t^{\alpha_1})}$.  
  \item Similarly, for $y > q_t^{\alpha_m}$, we first find the largest $i$ such
    that $q_t^{\alpha_{i+1}} \neq q_t^{\alpha_i}$, and compute \smash{$\rho =  
    \frac{\alpha_{i+1}-\alpha_i}{q_t^{\alpha_{i+1}} - q_t^{\alpha_i}}$}. Define 
    $\lambda = \frac{\rho}{1-\alpha_m}$. We then let $F_t(y) = 1 - (1-\alpha_m)
    e^{-\lambda (y-q_t^{\alpha_m})}$. 
  \end{itemize}
\end{itemize}
After performing this procedure to translate a given sequence of quantile
forecasts into CDFs, we compute the PIT values and PIT entropy as described
above.   

We now reproduce the main figures from Section \ref{sec:experiments} using PIT 
entropy in place of average calibration error, as used in the main
text. Figure \ref{fig:covid_arrows_pit_entropy} is the analog of Figure 
\ref{fig:covid_arrows}, and Figure \ref{fig:energy_arrows_PIT_entropy} the
analog of Figure \ref{fig:energy_arrows}. We see qualitatively very similar
trends, and MultiQT again results in strong improvements in calibration.     

\begin{figure}[ht]
\includegraphics[width=\textwidth]{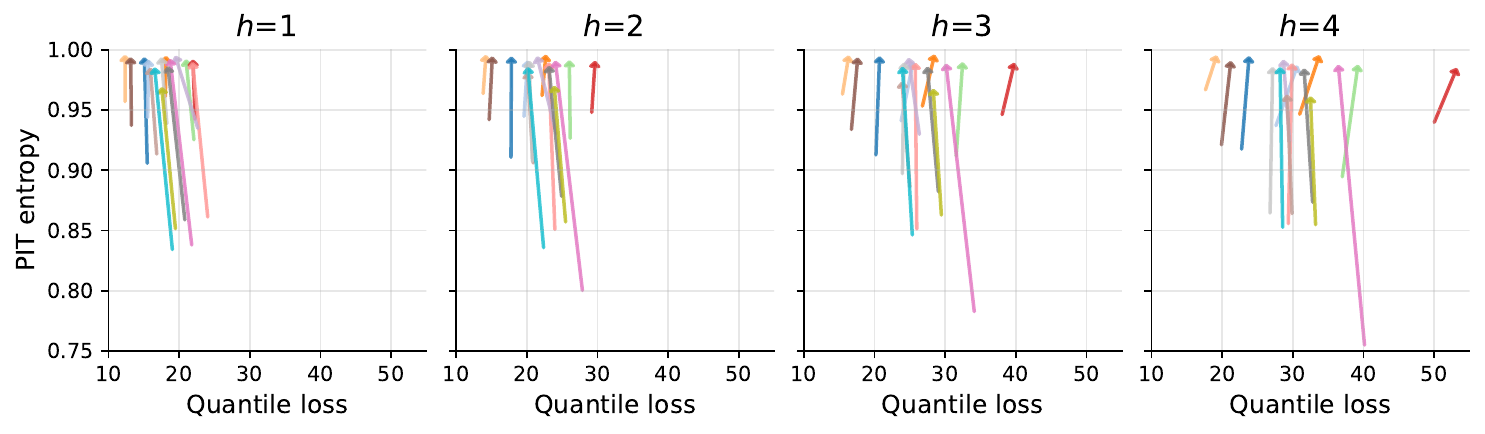}
\caption{PIT entropy versus quantile loss on the COVID-19 death dataset, analogous to Figure
  \ref{fig:covid_arrows}. For PIT entropy, higher is better.}
\label{fig:covid_arrows_pit_entropy}
\end{figure}

\begin{figure}[ht]
\begin{subfigure}{0.5\textwidth}
\includegraphics[width=\linewidth]{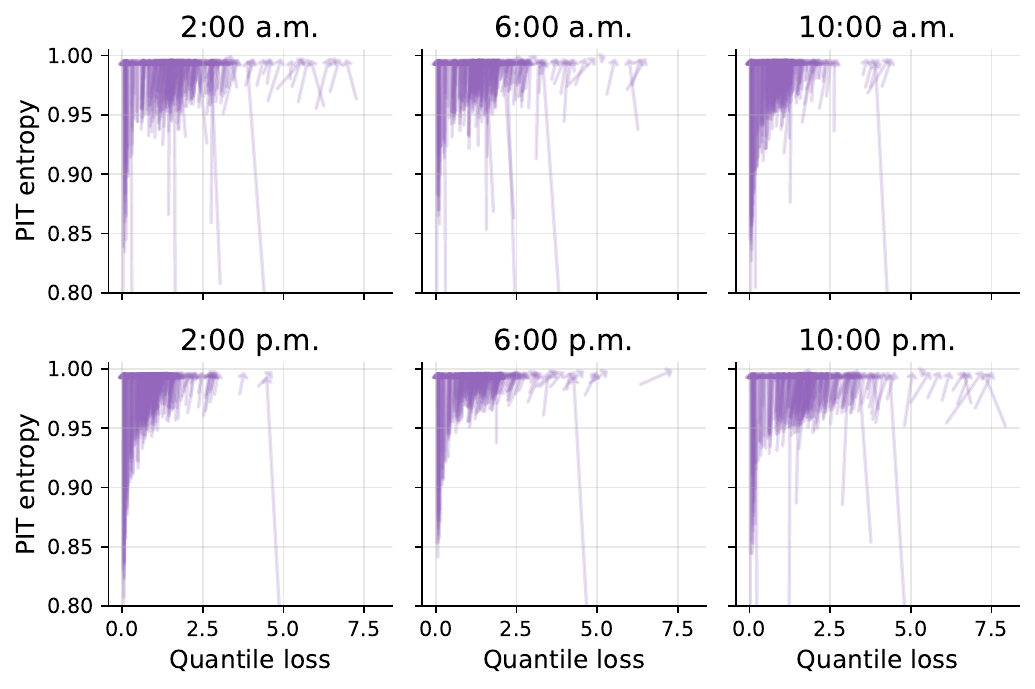}
\caption{Wind}
\label{fig:energy_wind_arrows}
\end{subfigure}%
\begin{subfigure}{0.5\textwidth}
\includegraphics[width=\linewidth]{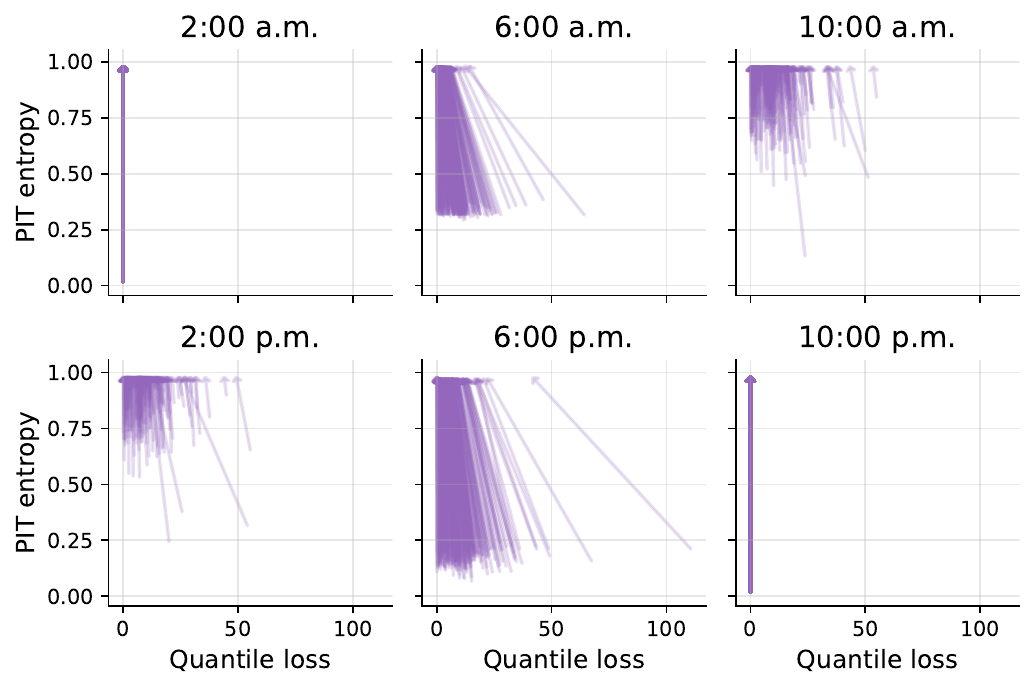}
\caption{Solar}
\label{fig:energy_solar_arrows}
\end{subfigure}
\caption{PIT entropy versus quantile loss on the energy dataset, analogous to Figure
  \ref{fig:energy_arrows}. For PIT entropy, higher is better.}
\label{fig:energy_arrows_PIT_entropy}
\end{figure}

\subsection{Additional COVID-19 forecasting results} 
\label{sec:covid_case_studies_APPENDIX}

To complement Figure \ref{fig:covid_calibration_h=1} in the main paper, which
shows actual versus desired coverage for one-week-ahead COVID-19 death
forecasters before and after applying MultiQT, Figure
\ref{fig:covid_calibration_all_h} shows the same calibration plots for all
forecasting horizons (one, two, three, and four weeks ahead). We observe that
MultiQT consistently improves calibration across all forecasting horizons.

\begin{figure}[h!]
\centering
\begin{subfigure}{0.24\textwidth}
\includegraphics[width=\linewidth]{figs/covid_calibration_raw_horizon=1.pdf}
\end{subfigure}%
\begin{subfigure}{0.24\textwidth}
\includegraphics[width=\linewidth]{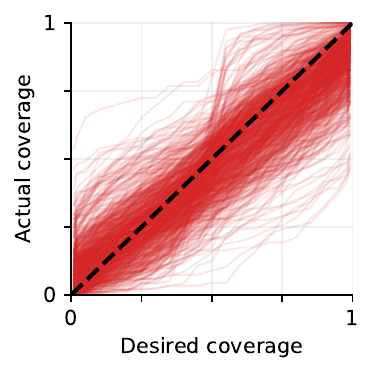}
\end{subfigure}%
\begin{subfigure}{0.24\textwidth}
\includegraphics[width=\linewidth]{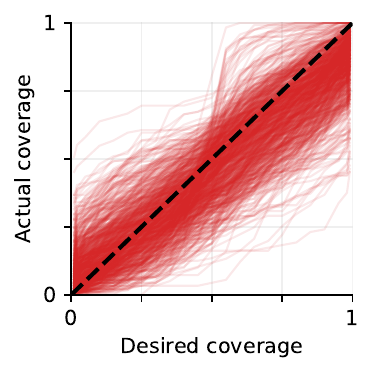}
\end{subfigure}%
\begin{subfigure}{0.24\textwidth}
\includegraphics[width=\linewidth]{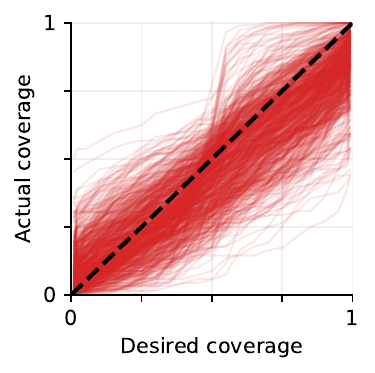}
\end{subfigure}

\begin{subfigure}{0.24\textwidth}
\includegraphics[width=\linewidth]{figs/covid_calibration_cal_horizon=1.pdf}
\caption{$h=1$}
\end{subfigure}%
\begin{subfigure}{0.24\textwidth}
\includegraphics[width=\linewidth]{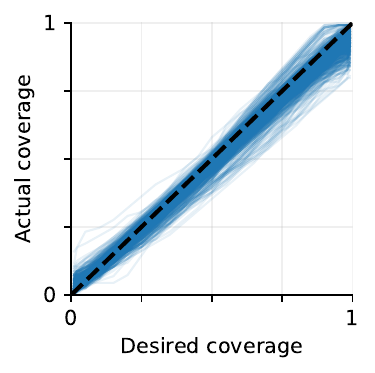}
\caption{$h=2$}
\end{subfigure}%
\begin{subfigure}{0.24\textwidth}
\includegraphics[width=\linewidth]{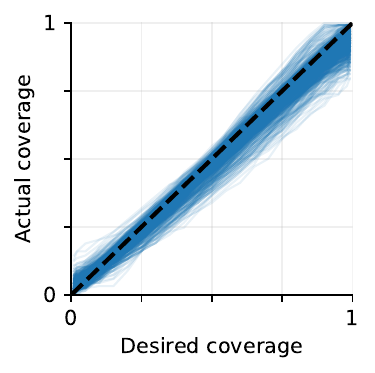}
\caption{$h=3$}
\end{subfigure}%
\begin{subfigure}{0.24\textwidth}
\includegraphics[width=\linewidth]{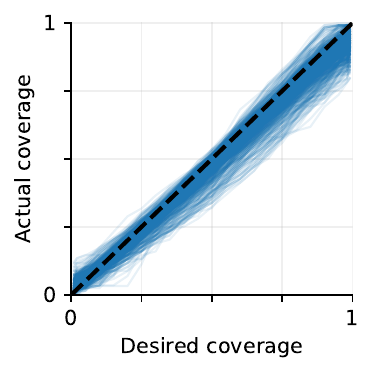}
\caption{$h=4$}
\end{subfigure}

\caption{Actual versus desired coverage for COVID-19 death forecasting,
  analogous to Figure \ref{fig:covid_calibration_h=1}. Here the top row
  represents base forecasts, and the bottom row the forecasts after applying
  MultiQT.}  
\label{fig:covid_calibration_all_h}
\end{figure}

Figures
\ref{fig:all_forecasters_ca_h=1_pt1}--\ref{fig:all_forecasters_vt_h=1_pt2}
display individual COVID-19 death forecasts before and after applying MultiQT,
analogous to Figure \ref{fig:example} in the main text. To provide a sense of
the effect of MultiQT on forecasts for states having large and small
populations, the first pair of figures (Figures
\ref{fig:all_forecasters_ca_h=1_pt1} and \ref{fig:all_forecasters_ca_h=1_pt2})  
show the effect of using MultiQT to correct one-week-ahead forecasts for California (the largest state by population) from each of the forecasting teams, and the
second pair of figures (Figures \ref{fig:all_forecasters_vt_h=1_pt1} and 
\ref{fig:all_forecasters_vt_h=1_pt2}) show the same for Vermont (one of the
smallest states).    

Recall that forecasts are made at levels 0.01, 0.025, 0.05, 0.1, 0.15, 0.2,
0.25, 0.3, 0.35, 0.4, 0.45, 0.5, 0.55, 0.6, 0.65, 0.7, 0.75, 0.8, 0.85, 0.9,
0.95, 0.975, and 0.99.  To visualize these, we plot colored bands where the
lightest opacity connects the 0.01 and 0.99 level forecasts, the next lightest
connects the 0.025 and 0.975 level forecasts, and so on. We can use these plots
to inspect calibration---if the 0.01 and 0.99 level quantile forecasts are
calibrated, then we should see that the true value falls within the lightest
opacity band 98\% of the time. Zooming in on the raw forecasts, we can see that
this is not the case for many of the forecasters initially, but after applying
MultiQT the coverage of the extreme quantiles is much improved.

\begin{figure}[p]
\includegraphics[width=\linewidth]{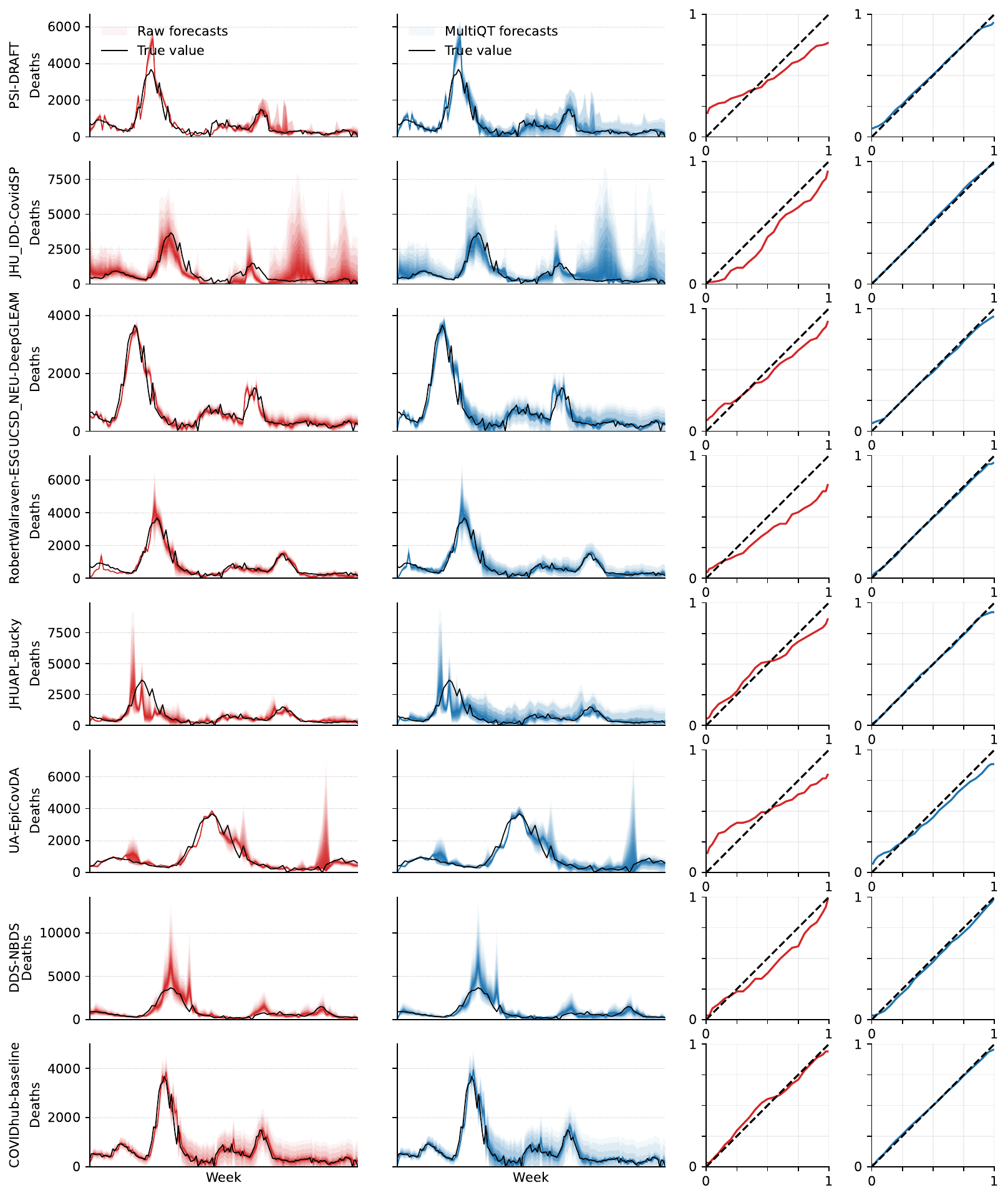}
\caption{One-week-ahead forecasts of weekly COVID-19 deaths in California (part 1 of 2). Each row corresponds to one forecaster, each with their own forecast date range. The  
first column shows the raw forecasts, the second column shows the forecasts after using MultiQT, the third column shows actual versus desired coverage for the raw forecasts, and the fourth shows the same for the MultiQT forecasts.}   
\label{fig:all_forecasters_ca_h=1_pt1}
\end{figure}

\begin{figure}[p]
\includegraphics[width=\linewidth]{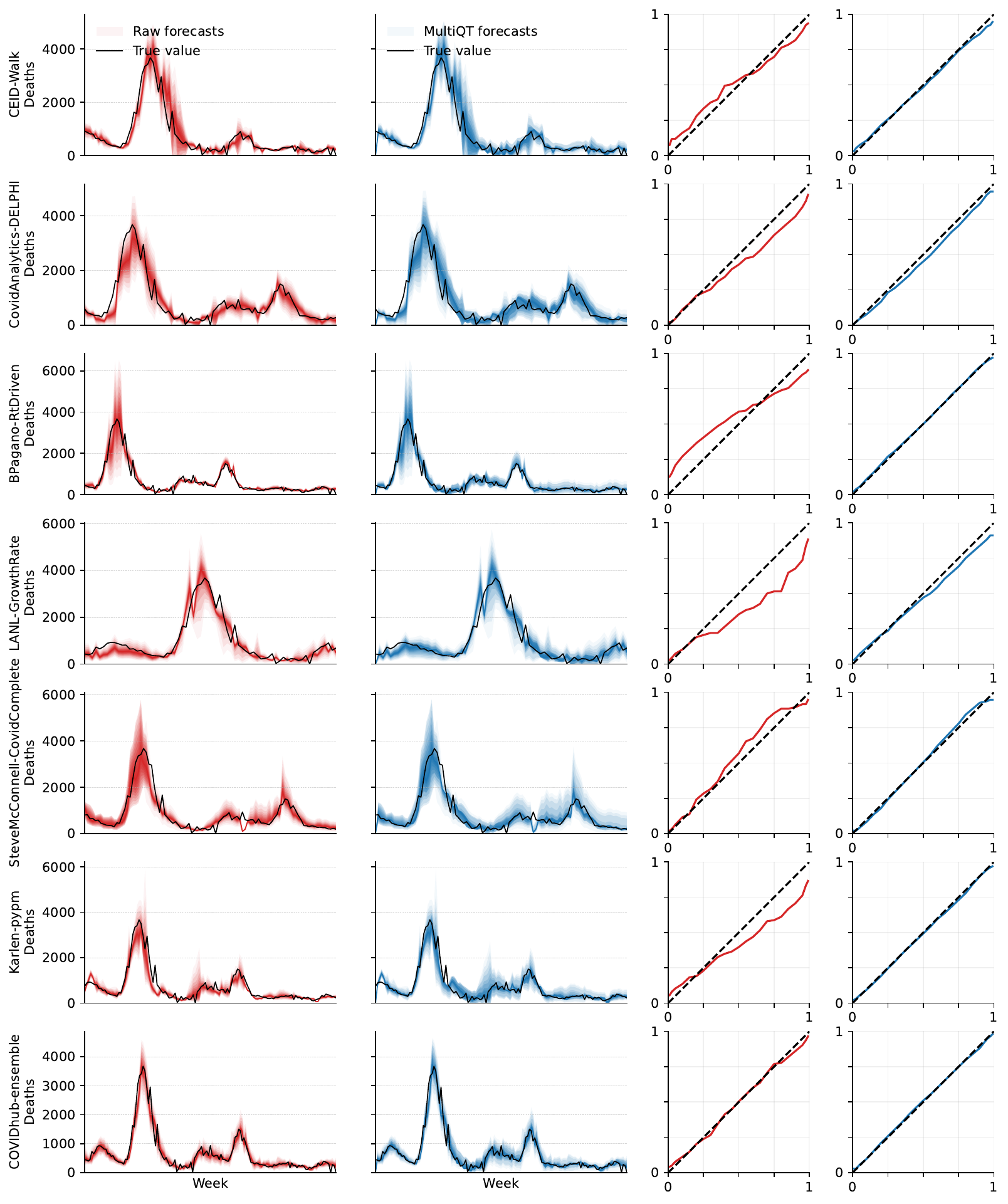}
\caption{One-week-ahead forecasts of weekly COVID-19 deaths in California (part 2 of 2), as in Figure
  \ref{fig:all_forecasters_ca_h=1_pt1}, for the remaining forecasters.}
\label{fig:all_forecasters_ca_h=1_pt2}
\end{figure}

\begin{figure}[p]
\includegraphics[width=\linewidth]{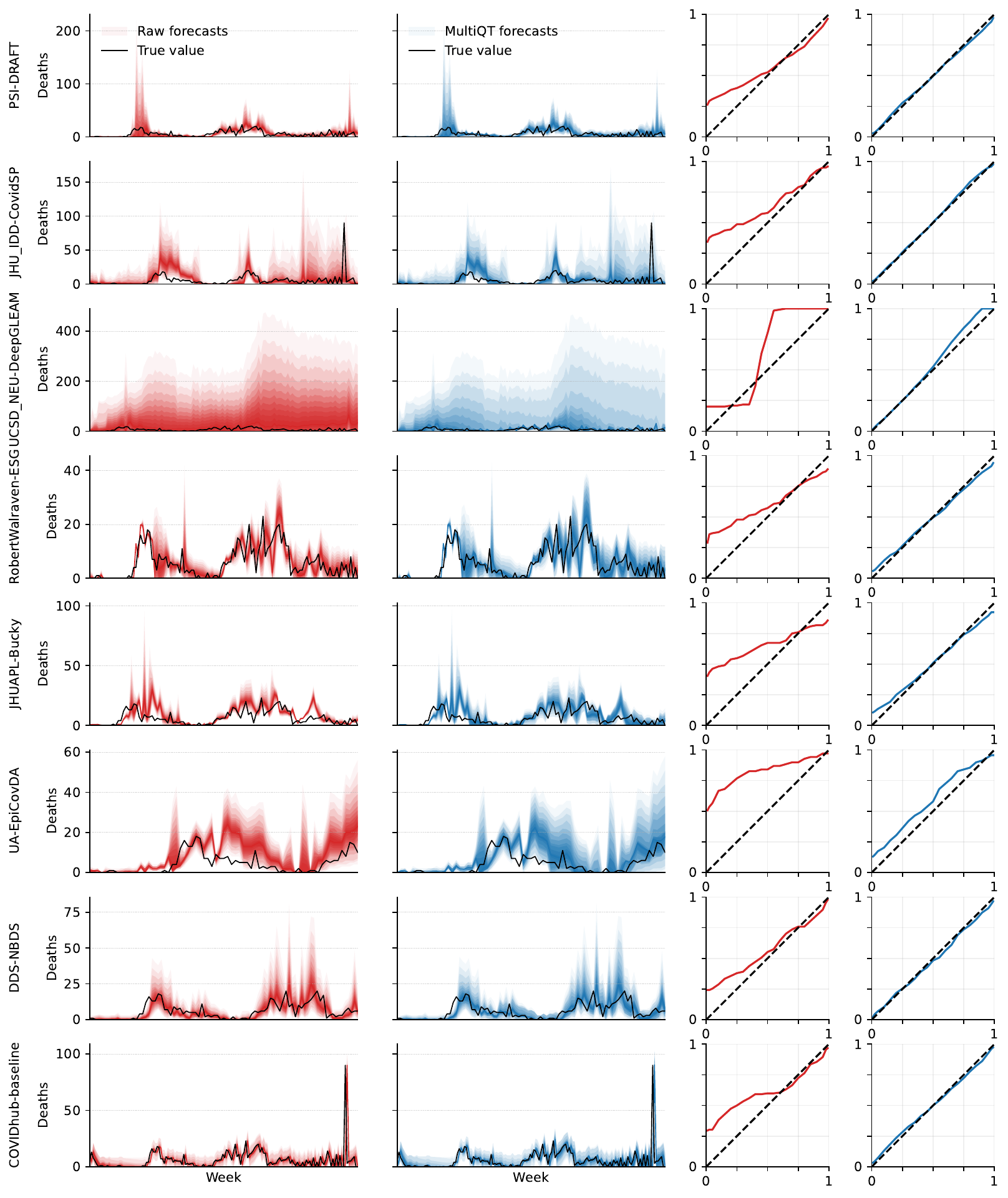}
\caption{One-week-ahead forecasts of weekly COVID-19 deaths in Vermont (part 1 of 2). This is analogous to Figure \ref{fig:all_forecasters_ca_h=1_pt1}, but for Vermont.}   
\label{fig:all_forecasters_vt_h=1_pt1}
\end{figure}

\begin{figure}[p]
\includegraphics[width=\linewidth]{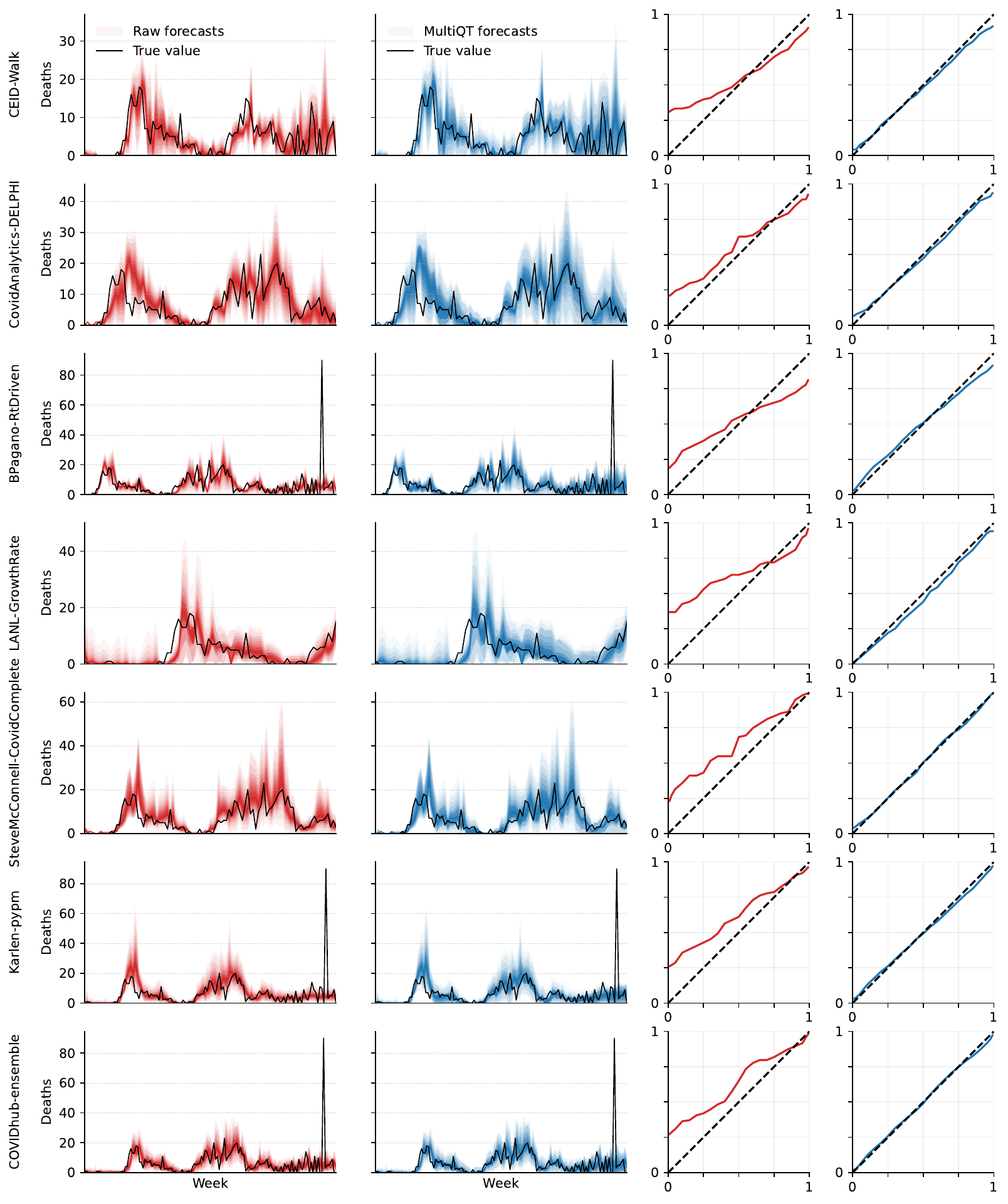}
\caption{One-week-ahead forecasts of weekly COVID-19 deaths in Vermont (part 2 of 2), as in Figure
  \ref{fig:all_forecasters_vt_h=1_pt1}, for the remaining forecasters.}
\label{fig:all_forecasters_vt_h=1_pt2}
\end{figure}

\subsection{Additional energy forecasting results} 
\label{sec:energy_case_studies_APPENDIX}

To complement Figure \ref{fig:energy_calibration_curves} from before,
which shows calibration curves for daily energy forecasts at 10:00 a.m., we
provide analogous plots for 2:00 a.m., 6:00 a.m., 2:00 p.m., 6:00 p.m., and
10:00 p.m., in Figure \ref{fig:wind_calibration_all_hours} (wind energy) and
Figure \ref{fig:solar_calibration_all_hours} (solar energy). MultiQT produces
near-perfect calibration at all hours. We note that in practice, it would be
unnecessary to generate solar energy forecasts for 2:00 a.m.\ and 10:00 p.m. At
these nighttime hours, the solar energy production is always zero, and the raw
quantile forecasts are also zero for all levels. 
Figures \ref{fig:wind_8forecasters} and \ref{fig:solar_8forecasters} visualize
the forecasts before and after applying MultiQT for the 10:00 a.m.\ time block
at eight randomly sampled wind and solar farm sites. 

\begin{figure}[h!]
\centering
\begin{subfigure}{0.16\textwidth}
\includegraphics[width=1.1\linewidth]{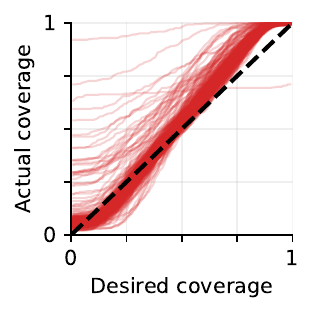}
\end{subfigure}%
\begin{subfigure}{0.16\textwidth}
\includegraphics[width=1.1\linewidth]{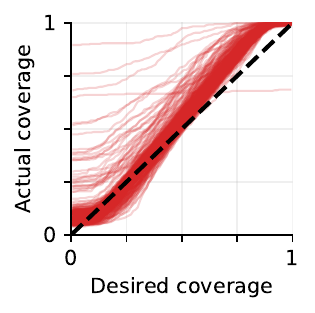}
\end{subfigure}%
\begin{subfigure}{0.16\textwidth}
\includegraphics[width=1.1\linewidth]{figs/energy_hour_16_calibration_raw_target_variable=Wind.pdf}
\end{subfigure}%
\begin{subfigure}{0.16\textwidth}
\includegraphics[width=1.1\linewidth]{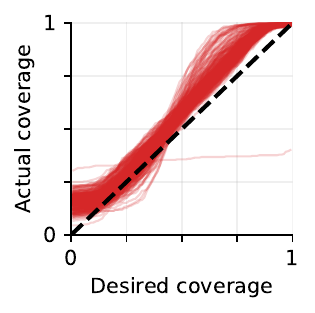}
\end{subfigure}%
\begin{subfigure}{0.16\textwidth}
\includegraphics[width=1.1\linewidth]{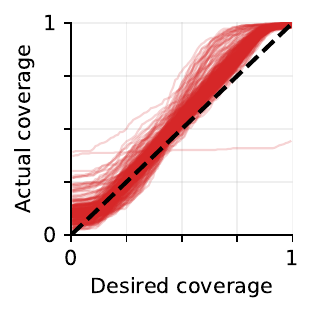}
\end{subfigure}%
\begin{subfigure}{0.16\textwidth}
\includegraphics[width=1.1\linewidth]{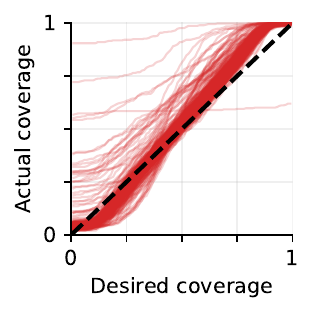}
\end{subfigure}

\begin{subfigure}{0.16\textwidth}
\includegraphics[width=1.1\linewidth]{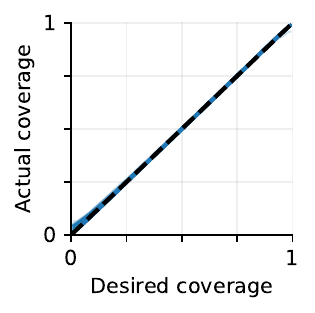}
\caption{2:00 a.m.}
\end{subfigure}%
\begin{subfigure}{0.16\textwidth}
\includegraphics[width=1.1\linewidth]{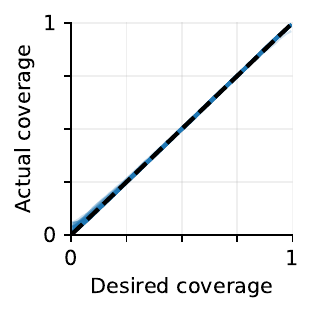}
\caption{6:00 a.m.}
\end{subfigure}%
\begin{subfigure}{0.16\textwidth}
\includegraphics[width=1.1\linewidth]{figs/energy_hour_16_calibration_cal_target_variable=Wind.pdf}
\caption{10:00 a.m.}
\end{subfigure}%
\begin{subfigure}{0.16\textwidth}
\includegraphics[width=1.1\linewidth]{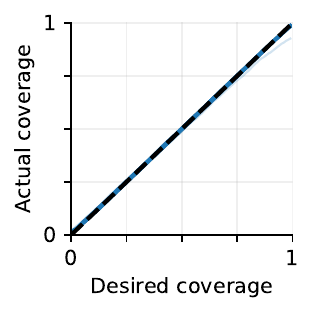}
\caption{2:00 p.m.}
\end{subfigure}%
\begin{subfigure}{0.16\textwidth}
\includegraphics[width=1.1\linewidth]{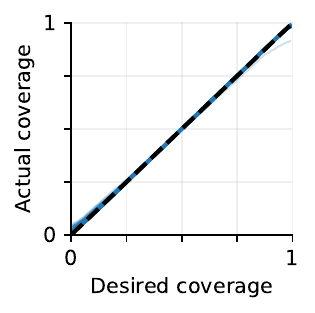}
\caption{6:00 p.m.}
\end{subfigure}%
\begin{subfigure}{0.16\textwidth}
\includegraphics[width=1.1\linewidth]{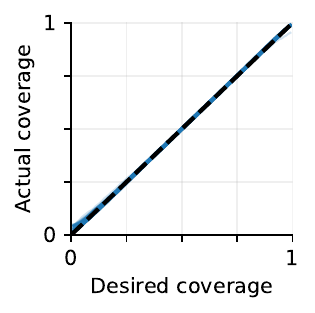}
\caption{10:00 p.m.}
\end{subfigure}

\caption{Actual versus desired coverage for wind energy forecasting, analogous
  to Figure \ref{fig:energy_calibration_curves}. Here the top row corresponds to the 
  raw forecasts, and the bottom row corresponds to the forecasts after applying MultiQT.}   
\label{fig:wind_calibration_all_hours}
\end{figure}

\begin{figure}[h!]
\centering
\begin{subfigure}{0.16\textwidth}
\includegraphics[width=1.1\linewidth]{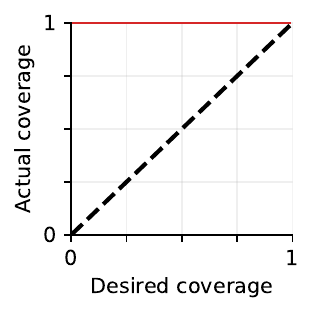}
\end{subfigure}%
\begin{subfigure}{0.16\textwidth}
\includegraphics[width=1.1\linewidth]{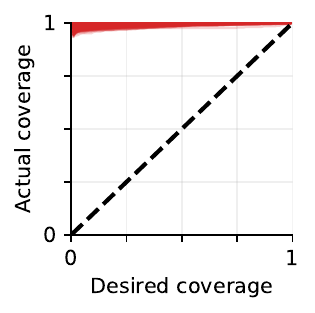}
\end{subfigure}%
\begin{subfigure}{0.16\textwidth}
\includegraphics[width=1.1\linewidth]{figs/energy_hour_16_calibration_raw_target_variable=Solar.pdf}
\end{subfigure}%
\begin{subfigure}{0.16\textwidth}
\includegraphics[width=1.1\linewidth]{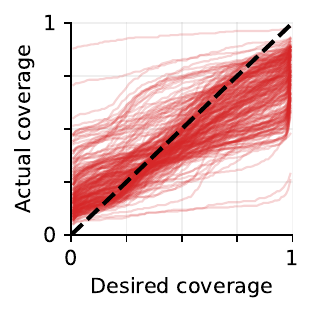}
\end{subfigure}%
\begin{subfigure}{0.16\textwidth}
\includegraphics[width=1.1\linewidth]{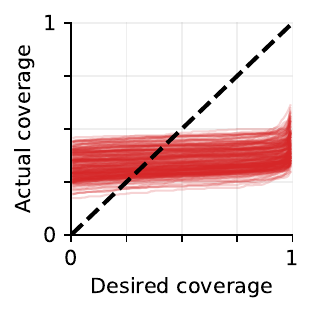}
\end{subfigure}%
\begin{subfigure}{0.16\textwidth}
\includegraphics[width=1.1\linewidth]{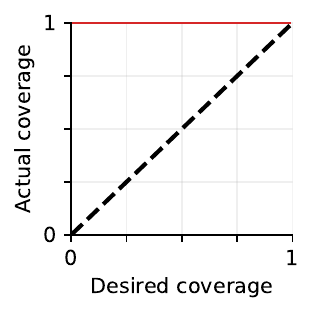}
\end{subfigure}

\begin{subfigure}{0.16\textwidth}
\includegraphics[width=1.1\linewidth]{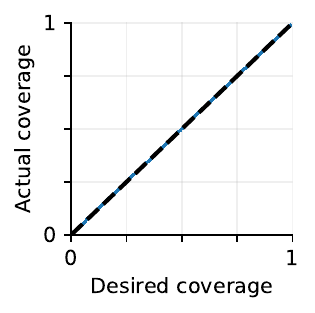}
\caption{2:00 a.m.}
\end{subfigure}%
\begin{subfigure}{0.16\textwidth}
\includegraphics[width=1.1\linewidth]{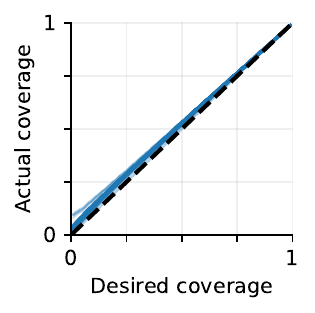}
\caption{6:00 a.m.}
\end{subfigure}%
\begin{subfigure}{0.16\textwidth}
\includegraphics[width=1.1\linewidth]{figs/energy_hour_16_calibration_cal_target_variable=Solar.pdf}
\caption{10:00 a.m.}
\end{subfigure}%
\begin{subfigure}{0.16\textwidth}
\includegraphics[width=1.1\linewidth]{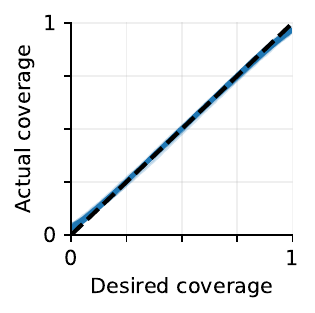}
\caption{2:00 p.m.}
\end{subfigure}%
\begin{subfigure}{0.16\textwidth}
\includegraphics[width=1.1\linewidth]{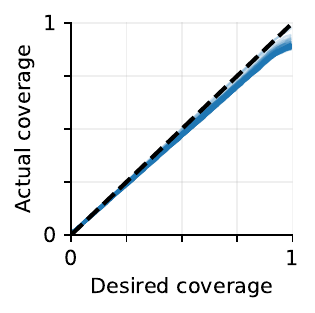}
\caption{6:00 p.m.}
\end{subfigure}%
\begin{subfigure}{0.16\textwidth}
\includegraphics[width=1.1\linewidth]{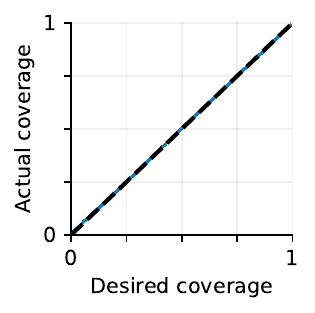}
\caption{10:00 p.m.}
\end{subfigure}

\caption{Actual versus desired coverage for solar energy forecasting, analogous
  to Figure \ref{fig:energy_calibration_curves}. Here the top row corresponds to the 
  raw forecasts, and the bottom row corresponds to the forecasts after applying MultiQT.}   
\label{fig:solar_calibration_all_hours}
\end{figure}

\begin{figure}[p]
\includegraphics[width=\linewidth]{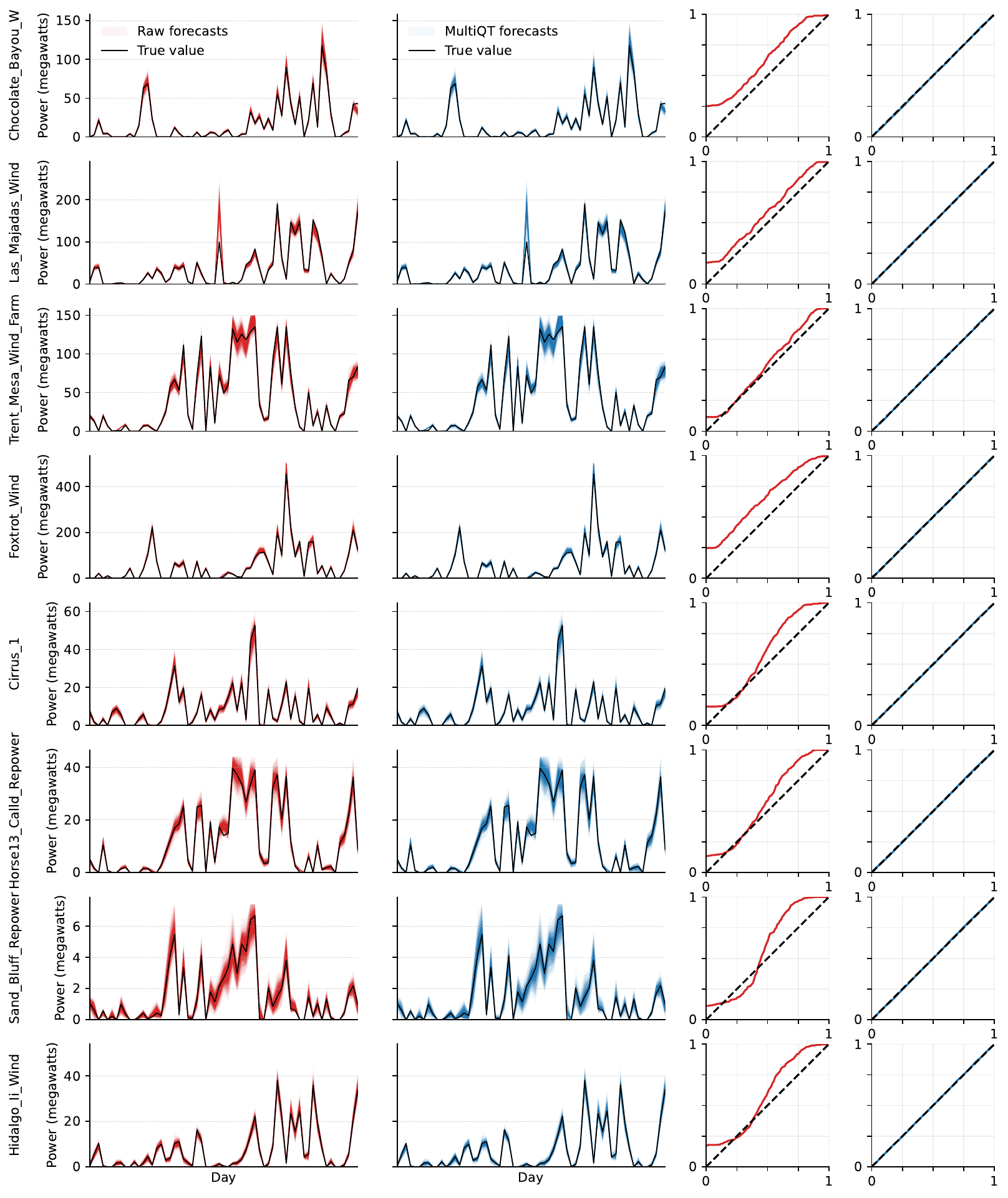}
\caption{Day-ahead wind energy forecasts for the 10:00 a.m.\ time block at eight
  randomly sampled wind farm sites, where each row is a different site. For the
  sake of illustration, forecasts are plotted only for September 1, 2018 to
  October 31, 2018, but calibration is computed using forecasts for every day in
  2018.}  
\label{fig:wind_8forecasters}
\end{figure}

\begin{figure}
\includegraphics[width=\linewidth]{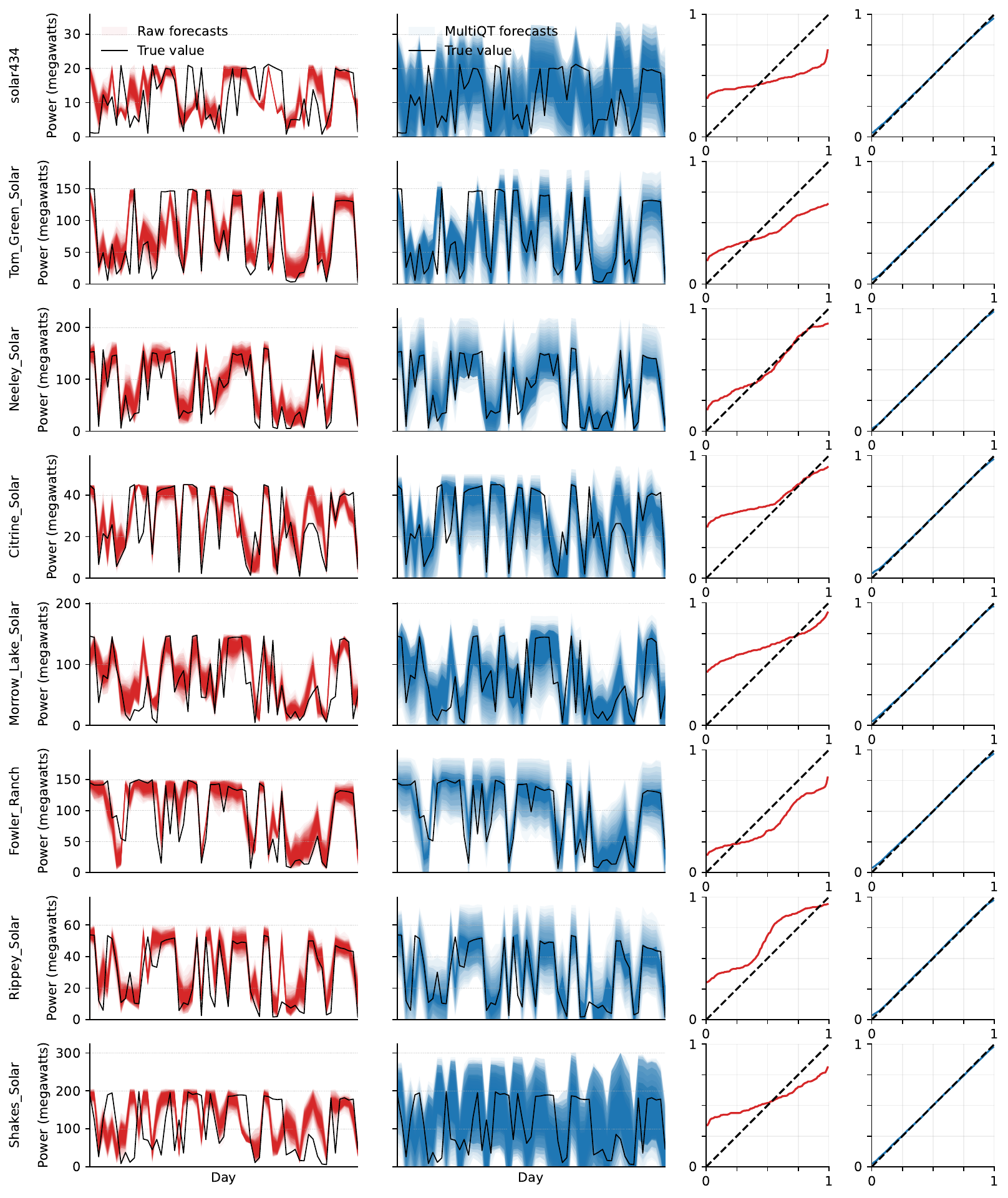}
\caption{As in Figure \ref{fig:wind_8forecasters}, now for solar energy
  forecasting.} 
\label{fig:solar_8forecasters}
\end{figure}

\end{document}